\documentclass[preprint,11pt,authoryear]{elsarticle}
\usepackage{natbib}
\usepackage{amssymb}
\usepackage{amsmath}
\usepackage{amsthm}
\usepackage{bbm}
\usepackage{subcaption}
\usepackage{dashrule}
\usepackage{url}
\usepackage{booktabs}
\usepackage{makecell}
\usepackage{float}
\usepackage{xcolor}

\usepackage{arydshln}

\newtheorem{theorem}{Theorem}[section]

\newtheorem{lemma}[theorem]{Lemma}
\newtheorem{definition}[theorem]{Definition}
\newtheorem{example}{Example}[section]

\newcommand{\cochain}[1]{\mathcal{C}^{#1}(G;\mathcal{F})}
\newcommand{\cohomology}[1]{H^{#1}(G;\mathcal{F})}
\newcommand{\sheaflap}{\Delta^{\mathcal{F}}}

\newcommand{\sheaflaparg}[1]{\Delta^{#1}}

\usepackage{stackengine,scalerel}
\stackMath
\def\leqt{\ThisStyle{\mathrel{%
  \stackinset{r}{.75pt+.15\LMpt}{t}{.1\LMpt}{\rule{.3pt}{1.1\LMex+.2ex}}{\SavedStyle\leqslant}%
}}}

\journal{}

\begin{document}

\begin{frontmatter}

%% Title, authors and addresses

%% use the tnoteref command within \title for footnotes;
%% use the tnotetext command for theassociated footnote;
%% use the fnref command within \author or \address for footnotes;
%% use the fntext command for theassociated footnote;
%% use the corref command within \author for corresponding author footnotes;
%% use the cortext command for theassociated footnote;
%% use the ead command for the email address,
%% and the form \ead[url] for the home page:
%% \title{Title\tnoteref{label1}}
%% \tnotetext[label1]{}
%% \author{Name\corref{cor1}\fnref{label2}}
%% \ead{email address}
%% \ead[url]{home page}
%% \fntext[label2]{}
%% \cortext[cor1]{}
%% \affiliation{organization={},
%%             addressline={},
%%             city={},
%%             postcode={},
%%             state={},
%%             country={}}
%% \fntext[label3]{}

\title{On the use of graph models to achieve individual and group fairness}

%% use optional labels to link authors explicitly to addresses:
%% \author[label1,label2]{}
%% \affiliation[label1]{organization={},
%%             addressline={},
%%             city={},
%%             postcode={},
%%             state={},
%%             country={}}
%%
%% \affiliation[label2]{organization={},
%%             addressline={},
%%             city={},
%%             postcode={},
%%             state={},
%%             country={}}

\author[inst3]{Arturo P\'erez-Peralta}

\author[inst2,inst3]{Sandra Ben\'itez-Peña}
\author[inst2,inst3]{Rosa E. Lillo }

\affiliation[inst2]{organization={Department of Statistics. Universidad Carlos III de Madrid},%Department and Organization
%            city={Madrid},
 %           state={Community of Madrid},
            country={Spain}}
\affiliation[inst3]{organization={uc3m-Santander Big Data Institute. Universidad Carlos III de Madrid},%Department and Organization
%            city={Madrid},
 %           state={Community of Madrid},
            country={Spain}}

\begin{abstract}
%% Text of abstract
Machine Learning algorithms are ubiquitous in key decision-making contexts such as justice, healthcare and finance, which has spawned a great demand for fairness in these procedures. However, the theoretical properties of such models in relation with fairness are still poorly understood, and the intuition behind the relationship between group and individual fairness is still lacking. In this paper, we provide a theoretical framework based on \emph{Sheaf Diffusion} to leverage tools based on dynamical systems and homology to model fairness. Concretely, the proposed method projects input data into a bias-free space that encodes fairness constrains, resulting in fair solutions. Furthermore, we present a collection of network topologies handling different fairness metrics, leading to a unified method capable of dealing with both individual and group bias. The resulting models have a layer of interpretability in the form of closed-form expressions for their SHAP values, consolidating their place in the responsible Artificial Intelligence landscape. Finally, these intuitions are tested on a simulation study and standard fairness benchmarks, where the proposed methods achieve satisfactory results. More concretely, the paper showcases the performance of the proposed models in terms of accuracy and fairness, studying available trade-offs on the Pareto frontier, checking the effects of changing the different hyper-parameters, and delving into the interpretation of its outputs.
\end{abstract}

\begin{keyword}
%% keywords here, in the form: keyword \sep keyword
algorithmic fairness \sep machine learning \sep bias mitigation \sep sheaf diffusion \sep fair topology \sep topological deep learning
\end{keyword}

\end{frontmatter}

%% \linenumbers

%\begin{figure}[h]
%    \centering
%    \includegraphics[width=0.75\linewidth]{FSDIntro.pdf}
%    \caption{Encoding a fairness metrics as a linear vector $f$ allows us to design dynamical systems based on graph diffusion such that, in the time limit, the solution to the system will be aligned with $f$ and thus achieve fairness with respect to $f$.}
%    \label{fig:Intro}
%\end{figure}

%% main text
\section{Introduction}
\label{sec:Introduction}

In recent years, advances in Artificial Intelligence (AI) and Machine Learning (ML) have unleashed unprecendented improvements in automation, becoming a key tool in tasks like speech recognition \citep{SpeechRecognition}, computer vision \citep{VisionReview} and recommendation \citep{RecommenderReview}. However, the use of these automated systems in critical decision-making contexts such as personnel selection \citep{HIRING} or healthcare \citep{HEALTHCARE} raises concerns about the negative impact they might have on certain demographic groups  or individuals \citep{COMPAS}, leading to a widespread demand from both a social and legal perspective for fair models aware of these issues \citep{CRITICALSURVEY, PRESIDENT}. This is the fundamental goal of fair Machine Learning: to demonstrate that it is possible to build accurate models while mitigating harmful biases through careful data processing \citep{FAIRMLBOOK}. Significant progress has been made across multiple domains, including tabular data \citep{JASA_REVIEW}, images \citep{VISION}, text \citep{TEXT} and policy-making \citep{JASA_POLICY}. This work delves into the relation between relational data and bias motivated by the recent successes found in graph-based classifiers, whose foundation lies on the relationship between combinatorial data and true reasoning \citep{INDUCTIVEBIAS}, providing astonishing results in various tasks such as fraud detection \citep{FRAUD}, recommendation \citep{RECSYS}, and social modeling \citep{SOCIAL}.\\
In particular, this paper investigates the use of graph models in fair ML, offering a unified framework that encompasses both group and individual fairness and introducing tools with demonstrable practical effectiveness and deep theoretical implications as a consequence of the underlying object of study. In short, sheaves are algebraic geometric objects that provide a link between topology and abstract algebra, assigning algebraic objects to the open sets of a topological space. This work focuses on their application to combinatorial data in the form of cellular sheaves and Sheaf Diffusion (SD), which has led to the development of a rich theory capable of modeling complex phenomena such as opinion dynamics \citep{OPINION}, and provides a useful framework capable of answering certain fundamental questions about relational models linked to topics ranging from oversmothing to heterophily while achieving state-of-the-art results on common benchmarks \citep{NSD}. In light of these results, this paper defines sheaf models capable of addressing and mitigating bias, being, to our knowledge, the first study relating cellular sheaves to algorithmic prejudice.\\

\subsection{Contributions}
\label{subsec:Contributions}
We provide a unifying framework based on network topology and cellular sheaves capable of tackling both individual and group bias by codifying fairness constraints as a set of algebraic equations that determine the kernel of a linear map. Concretely, this paper defines a collection of graphs encoding fairness constraints from non-graph data, and sheaves which encourage the minimization of said metrics, both at the individual and group level. The resulting model is flexible enough to fulfill the role of pre-processor, in-processor, and post-processor, and is capable of tackling intersectional bias. The intuitions presented are backed by a series of theoretical results which we reproduce for easier availability. Furthermore, the resulting model has interesting interpretability properties in the form of closed-form SHAP values \citep{SHAP}, thus achieving another dimension of responsible AI. Finally, we achieve notable results both in terms of bias and performance on common fairness benchmarks, consolidating the approach introduced in the work.  Moreover, we perform an extensive analysis of the available trade-offs between fairness and accuracy by studying the Pareto frontiers resulting from a grid search, the effect of different hyper-parameters on the used metrics, and the difference between the SHAP values of the proposed models and the benchmark.\\

\subsection{Structure} 
\label{subsec:Structure}
This work is organized as follows: Section~\ref{sec:Related} begins with a review of related work in fairness and graph literature. This exposition is followed by Section~\ref{sec:Theory}, which presents the theoretical background, delving into the details of fair classification both at the group and individual level, and presenting the theory behind cellular sheaves and Sheaf Diffusion. Section~\ref{sec:SheafModels} introduces the concrete models for Fair Sheaf Diffusion (FSD), which are built over a collection of fairness-encouraging network topologies presented in Section~\ref{subsec:TopologicalFairness}. Finally, Section~\ref{sec:Experiments} explains the experimental setup and discusses the results, and Section~\ref{sec:Conclusion} recapitulates the most important ideas synthetizing our conclusions.

\section{Related work}
\label{sec:Related}
This section is devoted to introducing the most significant aspects of the state of the art in three topics that will be linked in this work: fairness, sheaf models, and network topologies.
\subsection{Fairness}
\cite{FAIRMLBOOK} provide a comprehensive compilation on group fairness on tabular data, explaining the motivation, metrics, and methods behind bias mitigation. In practical terms, \cite{EJORSurvey} compare fairness processors in the setting of credit scoring, analyzing profit-fairness trade-offs. By comparison, individual fairness is still relatively underdeveloped with debate over even the most fundamental metrics. \cite{LIPSCHITZ} attempted a first stab at this problem by formalizing a notion of bias at the individual level through the notion of the Lipschitz condition, while \cite{CONSISTENCY} propose a definition based on the $k$ nearest neighbors. Finally, we attempt to provide a unified perspective on these matters, which is similar to \cite{ENTROPY}, who propose a framework to address and quantify the individual-group bias trade-offs. \\
\subsection{Sheaf models}
Sheaf Neural Networks were first introduced by \cite{SNN} and then gained a certain degree of notoriety thanks to \cite{NSD}. Our approach is similar to \cite{SNN} in that we propose a set of \emph{hand-crafted} sheaves, leaving the question on how to extend our methodology to an end-to-end framework like \cite{NSD} for future research. Related work could have applications in fairness, including non-linear sheaf diffusion \citep{NONLINEAR}, using higher-order connection Laplacians \citep{HIGHERORDER}, or considering sheaf models based on the wave equation \citep{SHEAFSUR}. Finally, readers interested in more comprehensive treatments of sheaf theory are encouraged to consult Chapter 2 of \cite{VAKIL}, while \cite{CELLULAR} provides a detailed discussion on cellular sheaves.\\
\subsection{Network topologies}
The interaction of graph models and fairness has already been explored in the literature. \cite{GRAPHINDIVIDUAL} propose a message passing model to achieve a fair ranking at the individual level, while \cite{GRAPHGROUP} propose a fair architecture on graph neural networks to achieve group fairness. However, this work distinguishes itself by proposing graph configurations on tabular data which might not have a readily-available graph, thus gaining a certain degree of generality. One of the ideas presented in this paper revolves around imposing a global graph structure to achieve enhanced results, which can be traced back to the original PageRank paper by \cite{PAGERANK}, who avoid rank sinks by adding a global source of rank. \cite{FAIRPAGERANK} explore how this very structure lends itself to applications in fairness. 

\section{Theoretical background}
\label{sec:Theory}

This Section presents the theoretical background necessary in order to understand the work. First, we will explain the problem of classification with a sensitive variable and present both group and individual fairness metrics. Then, we will delve into the theory of cellular sheaves, exposing the basics of the tools borrowed from this field with the goal of understanding SD and the main results we build upon.

\subsection{Fair classification}
\label{subsec:Fair}

Suppose a dataset, $D = (X, Y, A)$, where $X$ is a set of $d$ attributes, $Y$ is the target variable, and $A$ represents sensitive information. In the general case, $A$ may comprise one or more sensitive attributes; for the sake of simplicity and without loss of generality, we focus on the common scenario where $A$ consists of a single binary variable, i.e., $A \in \{0,1\}$. In this setting, we assume that one of the two values of $A$ represents a privileged group, while the other represents a protected group particularly affected by prejudice. More concretely, we adopt the convention that $A=1$ signifies belonging to the privileged group while $A=0$ represents the protected group. For example, when modeling the gender wage gap one would assign $A=1$ to men and $A=0$ to women. Furthermore, we also assume that the target is binary, $Y \in \{0,1\}$. The goal of statistical classification is then to develop a model, $\hat{Y} : dom(X) \longrightarrow \{0,1 \}$, which receives the name of classifier, while maximizing its accuracy (or, equivalently, while minimizing its error rate),
\begin{equation}
    \text{max}\ \mathbb{P}(\hat{Y} = Y ) \equiv \text{min} \ \mathbb{P} (\hat{Y} \neq Y ),
\end{equation}
where $\mathbb{P}$ denotes the probability of an event. However, in the context of fairness it makes more sense to maximize the balanced accuracy, that is, the mean group accuracy,
\begin{equation}
\textit{balanced accuracy} = \frac{1}{|A|} \sum_{a\in A}  \mathbb{P}(\hat{Y} = Y | A = a).
\label{balacc}
\end{equation}
Other important performance quantities to take into consideration  are the entries of the confusion matrix:
\begin{itemize}
    \item \textbf{True positive ratio:} $TPR = \mathbb{P}(\hat{Y}  = 1 \mid Y = 1)$.
    \item \textbf{True negative ratio:} $TNR = \mathbb{P}(\hat{Y}  = 0 \mid Y = 0)$.
    \item \textbf{False positive ratio:} $FPR = \mathbb{P}(\hat{Y}  = 1 \mid Y = 0)$.
    \item \textbf{False negative ratio:} $FNR = \mathbb{P}(\hat{Y}  = 0 \mid Y = 1)$.
\end{itemize}
Usually, the classifier $\hat{Y}$ is obtained by thresholding a numeric model $f: dom(X) \longrightarrow \mathbb{R}$ that ranks the different observations according to their probability of success. Formally, a threshold $\tau$ is set and the classifier is built by thresholding the score $f$, that is, $\hat{Y} = 1$ if $f(X) > \tau$ and $\hat{Y} = 0$ otherwise. In this work, we consider that the predictive model $f$ can be represented in a parametric form as $f_{\theta}$, where $\theta \in \Theta \subset \mathbb{R}^{n_{\theta}}$ is a vector parameter, usually chosen by minimizing a loss function that smoothly approximates accuracy and is denoted by $\mathcal{L}(f;\theta)$. In the case of binary classification this loss is given by the negative log-likelihood:

\begin{equation*}
    \mathcal{L}(f_{\theta};\theta) = \sum_{i=1}^n y_i \log f_{\theta}(x_i) + (1-y_i) \log (1-f_{\theta}(x_i))
\end{equation*}

Finally, another dimension of trustworthy AI complementary to fairness is found in interpretability \citep{EXPLAIN}. One of the most powerful tools of interpretable ML is given by SHAP values \citep{SHAP}. This tool provides a general perspective on additive model explanations, crowning itself as the only interpretative model satisfying a series of desirable properties which further generalizes common explainability tools. The biggest shortcoming of SHAP values lies in their expensive computation. However, the use of linear models can overcome this weakness thanks to the existence of a closed-form expression for the SHAP values, which will become particularly relevant as we introduce our methodology. In particular, if a linear model with $d$ features is parametrized as $f(x) = \beta_0 + \sum_{j=1}^d \beta_j x_{j}$, the SHAP contribution of feature $k$ on individual $i$ is given by

\begin{equation}
    \phi_{ki} = \beta_k \left( x_{ik} - \mathbb{E}(x_{k}) \right),
    \label{SHAPEQ}
\end{equation}
where $\mathbb{E}(x_{j})$ denotes the expected value of feature $j$.

Now the only question that remains is how to measure bias and prejudice.

\subsection{Group Fairness metrics}
\label{subsec:Group}

Group Fairness metrics quantify discrimination against a protected group by aggregating model performance at the group level and checking if there are any notable differences. Although the literature is packed with metrics of this nature, most of them are equivalent to just three \citep{FAIRMLBOOK}: Independence, Separation, and Sufficiency; all of which can be understood as statements about the statistical independence of the classifier, $\hat{Y}$, the response, $Y$, and the sensitive attribute, $A$.

\subsubsection{Independence}
\label{subsubsec:Independence}

A classifier $\hat{Y} $ is said to satisfy independence at a threshold $\tau$ if

\begin{equation}
\mathbb{P}[\hat{Y}  = 1 \mid A = 1 ] = \mathbb{P}[\hat{Y}  = 1 \mid A = 0 ].
\label{INDEPENDENCE}
\end{equation}
That is, independence requires that the distribution of the classifier is independent of $A$. However, this assumption might be counterproductive if the underlying distribution of the data is not independent itself, in which case enforcing this metric may lead to an inaccurate model \citep{EJORSurvey}. In order to use this metric, some institutions propose to bound the quotient of the quantities in \eqref{INDEPENDENCE}, generally by $0.8$ which is commonly referred to as the \emph{80\% rule} \citep{EMPLOYEESELECTION}. In our case we will measure independence by the absolute value of the deviation in the above probabilities,

\begin{equation}
IND = | \mathbb{P}[\hat{Y} = 1 \mid A = 1 ] - \mathbb{P}[\hat{Y}  = 1 \mid A = 0 ]|.
\label{IND}
\end{equation}
Therefore, a positive value for $IND$ implies a divergence from the equality in \eqref{INDEPENDENCE}. Hence, the closer $IND$ is to zero, the lower the discrimination is.

\subsubsection{Separation}
\label{subsubsec:Separation}

A classifier $\hat{Y}$ is said to satisfy separation at a threshold $\tau$ if
\begin{subequations}
\begin{align}
    \mathbb{P}[\hat{Y}  = 1 \mid Y = 0, A = 1 ] = \mathbb{P}[\hat{Y}  = 1 \mid Y = 0, A = 0 ], \label{SEPARATION1} \\
   \mathbb{P}[\hat{Y}  = 1 \mid Y = 1, A = 1 ] = \mathbb{P}[\hat{Y}  = 1 \mid Y = 1, A = 0 ].
    \label{SEPARATION2}
\end{align}
\end{subequations}
That is, a score fulfils separation if all groups have equal error rates. The positive outcome is not assumed to be equally distributed, but ideally the error rate across population groups are equalized across the different values of the response; that is, a classifier that aims for separation does not try to improve global accuracy at the cost of misclassifying individuals from the minority group. This metric also receives the name \emph{equal odds}. A relaxation can be found in \emph{equal opportunity} in which only equation \eqref{SEPARATION1} is satisfied \citep{POSTEQODDS}.\\ 
Separation can be measured with the unweighted average of the absolute value of the deviation of the false positive and negative rates,
\begin{equation}
SEP = \frac{1}{2} | (FPR_{A = 1} - FPR_{A = 0} ) + (FNR_{A = 1} - FNR_{A = 0} )|.
\label{SP}
\end{equation}
$SEP$ has a similar interpretation to $IND$: the closer it is to zero, the closer the classifier is to achieve separation and the lower the discrimination.\\

\subsubsection{Sufficiency}
\label{subsubsec:Sufficiency}

A score $\hat{Y}$ is said to satisfy sufficiency at a threshold $\tau$ if
\begin{equation}
\mathbb{P}[Y = 1 \mid\hat{Y}  = 1, A = 1 ] = \mathbb{P}[Y = 1 \mid \hat{Y}  = 1, A = 0 ].
\label{SUFFICIENCY}
\end{equation}
That is, it requires that all the information of the target variable is contained in the classifier. This metric is related with the notion of calibration, that is, that the probabilities given by the model reflect actual probabilities. However, this metric has been heavily criticized because it fails to accurately measure bias and prejudice, see for instance \cite{EJORSurvey}.
In any case, sufficiency can be measured as the absolute value of the deviation of the above probabilities,
\begin{equation}
SUF = |\mathbb{P}[Y = 1 \mid \hat{Y}  = 1, A = 1 ] - \mathbb{P}[Y = 1 \mid \hat{Y}  = 1, A = 0 ] |.
\label{SF}
\end{equation}

\subsection{Individual Fairness Metrics}
\label{subsec:Individual}

On the other hand, Individual Fairness metrics are born from the notion that similar individuals should receive similar treatment. This change of philosophy shakes the very focus of the metrics, which must now accomodate a notion of similarity on both individuals and treatments, which is manifested through a pair of distance functions $d_X, d_Y$ on the response variable and covariates, respectively. The choice of distance functions is hardly trivial. Some works exploring the use of different metrics and their effects on Individual Fairness include \cite{CONSISTENCYCRITIC} and \cite{LIPSCHITZ}. In general, this choice requires deep domain knowledge of the use case to faithfully encode similarity between individuals. In our case, we have opted to use the total variation for the predictions and the euclidean metric for the covariates for simplicity. Our idea is to do an apples-to-apples comparison by checking the effect of the proposed processing on distance-dependent metrics under a constant choice of said distance. It is possible that processing the data or the model with a \emph{different} metric could lead to better results with the \emph{desired metric}, in a kind of apples-to-oranges situation. For example, creating network topologies based on the Mahalannobis distance might lead to a better consistency result with the euclidean metric. Nonetheless, this connection is not obvious to us and we relegate it to future work. 

Returning to the topic of individual fairness, despite the recency of the field, there are three promising metrics which have garnered certain degree of notoriety: Lipschitz Constants, Consistency and Generalized Entropy.

\subsubsection{Lipschitz constant}
\label{subsubsec:Lipschitz}

The notion that similar observations should yield similar results is related to ideas of continuity. In particular, \cite{LIPSCHITZ} formalize these concepts through the Lipschitz condition; that is, a function $f: X \longrightarrow Y$ between two metric spaces $X$ and $Y$ with respective metrics $d_X$ and $d_Y$ is said to be Lipschitz continuous if there exists a constant $L > 0$ such that

\begin{equation*}
    d_Y(f(x_1), f(x_2)) \leq L d_X(x_1, x_2)
\end{equation*}
for all $x_1, x_2 \in X$. A model $f$ is then said to be individually fair if the Lipschitz constant is less than or equal to one. The fact that we are dealing with finite datasets instead of (possibly) uncountable metric spaces means that in practise all proposed models have an effective Lipschitz constant. Effectively, assume $x_1 \neq x_2$:

\begin{multline*}
d_Y(f(x_1), f(x_2)) = \frac{d_Y(f(x_1), f(x_2))}{d_X(x_1, x_2)} d_X(x_1, x_2) \\ \leq \max_{x_1 \neq x_2} \left\{ \frac{d_Y(f(x_1), f(x_2))}{d_X(x_1, x_2)}\right\} d_X(x_1, x_2),
\end{multline*}
hence we can quantify how individually fair our model is by computing the maximum quotient of distances. However, we have opted instead to compute the $0.99$ quantile to obtain a more stable measurement:

\begin{equation*}
    LIP = Q\left(0.99, \left\{ \frac{|f(x_1) - f(x_2)|}{\|x_1 - x_2 \|_2} \mid x_1, x_2 \in X, x_1 \neq x_2 \right\}\right),
\end{equation*}
where we use $Q(p; X)$ to denote the quantile function of distribution $X$ at point $p$. In this conext, a model is individually fair if $LIP \leq 1$.\\
Another issue with Lipschitz constants is that they can become unwieldly to use, resulting in restrictive programs to achieve the desired debiasing. \cite{LOCALLIPSCHITZ} propose to relax this conditions by instead aiming for a model that is locally Lipschitz: that is, a function $f: X \longrightarrow Y$ is said to be locally Lipschitz if for all $x\in X$ there exists a neighbor $x \in B_x \subset X$ such that $f\mid_{B_x}$ is Lipschitz. This notion serves as a bridge between the more restrictive Lipschitz condition and a more tractable local notion of fairness, achieving a trade-off between fairness and efficiency.

\subsubsection{Consistency}
\label{subsubsec:Consistency}

Another local relaxation of individual fairness is found in consistency, which compares every observation with its $k$-Nearest Neighbors (kNN), resulting in a metric that is less computationally demanding. In particular, it measures the mean difference between the prediction for a given individual and the average prediction of its kNN. Formally, given a metric $d$ on the covariates $X$, let $\mathcal{N}^k_d(x)$ denote the set of the kNN of observation $x$ according to metric $d$. The consistency of a function $f$ is then expressed as

\begin{equation*}
     1 - \frac{1}{n} \sum_{i=1}^n \left|f(x_i) - \frac{1}{k} \sum_{z\in \mathcal{N}_d^k(x)} f(z)\right|.
\end{equation*}
Note that, expressed in this manner, consistency is the only direct measure of fairness we have presented; that is, the bigger it is the fairer the model. In order to stay consistent with our use of metrics we instead use the following measure of unfairness:

\begin{equation*}
     CON_{d}^{k} = \frac{1}{n} \sum_{i=1}^n \left|f(x_i) - \frac{1}{k} \sum_{z\in \mathcal{N}_d^k(x)} f(z)\right|,
\end{equation*}
where both $k$ and $d$ may be omitted if clear from context; in our case, we will use five neighbors and the euclidean distance. This way, the smaller the value of $CON_{d}^{k}$, the fairer the model is.
In spite of the conceptual and computational advantages associated with this metric it is not without its shortcomings. \cite{CONSISTENCYCRITIC} are concerned with how consistency may obscure individual fairness through the averaging, and its sensitivity to the underlying metric. We insist in using the metric due to its simplicity, emphazising the importance of understanding its pitfalls. 

\subsubsection{Generalized entropy}
\label{subsubsec:Entropy}

\cite{ENTROPY} propose a collection of metrics to quantify both individual and group bias, providing a unifying perspective on the matter. This collection is given by the family of Generalized Entropy indices which is parametrized by a real number $\alpha \neq \{0, 1\}$

\begin{equation*}
    \mathcal{E}^\alpha (\mathbf{b}) = \frac{1}{n\alpha(1-\alpha)} \sum_{i = 1}^n \left[ \left(\frac{b_i}{\mu}\right)^{\alpha} - 1 \right],
\end{equation*}
where $b_i$ represents the ``benefit'' perceived by individual $i$ and $\mu$ is the average benefit. In binary classification there are four possible values the benefits can take corresponding to true and false positives and negatives. This metric provides an information theoretic measurement of the mean deviation from the average $\mu$ weighted by an exponent $\alpha$, thus measuring individual fairness. The most interesting property of this metric is found in its additive decomposition into within-group and between-group components, measuring both individual and group bias at the same time and formalizing their trade-offs. Given a partition of the dataset into $G$ groups of size $n_g$ with mean benefit $\mu_g$ for $g\in G$ this decomposition is explicitly given by:
\begin{equation*}
    \mathcal{E}^{\alpha}(\mathbf{b}) = \mathcal{E}_{\omega}^{\alpha}(\mathbf{b}) + \mathcal{E}_{\beta}^{\alpha}(\mathbf{b}),
\end{equation*}
where 
\begin{equation*}
    \mathcal{E}_{\beta}^{\alpha}(\mathbf{b}) = \sum_{g\in G} \frac{n_g}{n\alpha (1-\alpha)}\left[\left(\frac{\mu_g}{\mu} \right)^{\alpha}- 1 \right]
\end{equation*}
denotes the between-group component, measuring group fairness, and
\begin{equation*}
    \mathcal{E}_{\omega}^{\alpha}(\mathbf{b}) = \sum_{g\in G} \frac{n_g}{n} \left(\frac{\mu_g}{\mu}\right)^{\alpha} \mathcal{E}^{\alpha}(\mathbf{b}^g)
\end{equation*}
is the within-group component, which measures individual fairness within each group.\\
Following \cite{ENTROPY} we opt for $\alpha = 2$, and thus the metric becomes half the coefficient of variation, which is a more widely-known metric of dispersion. However, we depart from previous work by computing benefits as $1$ for correct classifications and $0$ for misclassifications, therefore $b_i = \mathbbm{1}(y_i = \hat{y}_i)$ is simply the accuracy of prediction $i$. The resulting metric is:

\begin{equation*}
    ENT = \frac{1}{2n} \sum_{i = 1}^n \left[ \left(\frac{\mathbbm{1}(y_i = \hat{y}_i)}{\mu}\right)^{2} - 1 \right],
\end{equation*}
measuring the mean deviation from the average precision. We have now finished reviewing all relevant fairness metrics. A summary can be found in Table \ref{tab:metrics} which keeps track of all metrics, their expressions and whether they measure group or individual fairness.

\begin{table}[ht]
    \centering
    \begin{tabular}{ccc}
        \toprule
        Metric & Expression & Scope  \\
        \midrule
        IND & $| \mathbb{P}[\hat{Y} = 1 \mid A = 1 ] - \mathbb{P}[\hat{Y}  = 1 \mid A = 0 ]|$ & Group \\ \hdashline
        SEP & $\frac{1}{2} | (FPR_{A = 1} - FPR_{A = 0} ) + (FNR_{A = 1} - FNR_{A = 0} )|$ & Group \\ \hdashline
        SUF & $|\mathbb{P}[Y = 1 \mid \hat{Y}  = 1, A = 1 ] - \mathbb{P}[Y = 1 \mid \hat{Y}  = 1, A = 0 ] |$ & Group \\ \hdashline
        LIP & $Q\left(0.99; \left\{ \frac{|f(x_1) - f(x_2)|}{\|x_1 - x_2 \|_2} \mid x_1, x_2 \in X, x_1 \neq x_2 \right\}\right)$ & Individual \\ \hdashline
        CON & $\frac{1}{n} \sum_{i=1}^n \left|f(x_i) - \frac{1}{k} \sum_{z\in \mathcal{N}_2^5(x)} f(z)\right|$ & Individual \\ \hdashline
        ENT & $\frac{1}{2n} \sum_{i = 1}^n \left[ \left(\frac{\mathbbm{1}(y_i = \hat{y}_i)}{\mu}\right)^{2} - 1 \right]$ & Both \\
        \bottomrule  
    \end{tabular}
    \caption{Summary of used fairness metrics.}
    \label{tab:metrics}
\end{table}

\subsection{Cellular sheaf theory}
\label{subsec:SheafTheory}

This Section presents the core theoretical background necessary to understand the incoming methodology, broadly following \cite{CELLULAR, OPINION, NSD}. We will introduce the main definitions behind cellular sheaves with the goal of building towards Sheaf Diffusion. The exposition will conclude with a discussion of Sheaf Diffusion dynamics and the main results backing the following debiasing approach.

\subsubsection{Cellular sheaves}
\label{subsubsec:Cellular}

A \emph{cellular sheaf} is an algebraic-topological structure associated with a graph that assigns spaces to nodes and edges. Formally, assume an undirected graph $G = (V, E)$. Given $u,v \in V$, $e \in E$ we will adopt the notation $u \sim v$ to imply $(u,v) \in E$, and $v\leqt e$ to indicate that $v$ is an endpoint of $e$. A cellular sheaf $\mathcal{F}$ on an undirected graph $G$ is given by the following data:

\begin{enumerate}
    \item a vector space $\mathcal{F}(v)$ for each vertex $v$ of $G$.
    \item a vector space $\mathcal{F}(e)$ for each edge $e$ of $G$.
    \item a linear map $\mathcal{F}_{v \leqt e}: \mathcal{F}(v) \longrightarrow \mathcal{F}(e)$ for each incident vertex-edge pair $v \leqt e$ of $G$.
\end{enumerate}    
The vector spaces associated to each edge and vertex receive the name of \emph{stalks}, and the linear maps between vertex and edge stalks are called \emph{restriction maps}.
The vector space associated with the data vertices of $G$ is the space of \emph{0-cochains} valued in $\mathcal{F}$, formally

\begin{equation*}
    \cochain{0} \equiv \bigoplus_{v\in V} \mathcal{F}(v),
\end{equation*}
where $\bigoplus$ denotes the direct sum of vector spaces. To keep our notation straight, we will write $0$-cochains with lower case, $x \in \cochain{0} \cong \mathbb{R}^{\sum_{v}n_v}$, where $n_v$ is the dimension of the node stalk $v$. If all node stalks have the same dimension $d$ then $\cochain{0} \cong \mathbb{R}^{| V | \times d}$. Sometimes $0$-cochains will be indexed with a temporal label, which we will write as a superscript, $x^t$. The projection of the cochain onto a node stalk $\mathcal{F}(v)$ will be denoted by $x_v \in \mathcal{F}(v) \cong \mathbb{R}^{n_v}$ or, equivalently, by $x_i$ if we are talking about the node associated with observation $i$. Finally, each component of the projection onto the node stalk will be denoted with a double subscript by $x_{v,j}$ or $x_{i,j}$ depending on context.\\
Similarly, the space of data associated to edges is the space of \emph{1-cochains} valued in $\mathcal{F}$:

\begin{equation*}
    \cochain{1} \equiv \bigoplus_{e\in E} \mathcal{F}(e)
\end{equation*}
Given an edge $e = (u, v)$ we say that a choice of data $x_v \in \mathcal{F}(v), x_u \in \mathcal{F}(u)$ is \emph{consistent} over $e$ if $\mathcal{F}_{u\leqt e} x_u = \mathcal{F}_{v\leqt e} x_v$. Therefore, each edge imposes a constraint by restricting the space associated with two incident vertices. The subspace of $C^0(G;\mathcal{F})$ of signals satisfying all edge constraints is the space of \emph{global sections} of $\mathcal{F}$, and is denoted by $H^0(G; \mathcal{F})$. As a matter of fact, the space of global sections is given by the kernel of a linear map called \emph{coboundary}, $\delta: C^0(G; \mathcal{F}) \longrightarrow C^1(G; \mathcal{F})$. To define $\delta$ explicitly, an orientation must be chosen for each edge, in which case the for and edge $e=(u,v)$ the coboundary is given by $(\delta x)_e = \mathcal{F}_{v \leqt e} x_v - \mathcal{F}_{u \leqt e} x_u$ where $e = (u, v)$. However, we only care about the kernel of $\delta$, thus making the choice of orientation irrelevant. In any case, if $x \in \ker \delta$, then $\mathcal{F}_{u \leqt e} x_u  = \mathcal{F}_{v \leqt e} x_v$ for every edge $e$. From the coboundary operator we may construct the \emph{sheaf laplacian} $L^{\mathcal{F}} = \delta ^{\top} \delta$, which is a positive semidefinite linear operator on $C^0(G;\mathcal{F})$ with kernel equal to $H^0(G; \mathcal{F})$ by Hodge theorem \citep{HODGE}. Note that the difference in signs which appears in the coboundary operator implies a choice of orientation for the edges. However, this technical detail becomes irrelevant after introducing the sheaf laplacian, which is independent of this choice in undirected graphs.\\
The sheaf laplacian is a block matrix given by:
\begin{equation*}
\begin{split}
    L^{\mathcal{F}}_{vu} &= -\mathcal{F}_{v\leqt e}^{\top} \mathcal{F}_{u \leqt e}, \quad \text{if } e = (u,v)\in E, u\neq v, \\
        L^{\mathcal{F}}_{vv} &= \sum_{v\in e} \mathcal{F}_{v\leqt e}^{\top} \mathcal{F}_{v \leqt e}, \text{ otherwise}.
\end{split}
\end{equation*}
This operator is usually normalized in order to bound its spectrum for numerical stability:

\begin{equation*}
    \Delta^{\mathcal{F}} = D^{-1/2} L^{\mathcal{F}} D^{-1/2},
\end{equation*}
where $D$ is a block diagonal matrix whose elements are $L^{\mathcal{F}}_{vv}$. Therefore, it is symmetric semi-definite positive and its square root is well defined. If this matrix was not full rank, the negative power is implied to mean the Moore-Penrose pseudoinverse.

\subsubsection{Sheaf Diffusion}
\label{subsubsec:SheafDiffusion}

Sheaf Diffusion (SD) is a process analogous to the heat equation on a graph, with the sheaf laplacian playing the role of the usual laplacian operator on euclidean space. Basically, the heat equation models temperature exchange between points in a medium through the laplacian differential operator, while SD models information exchange between nodes through the sheaf laplacian. Intuitively, we can think that nodes store information encoded in a vector space $\mathcal{F}(v)$, usually $\mathbb{R}^{n_v}$. In our use case, this vector space represents either the inputs or outputs of a given model. Furthermore, nodes are allowed to exchange information with their neighbors through communication channels given edge stalks, $\mathcal{F}(e)$, which can have an identical or bigger dimension than node stalks (thus allowing the free flow of information) or a lower dimension (thus introducing a bottleneck). The information exchange is mediated by the restriction maps, $\mathcal{F}_{v \leqt (u,v)}$, which can let information flow freely or transform it through a linear map. This process develops until all nodes reach a global consensus, just as the heat equation develops until all points reach the same temperature. Consensus between nodes is achieved when $\mathcal{F}_{v \leqt (u,v)} x_v = \mathcal{F}_{u \leqt (u,v)} x_u$ or, equivalently, when $\delta x_{(u,v)} = 0$; and the space of global consensuses is given by $\cohomology{0}$, towards which the SD converges.\\
Opinion dynamics \citep{OPINION} provides a concrete mental model which cements this previous intuition. Basically, node stalks represent discourse spaces built from a series of base topics. The opinion of an individual represented by said node on any particular matter can be obtained as a combination of these base topics. Individuals may exchange opinions through a communication channel represented by and edge. The stalk of this edge represents the common discourse space for two agents, likewise composed of a series of common topics through which may differ from any individual core subject configuring each node stalk. The restriction maps represents the way each person translates his opinion into a concrete communication channel, and SD represents the process through which all agents communicate their opinions until global consensus is reached. Therefore, we can think of the SD process as a exchange of information until all agents involved agree on their publicly shared opinions even if they disagree in one concrete topic only present in their own private headspace represented by the node stalk.\\
Moving onto the formal discussion of SD, a signal $x^t\in \cochain{0}$ is said to evolve according to Sheaf Diffusion if it follows the ordinary differential equation (ODE):
\begin{equation}
    \frac{dx^t}{dt} = - \alpha \sheaflap x^t.
    \label{SheafDif}
\end{equation}
where $\alpha > 0$.\\
In analogy to the heat equation, sheaf diffusion has an associated energy in the form of the Dirichlet Energy of a signal $x \in \cochain{0}$:
\begin{equation*}
    E_{\mathcal{F}}(x) = x^{\top} \sheaflap x = \frac{1}{2} \sum_{(u,v) \leqt e} \| \mathcal{F}_{v\leqt e} x_v - \mathcal{F}_{u \leqt e} x_u \|,
\end{equation*}
which is a Lyapunov function for the dynamic process. As the diffusion process develops, the Dirichlet Energy of the initial signal is dissipated, decreasing until becoming zero in the time limit. This suggests an outline for creating a fair model: designing a sheaf laplacian whose energy codifies an unfairness metric or, equivalently, whose kernel codifies fairness constraints, and running the diffusion process should result in an increasingly fair solution. This notion is formalized by the following result by \cite{OPINION}:
\begin{theorem}
    Suppose a signal $x^t\in \cochain{0}$ under a sheaf diffusion process with initial condition $x^0$. In the limit  $t\longrightarrow \infty$ the signal $x^t$ tends to the orthogonal projection of $x^0$ onto $\ker \sheaflap$. 
\end{theorem}
\noindent The proof of this theorem is based on the analytical solution of the linear system \eqref{SheafDif}:
\begin{equation}
    x^t = \exp\left(-\alpha t \sheaflap \right) x_0,
    \label{SheafDifuSolution}
\end{equation}
which further shows that the convergence has exponential speed. In any case, this result formalizes our previous ideas and guarantees that the information lost to the diffusion process is minimal due to the fact that the projection onto the kernel is orthogonal.\\
To conclude this section, we discuss concrete implementations of the ODE \eqref{SheafDif}, namely, this ODE can be implemented by either discretizing o using the exact solution to \eqref{SheafDif}. Therefore, given a signal $x^0 \in \cochain{0}$  there are two different schemes for SD, one discrete and one continuous: 
\begin{equation*}
    \text{CONT}\equiv x^t = \exp\left(-t \alpha \sheaflap\right) x^0 \quad \quad \text{DISC}\equiv x^n = (1 - \alpha\sheaflap) x^{n-1}
\end{equation*}
This process is carried out until a given integration time $T$ or a maximum number of layers $N$. Despite their difference, both models have the same functional form, obtained by multiplying the initial signal by a diffusion matrix $\mathcal{D}$ given by either the exponential of the sheaf laplacian or its powers depending on whether we are dealing with the continuous or the discrete case. Hence, from here on out we will simply write:
\begin{equation*}
    x^t = \mathcal{D} x^0
\end{equation*}
This concludes the introduction to cellular sheaves and SD. We are now ready to integrate this framework in the setting of fair ML.

\section{Fair Sheaf Diffusion}
\label{sec:SheafModels}
This Section introduces the models that will integrate the ideas seen in the previous Section together. We begin by presenting two sheaves that fulfill different roles and can be integrated at various stages of the ML pipeline, illustrating the method’s flexibility. Next, we introduce a series of network topologies that enforce different fairness constraints and describe how to combine them to obtain a model satisfying all metrics.

\subsection{Sheaves}
\label{subsec:Sheaves}
We start by introducing two different hand-crafted sheaves which can be used to achieve fairness. Before delving into the concrete sheaves, it is important to keep in mind that we will always train logistic regression models on the covariates:
\begin{equation*}
    z_i = \beta_0 + \sum_{j = 1}^d \beta_j x_{i,j} = x_i \cdot \beta  \quad \text{and} \quad \hat{y}_i = \frac{1}{1+e^{-z_i}},
\end{equation*}
where $\hat{y}$ are the predictions, $z$ are the logits and $x$ are either the raw covariates or the transformed covariates after applying SD. We are now ready to explain the sheaves used.\\
First, consider the \textbf{identity sheaf} $\mathcal{F}^{\mathbbm{1}}$. This sheaf is obtained by assuming all nodes and edges stalks have the same dimension, $\mathcal{F}^{\mathbbm{1}}(v) \cong \mathcal{F}^{\mathbbm{1}}(e) \cong \mathbb{R}^m$, and all transition maps are proportional to the identity. That is, we allow the free exchange of information between nodes. This process can be applied on both the initial covariates as a way of pre-processing to obtain unbiased covariates or on the final predictions as a form of post-processing. In any case, the corresponding sheaf laplacian is given by $L^{\mathbbm{1}} =  L^w \otimes \mathbbm{1}$, where $L^w$ is the weighted graph laplacian with weights determined by restriction maps, and the dimension of the identity matrix is $d$ if the process is applied during pre-processing or $1$ if used during post-processing.\\
The choice between the pre- and the post-processing versions is trivial when using a linear regression model as we are, and it boils down to the specific logistics and necessities of practitioners involved. This is a consequence of the linearity of the operators involved. On the one hand, the expression for the pre-processing version is given by:
\begin{equation}
    x_{i}^{dif} = \mathcal{D} x_i = \sum_{j} \mathcal{D}_{ij} x_j \Rightarrow z_{i} = x_{i}^{dif} \cdot \beta = \beta_0 + \sum_{j, k} \mathcal{D}_{ij} x_{jk} \beta_k,
    \label{PRE}
\end{equation}
while in post-processing the linear function is applied first and it is the diffused:
\begin{multline}
    z_{i} = \beta_0 + \sum_{k} x_{ik} \beta_{k} \Rightarrow\\ z_{i}^{dif} =  \mathcal{D} z_i = \sum_{j} \mathcal{D}_{ij} z_j = \beta_0\sum_{j} \mathcal{D}_{ij} + \sum_{j,k} \mathcal{D}_{ij} x_{jk} \beta_k .
    \label{POST}
\end{multline}
Therefore, modulo the intercept, both versions have the same functional form, differing only in their dimensions. As a consequence, we will use the post-processing version due to its lower computational expense.\\
We now consider the \textbf{vector sheaf}, whose vector stalks now model the covariates, $\mathcal{F}(v) \cong \mathbb{R}^d$ with $d$ the number of covariates, while introducing a bottleneck on the common edge stalks, which are now univariate in order to model the output of a binary classification model, $\mathcal{F}(e) \cong \mathbb{R}$, and the restriction maps are given by the vector of coefficients $\beta \in \mathbb{R}^d$ of a linear model, $\mathcal{F}_{v\leqt e} x_v = w_e \beta \cdot x_v$. Therefore, the sheaf laplacian is given by $L^{\beta} =  L^w \otimes (\beta \beta^{\top})$. The idea behind this sheaf is to leverage the sheaf structure by introducing an asymmetry between nodes and edge stalks which serves as bottleneck restricting the exchange between nodes to information about the logits. Borrowing an analogy from opinion dynamics, each individual now only communicates their opinion on one singular topic of great importance, like a war or election, but they are able of changing their world-view due to the ideas of their peers, like one family member changing their mind after being singled out in a family dinner. This differs from the pre-processing identity sheaf, in which they shared their whole worldview, and from post-processing, in which they were only able to change their final scores. This procedure instead incorporates the linear model into the sheaf diffusion process, hence creating a unique in-processor.\\
Expanding on equations \eqref{PRE} and \eqref{POST} and in analogy to \eqref{SHAPEQ}, we can extract a closed-form expression for both the effect of a variable $k$ on an observation $i$, $\beta_{ki}^{ef}$ and its shapeley value $\phi_{ik}$:
\begin{equation}
    \beta_{ki}^{ef} = \beta_k \sum_{j} \mathcal{D}_{ij} \quad \text{and} \quad \phi_{ik} = \sum_{j} \mathcal{D}_{ij} \beta_k (x_{jk} - \mathbb{E}(x_k) ) .
    \label{IMPORTANCE}
\end{equation}
Although the interpretation of these quantities is not straightforward due to the presence of neighbors, one could explain this away by thinking that the effect of a variable on each observation must include information from its neighbors via an average weighted by the diffusion matrix. Another interpretation of the diffusion matrix elements could be that their inclusion represents a factor on the coefficients whose overall effect is of the order of $\frac{1}{n} \sum_{i,j} D_{ij}$, thus discounting or exacerbating the effect of a variable on a given observation. In any case, it is clear that another strength of the proposed models lies in their interpretability and explainability, fulfilling another objective of responsible AI.\\
Recapitulating, we have found ourselves with a model which can be integrated in any stage of the ML pipeline. The identity sheaf can be incorporated both during pre- and post-processing, while the vector sheaf introduces a new in-processor altogether. Moreover, the concatenation of sheaf diffusion processes is an inherently hybrid fairness method combining different kinds of processors, being, to our knowledge, the first of its kind which will be the object of future study. All relevant design choices, like whether to use a continuous or discrete model, or which hyper-parameters are available to the user, are summarized in Table \ref{tab:Choices}.

\begin{table}[ht!]
    \centering
    \begin{tabular}{cc}
    \toprule
         Choice & Description \\
         \midrule
         $\alpha$ & Strength of the sheaf diffusion process. \\
         $n$ or $t$ &  Number of layers or integration time. \\
         $\mathcal{F}^{\mathbbm{1}} $ vs. $\mathcal{F}^{\beta} $ & Type of sheaf considered. \\
         DISC vs. CONT & Using a discrete or continuous implementation.   \\
 \bottomrule
    \end{tabular}
    \caption{Considered design choices for Sheaf Diffusion Models.}
    \label{tab:Choices}
\end{table}

\subsection{Towards topological fairness}
\label{subsec:TopologicalFairness}

This Section introduces a collection of network topologies capable of modeling a myriad of fairness constraints in the kernels of their laplacians. In particular, we will remember the expression of the sheaf laplacian kernel and expand on this definition to incorporate different fairness metrics, thus completing the analysis of sheaf diffusion and fairness.

\subsubsection{Fairness induced by a graph}
\label{subsubsec:Intrinsic}
The most natural notion of fairness in graphs emerges from the underlying edges: when they encode similarity, it stands to reason that a fair model should produce similar predictions for neighboring nodes \citep{INDIVIDUALGRAPH}. As we have previously seen, SD encourages solutions close to the sheaf laplacian kernel, which satisfies the following equation:

\begin{equation*}
    \sum_{u \sim v} \mathcal{F}_{v\leqt e}^{\top} \mathcal{F}_{v\leqt e} x_u  - \mathcal{F}_{v\leqt e}^{\top} \mathcal{F}_{v\leqt e} x_v = 0,
\end{equation*}
if we apply, say, the identity sheaf during pre-processing:

\begin{equation*}
    \sum_{u \sim v} c_{u, e} x_u  - c_{v,e} x_v = 0,
\end{equation*}
where $c_{u,e}$ denotes the proportionality constant of the restriction map $\mathcal{F}_{u\leqt e} \equiv c_{u, e} \mathbbm{1}_d$. Now apply a linear function:
\begin{equation*}
    \sum_{u \sim v} c_{u, e} y_u  - c_{v,e} y_v = 0,
\end{equation*}
thus achieving similar results for neighboring nodes. This idea provides the basis for encoding fairness constraints in graph structures, and serves as a motivation for Definition \ref{def:fair_sheaf}.

\begin{definition}[Fair Sheaf Diffusion] Let $G$ be an undirected graph, $g_{fair}(f, D, A) = 0$ a set of fairness constraints possibly dependent on a sensitive attribute $A$. A \emph{fair sheaf} is a cellular sheaf $\mathcal{F}$ on $G$ whose kernel contains the fairness constraints. The diffusion process arising from such a sheaf will receive the name of \emph{Fair Sheaf Diffusion} (FSD).   
\label{def:fair_sheaf}
\end{definition} 
\noindent However, this notion is very limited due to the abundance of datasets without an underlying graph. Moreover, when dealing with disconnected graphs, like in Figure \ref{fig:EJEMPLO}, there can be separated communities without interaction, limiting the reach of the graph method. This why these ideas need to be extended to overcome these limitations.

\begin{figure}
    \centering
    \includegraphics[width=0.33\linewidth]{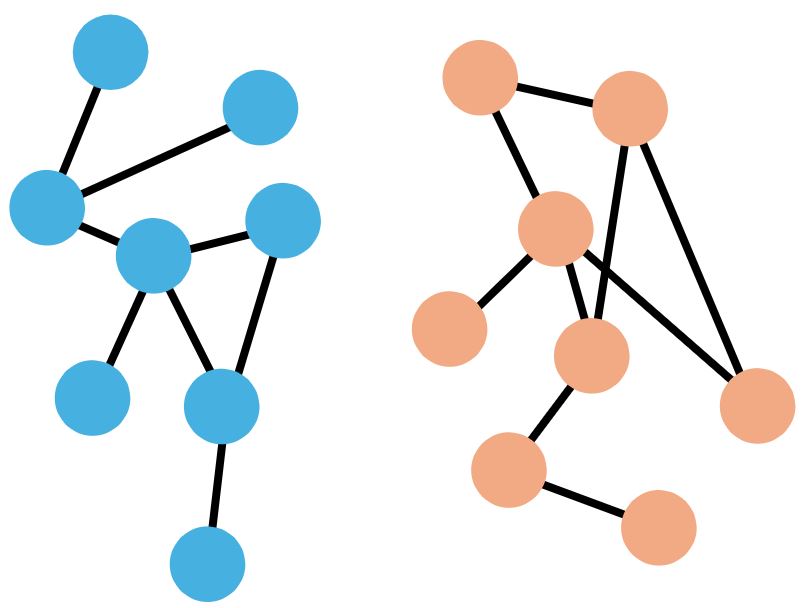}
    \caption{Graph with two disjoint communities. Sheaf Diffusion is unable to reconcile fairness metrics between individuals belonging to different communities.}
    \label{fig:EJEMPLO}
\end{figure}

\subsubsection{Subset topology}
\label{subsubsec:Subset}

To address both group fairness and isolated communities we propose the use of the subset topology, a global graph structures aimed at connecting different subsets of the dataset. This topology requires a collection of subsets and relations between them, adding a set of virtual nodes which serve as representatives of their respective subsets. All observations belonging to a given subset are connected to their representative, and the previously mentioned set of relationships between subsets indicates how representatives should be connected together, hence allowing communication between groups. Note that the relations between subsets represent which sectors of the population should have similar results, although in general we are interested in dense relationships between subsets, like it is the case for partitions derived from a single sensitive attribute whose influence we want to minimize. In analogy  Opinion Dynamics, one can think of the global topology as a mechanism by which members of a (possibly isolated) community are made aware of the current zeitgeist, like social media, news outlets or hearsay. These ideas are formalized through the following definition: 

\begin{definition} Let $D = (X, Y, A)$ be a dataset and $\mathcal{C = }\{D_i\}_{i=1}^n \subset \mathcal{P}(D)$ be a finite collection of subsets of $D$. A graph $G = (\mathcal{C}, E)$ on $\mathcal{C}$ induces a graph $G_{\mathcal{C}} = (V_{\mathcal{C}}, E_{\mathcal{C}})$ where $V_{\mathcal{C}} = D \cup \mathcal{C}$ and $E_{\mathcal{C}} = \{(v, D_j)| D_j\in \mathcal{C},  v \in D_j  \} \cup E$. 
\end{definition}
\noindent That is, the global subset topology is obtained by adding $n$ virtual nodes. The edges of the global topology are those of the graph on $C$, which represent relations between subsets, plus new edges that join every individual to its representative. A set of examples should clarify this notion, and the companion Figures \ref{fig:GLOBALS1} to \ref{fig:GLOBALS3} will help cement these concepts:

\begin{example}
Suppose a partition $P= \{P_i\}_{i=1}^n \subset \mathcal{P}(D)$, that is, $D$ is the union of elements of $P$ and all elements of $P$ are disjoint. Denote by $G^{full}_{P}$ the fully connected graph on $P$. Then the subset topology induced by $P$ is comprised of a set of virtual nodes that aggregate the information of all the observations of the subset they represent, and all virtual nodes communicate with each other. In this case we will talk about the \emph{partition topology or graph}. When this partition is determined by a sensitive variable we will prove it ensures independence with respect to said variable. Figures \ref{fig:GLOBALS1} and \ref{fig:GLOBALS3} show the partition topology for partitions of the underlying set into two and four groups. One could think of Figure \ref{fig:GLOBALS1} as the graph emerging from a binary sensitive variable, like gender, while Figure \ref{fig:GLOBALS2} arises from considering a more general categorical sensitive attribute, like race or religious identity.
\end{example}
\begin{example}
Let $P^j= \{P^j_i\}_{i=1}^n \subset \mathcal{P}(D)$ with $j=1,...,n_j$ be a set of partitions, $G_{\cup P^j}$ the the graph given by the disjoint union of fully connected graphs on each partition. Then the resulting diffusion process aims at making representations independent of each partition. For example, when we have a set of sensitive variables such as gender indentity, religion, race... each variable induces a partition on the original dataset. Applying SD on the global topology determined by the disjoint union of these partitions results in independence to all variables at the same time. This situation is ilustrated by Figure \ref{fig:GLOBALS3}, in which there are two binary sensitive attributes represented by the shape and color of the nodes. Each variable induces a partition of the underlying set, and all nodes are connected to two representatives.
\end{example}
\begin{example}
Given the total set $D = \{D\} \subset{\mathcal{P}(D)}$ the resulting global topology is the star graph. 
\end{example}
\begin{figure}[ht!]
    \begin{subfigure}{.49\linewidth}
    \centering
    \includegraphics[width=.5\linewidth]{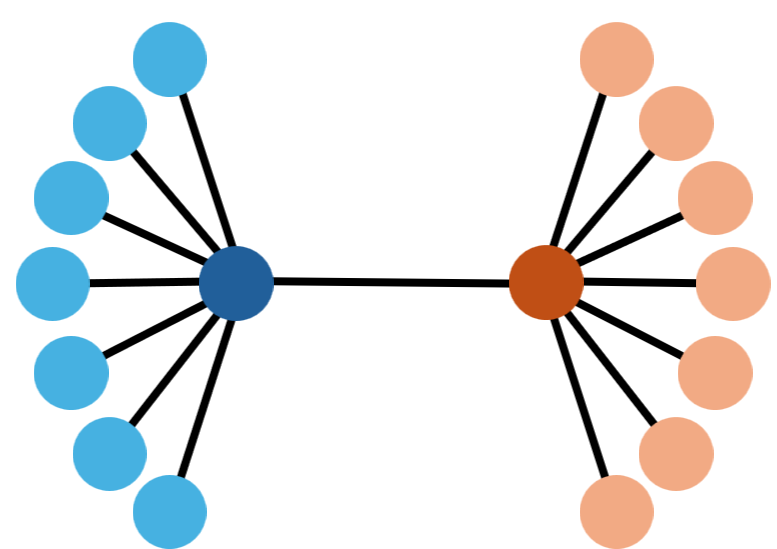}
    \caption{A partition with two subsets.}
    \label{fig:GLOBALS1}
    \end{subfigure}
    \begin{subfigure}{.49\linewidth}
    \centering
    \includegraphics[width=.5\linewidth]{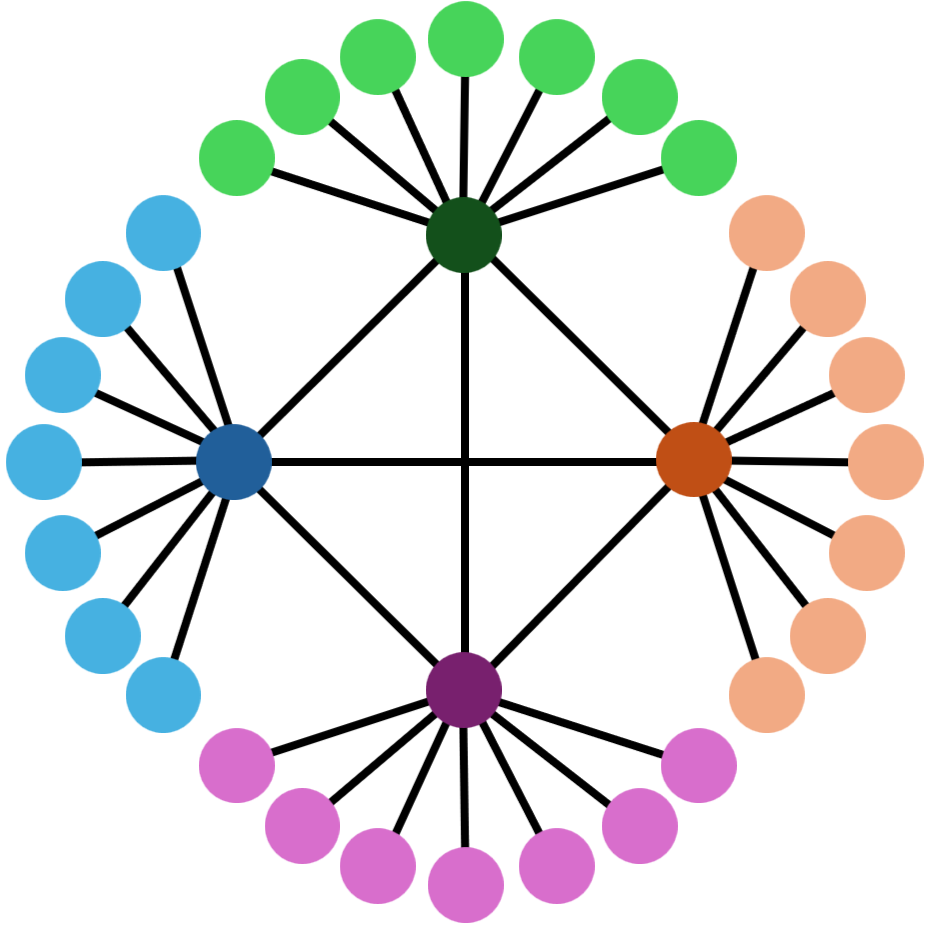}
    \caption{A partition with four subsets.}
    \label{fig:GLOBALS2}
    \end{subfigure}
    \\[1ex]
    \begin{subfigure}{\linewidth}
    \centering
    \includegraphics[width=0.25\linewidth]{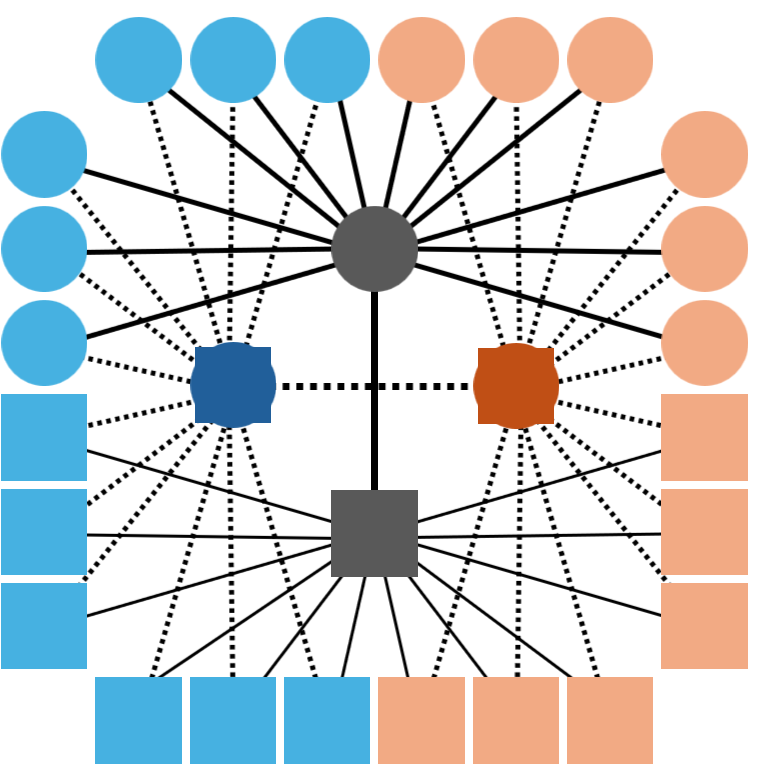}
    \caption{A partition two independent partitions with two subsets each.}
    \label{fig:GLOBALS3}
    \end{subfigure}
    \caption{Examples of the subset topologies. The top row represents two partition topologies with different number of elements, while the bottom row showcases a more exotic configuration with four pairwise independent agregators.}
    \label{fig:GLOBALS}
\end{figure}

\noindent Before starting with the theoretical analysis of the global topology, let us introduce the notation that will follow. Assume a dataset $D$, a collection of subsets $\mathcal{C} \subset \mathcal{P}(D)$, a graph $G$ on $\mathcal{C}$ and a sheaf $\mathcal{F}$ on $G_{\mathcal{C}}$. Signals $x\in \cochain{0}$ will be denoted by $x_v^j$ for a node $v$ in $D_j \in \mathcal{C}$, while the signal of the virtual node representing $D_j$ will be denoted by $x_j$. Finally, let $\mathcal{N}(v)$ denote the set of neighbors of $v$ in a graph $G$ and $\mathcal{N}(v,A)$ denote the set of neighbors of $v$ in a graph $G$ that lie on a set $A$; formally, $\mathcal{N}(v,A) = \mathcal{N}(v) \cap A$. Then the kernel of the sheaf laplacian on one of the virtual node representing $D_j$ is given by:

\begin{multline*}
    \sum_{v\in D_j} \mathcal{F}_{j\leq e_v}^{\top} \mathcal{F}_{j\leqt e_v} x_j - \mathcal{F}_{j\leqt e_v}^{\top} \mathcal{F}_{v\leqt e_v} x_v^j \\ + \sum_{D_i\in \mathcal{N}(D_j, \mathcal{C})} \mathcal{F}_{j\leqt e_i}^{\top} \mathcal{F}_{j\leqt e_i} x_j - \mathcal{F}_{j\leqt e_i}^{\top} \mathcal{F}_{i\leqt e_i} x_v^j = 0
\end{multline*}
The first term of the sum represents the aggregation of all the nodes that belong to $D_j$, while the second term encourages the similarity between neighboring representatives. Assume now identical restriction maps on edges given by a common linear map $\mathcal{F}_{\mathcal{C}}$ if they are in $E_{\mathcal{C}}$ or by $\frac{1}{\sqrt{|D_j|}} \mathcal{F}_{j}$ if they connect an observation in $D_j$ with its corresponding representative. Then the expression for the kernel becomes:

\begin{equation*}
    \mathcal{F}_{j}^{\top} \mathcal{F}_{j} \left( x_j - \sum_{v\in D_j} \frac{1}{|D_j|} x_v^j\right) + \mathcal{F}_{\mathcal{C}}^{\top} \mathcal{F}_{\mathcal{C}} \left( \sum_{D_i\in \mathcal{N}(D_j, \mathcal{C})}x_j - x_i\right) = 0
\end{equation*}
Finally, assuming the identity sheaf:
\begin{equation*}
    x_j - \frac{1}{|D_j|} \sum_{v\in D_j}x_v^j + \sum_{D_i\in \mathcal{N}(D_j, \mathcal{C})} (x_i - x_j) = 0
\end{equation*}
Which implies the following expression for the average signals on each subset:
\begin{equation*}
    \frac{1}{|D_j|} \sum_{v\in D_j}x_v^j = x_j  + \sum_{D_i\in \mathcal{N}(D_j, \mathcal{C})} (x_i - x_j) 
\end{equation*}
Consider two neighboring subsets $D_i, D_j$ and take the difference of the previous expression. The result is:

\begin{multline*}
        \frac{1}{|D_i|} \sum_{v\in D_i}x_v^j - \frac{1}{|D_j|} \sum_{v\in D_j}x_v^j \\ = x_i  + \sum_{D_k\in \mathcal{N}(D_i, \mathcal{C})} (x_k - x_i) - x_j -  \sum_{D_k\in \mathcal{N}(D_j, \mathcal{C})} (x_k - x_j)
\end{multline*}
Remember that the kernel of the sheaf laplacian coincides with the kernel of the coboundary operator, hence for the aggregator nodes we have:
\begin{equation*}
    \mathcal{F}_{i \leqt e} x_i = \mathcal{F}_{j \leqt e} x_j
\end{equation*}
In our particular case we have equality of signals for neighboring aggregator nodes, which implies equality of the subset average. When using a fully connected graph on the partition induced by a sensitive attribute and applying a linear model on the resulting signal, the subsequent scores are independent of the sensitive attribute. Therefore, both the identity and vector sheaves lead to independence when using the subset topology.\\
One noteworthy case is that of the star configuration. In this case the only virtual node stores the average signal in the time limit and the minimization of Dirichlet energy implies the minimization of the mean squared distance of the signal to its average; that is, its variance. This is close to minimizing squared entropy but not exactly. Furthermore, minimization of the signal does not necesarily imply minimization of the benefit function which is not lineal. It could be the case that non-linear SD might help with this goal, but this approach is outside the scope of the current work.

\subsubsection{k-Nearest Neighbor topology}
\label{subsubsec:kNN}

In order to build intuition for local graphs we might draw inspiration once more from Opinion Dynamics. Basically, local topologies represent organic communication channels that arise from similar persons. After all, it is easier for two particular individuals to communicate if they have something in common, like sharing the same hobby or job. Moreover, it is not hard to imagine that the opinion of a person might be more influenced by their inner circle, which is usually composed on some level of similar-minded people. In any case, the exchange of information in this situation is conditioned by similarity between agents, thus focusing more on individual fairness rather than group equity. In particular, the $k$-Nearest Neighbor graph is inspired by kNN and the consistency metric, attempting to encode it in its kernel:

\begin{definition} Let $D = (X, Y, A)$ be a dataset, $d$ a distance, $k$ a natural number. The $k$-Nearest Neighbor (kNN) graph is the graph $G_k = (V_k, E_k)$ where $V_k = D$ and $(v_1,v_2)\in E_k$ if and only if $v_1 \in \mathcal{N}_{d}^k(v_2)$ or $v_2 \in \mathcal{N}_{d}^k(v)$ where $\mathcal{N}_{d}^k(v)$ denotes the set of the closest $k$ observations to $v$ according to metric $d$. More succintingly, $E_k = \bigcup_{v\in D}\{ (v,u)| u\in\mathcal{N}_d^k(v)\}$.   
\end{definition}
\noindent It is straightforward to check that this metric encourages consistency. For a given node $v$ the kernel  takes the form:
\begin{equation*}
    \sum_{u \in \mathcal{N}_k^d(v)} \mathcal{F}_{v\leqt e_u}^{\top} \mathcal{F}_{v\leqt e_u} x_v - \mathcal{F}_{v\leqt e_u}^{\top} \mathcal{F}_{u\leqt e_u} x_u = 0 .
\end{equation*}
The lack of privileged nodes (aggregators) makes the analysis straightforward: The kNN configuration aggregates the signal over the $k$ closest individuals to $v$, thus encouraging consistency. In any case, assuming a constant restriction map:
\begin{equation*}
    \mathcal{F}^{\top} \mathcal{F}\left(x_v - \frac{1}{k}\sum_{u\in \mathcal{N}_d^k(v)} x_u\right) = 0
\end{equation*}
If the restriction map is proportional to the identity:

\begin{equation*}
    x_v - \frac{1}{k}\sum_{u\in \mathcal{N}_d^k(v)} x_u = 0
\end{equation*}
Therefore, the use of a linear model results in a consistent score function.

\subsubsection{Unit ball topology}
\label{subsubsec:UnitBall}
The last topology is also local in nature, although instead of smoothing the signal over the closest neighbors it instead aims to achieve equity inside of a ball of fixed radius $\delta$, the idea being to achieve a small local lipzchitz constant with the goal of minimizing the $LIP$ metric:

\begin{definition}[Unit ball graph] Let $D = (X, Y, A)$ be a dataset, $d$ a distance, $\delta > 0$. The unit ball graph with radius $\delta$ is the graph $G_{\delta} = (V_{\delta}, E_{\delta})$ where $V = D$ and $(v_1,v_2)\in E$ if and only if $d(v_1, v_2) < \delta$. More succintingly, $E_{\delta} = \bigcup_{v\in D} \{(v,u)| u\in \mathcal{N}_d(v, \delta)\}$ where $\mathcal{N}_d(v, \delta) = \{u\in D| d(x_u, x_v) < \delta\}$ is the set of observations whose signals are inside a ball of radius $\delta$ centered in $x_v$.
\end{definition}
\noindent This case is similar to the kNN graph. The kernel constraint for a node $v$ takes the form:
\begin{equation*}
    \sum_{u \in \mathcal{N}_d(v, \delta)} \mathcal{F}_{v\leqt e_u}^{\top} \mathcal{F}_{v\leqt e_u} x_v - \mathcal{F}_{u\leqt e_u}^{\top} \mathcal{F}_{u\leqt e_u} x_u = 0.
\end{equation*}
The interpretation is the same as in the previous local topology, although this time the aggregation takes place inside a ball of radius $\delta$. Assuming a constant restriction map:
\begin{equation*}
    \mathcal{F}^{\top} \mathcal{F}\left(x_v - \sum_{u\in \mathcal{N}_d(v, \delta)} x_u\right) = 0,
\end{equation*}
and further assuming that said map is the identity:
\begin{equation*}
    x_v - \frac{1}{|\mathcal{N}_d(v, \delta)|}\sum_{u\in \mathcal{N}_d(v, \delta)} x_u\ = 0,
\end{equation*}
hence serving as a smoothing mechanism inside the ball of radius $\delta$ and center $x_v$. Once again, if a linear function is applied the distance of the scores whose signals are inside said balls is minimized, which should reduce the local Lipschitz constant. Although the transition from the local Lipschitz constant to the global $LIP$ metric is not immediate, it should become effective in practise assuming compactness. After all, the denominator dominates for observations far apart, making $LIP$ small, while close observations should lie inside the same ball, in which case the numerator is minimized, thus reducing $LIP$. There is a trade-off involved in the choice of $\delta$: the bigger its value the more costly it is to compute the unit ball graph, but the smaller $\delta$ the less representative the balls. This is easily verified in the degenerate case where $\delta$ is less than the minimum distance observed, resulting in an empty graph. Furthermore, another pitfall of the method is that observations that lie close together are weighted equally as those that are on the boundary of the ball, although their effect on the Lipschitz constant is not the same. One possible remedy might be to to weight edges by the inverse of the distance between their endpoints. In fact, adopting this weighting scheme on a fully connected graph of the whole dataset should truly minimize $LIP$, but of course this procedure is unwieldly in practise for large datasets.\\
Before concluding the discussion on fair topologies, let us show how to combine different configurations to achieve combined fairness.

\subsubsection{Combining sheaves}
\label{subsubsec:Combining}
This Section has shown how different graph configurations can encode different fairness constraints. This result, while interesting, begs the question of how to combine them to achieve a compromise between the metrics entailed by different graphs. Our answer is inspired by the PageRank algorithm \citep{PAGERANK}, which allocates $\gamma$ probability to random jumps outside the neighborhood of a point of the graph (global process arising from a fully connected graph) and $1-\gamma$ to random graph diffusion (local process arising from the graph diffusion equation). Using a similar reasoning we arrive at the following definition:

\begin{definition} Let $D = (X, Y, A)$ be a dataset with a set of topologies, $\{G_i\} = \{(V_i, E_i)\}$ with compatible sheaves $\mathcal{F}_i$. and a set of strictly positive coefficients $w_1, ..., w_n > 0$. The linear combination of cellular sheaves is denoted by $\mathcal{F} = \sum_{i=1}^n w_i \mathcal{F}_i$ which is defined on the graph given by $G= (V, E)$ where $V = \bigcup_{i=1}^n V_i$, $E = \bigcup_{i=1}^n E_i$ and its restriction maps are given by:

\begin{equation*}
    \mathcal{F}_{u\leqt e} = \sum_{i=1}^n \sqrt{w_i} \tilde{\mathcal{F}}_{i, u \leqt e}
\end{equation*}
where $\tilde{\mathcal{F}}_{i, u \leqt e}$ is equal to $\mathcal{F}_{i, u \leqt e}$ if $e\in E_i$ and zero otherwise. The resulting sheaf laplacian is given by the linear combination of the sheaf laplacians:

\begin{equation*}
    \sheaflaparg{\mathcal{F}} = \sum_{i = 1}^n w_i\sheaflaparg{\mathcal{F}_i}
\end{equation*}
\end{definition}
\noindent The use of a linear combination of sheaves guarantees convergence to the intersection as per the following lemma:

\begin{lemma}
    Let $A_1,..., A_n$ be a set of square symmetric semipositive definite matrices. Consider a set of strictly positive coefficients $w_1,...,w_n > 0$. Then the linear combination $A = \sum_{i = 1}^n w_i A_i$ is symmetric semi-positive definite and $\ker A = \bigcap_{i= 1}^n \ker A_i$.
\end{lemma}

\begin{proof}
    The fact that $A$ is symmetric semipositive definite is immediate. Let us prove that its kernel is the intersection of the kernel of the other matrices. It is clear that $\bigcap_{i= 1}^n \ker A_i \subset \ker A$. Let us prove the reciprocal assertion: Suppose $x \in \ker A$. Then
    \begin{equation*}
        A x = 0 \Longrightarrow \sum_{i=1}^n w_i A_i x = 0,
    \end{equation*}
    multiplying by $x^{\top}$ on the left,
    \begin{equation*}
        \sum_{i=1}^n w_i x^{\top}A_i x = 0.
    \end{equation*}
    The fact that each matrix is semidefinite positive implies that all terms of the sum are greater than or equal to zero. Therefore, the only way of it becoming null is that each term is zero, that is
    \begin{equation*}
        x^{\top}A_i x = 0 \quad  \text{for all  } i = 1,...,n.
    \end{equation*}
    Hence, $x$ belong to the kernel of all matrices and hence to their intersection.
\end{proof}
\noindent Therefore, the diffusion process originating from the linear combination of a set of fair sheaves orthogonally projects the original signal onto the intersection of the kernels of all sheaf laplacians, hence satisfying the fairness constraints induced by all of them. This allows the treatment of different fairness problem with one single model.\\
This concludes the discussion for fairness-inducing graph configurations. A summary can be found in Table \ref{tab:Graphs} which keeps track of all the proposed topologies along with their defining parameters and the metrics they encode. 

\begin{table}[ht!]
    \centering
    \begin{tabular}{c c c}
    \toprule
         Graph & Hyper-parameter & Metric \\
         \midrule
         Subset & $\mathcal{C} \in \mathcal{P}(D)$ & IND  \\
         kNN &  $k\in \mathbb{N}$ & CON \\
         Unit ball & $\delta \in \mathbb{R}^{+}$ & LIP \\
         Mixed & $w_i \in \mathbb{R}^+$ & Multiple \\ \bottomrule
    \end{tabular}
    \caption{Proposed graph topologies.}
    \label{tab:Graphs}
\end{table}

\section{Experiments}
\label{sec:Experiments}

This Section provides an in-depth explanation of the methodology and experiments followed to explore the properties of FSD models. The analysis is preceded by a thorough discussion of both the methodology followed and the datasets used. Then, an in-depth discussion of the results follows, comprised by a simulation study intended to build intuition and consolidate the theoretical results, a comprehensive analysis of the effect of the different hyper-parameters, and a review of a list of relevant case studies, delving into the fairness-accuracy trade-offs through Pareto frontiers, the results obtained from a grid search, and the interpretability of the models through their SHAP values. The code used to implement this methodology can be found in a Github repository\footnote{\url{https://github.com/arturo-perez-peralta/HighOrderFairness}}.\\

\subsection{Experimental setup}
\label{subsec:Setup}
Before discussing the results, a quick review of the experimental setup ensues. The data generation process of the simulation study is based on the toy model by \cite{ADVERSARIAL} with some adjustements. The training sample will be denoted by $(x_i, y_i, a_i)_{i=1}^n$ where $x_i$ denotes the non-sensitive covariates, $y_i$ denotes the label and $a_i$ is the sensitive attribute. For each $i$, let $a^{i} \in \{0,1\}$ be picked at random with probability $p$ of drawing a one, and let $v_i \sim \mathcal{N}(a_i, 1)$ be an inaccurate measurement of the sensitive variable. Let $u_i, w_i \sim \mathcal{N}(v_i, 1)$ be independent, inaccurate measurements of $v_i$, where $w_i$ serves as a biased feature which contributes linearly to the response while $u_i$ is a proxy for $w_i$. Finally, let $t_i \sim \mathcal{N}(-0.5,1)$ represent a random independent covariate. The final dataset is given by $x_i = (u_i, t_i, a_i)$,  $y_i = \mathbbm{1}(\sigma(0.5 w_i + 0.5 t_i) > 0)$ where $\sigma(x) = 1/(1+e^{-x})$ represents the sigmoid function. The resulting proportion of individuals with a non-zero label are denoted by $q$. With this setting, the simulation is performed $100$ times.\\
On the other hand, we consider a set of usual datasets \citep{UCI} used as benchmarks in the fairness literature:
\begin{itemize}
    \item \textbf{German}\footnote{Source: \url{https://archive.ics.uci.edu/dataset/144/statlog+german+credit+data}}: This dataset is comprised of data belonging to $1000$ individuals who are labeled according to their credit risk. Following \cite{AGE}, we use age as a sensitive variable, considering an individual privileged if he or she is older than $25$.
    \item \textbf{Compas}\footnote{Source: \url{https://github.com/propublica/compas-analysis}}: This dataset provides a list of demographic data of criminal offenders being labeled by whether they recidivated within two years or not. Following \cite{COMPAS}, we use race as a sensitive variable, considering non African-American individuals privileged.
    \item \textbf{Adult}\footnote{Source: \url{https://archive.ics.uci.edu/dataset/2/adult}}: This dataset records demographic data labeled according to whether each individual has an income greater than $50$k USD or not. Following \cite{ADVERSARIAL}, we use gender as a sensitive attribute, considering male individuals privileged.
\end{itemize}
The relevant characteristics for each case are summarized in Table \ref{tab:Datasets}.

\begin{table}[ht!]
    \centering
    \begin{tabular}{c c c c c c}
    \toprule
         Dataset & Size & Features & Label rate & \makecell{Sensitive\\ variable}  & \makecell{Sensitive\\ label rate} \\
         \midrule
         Simulation & $5000$ & $3$ & $\approx0.5$ & $a$ & $\approx0.5$  \\ \hdashline
         German & $1000$ & $61$ & $0.30$ & Age & $0.15$ \\
         Compas & $6172$ & $95$ & $0.46$ & Race & $0.51$ \\
         Adult & $30162$ & $408$ & $0.25$ & Gender & $0.32$ \\ \bottomrule
         
    \end{tabular}
    \caption{Datasets used and their main characteristics.}
    \label{tab:Datasets}
\end{table}
\noindent Finally, the data pipeline starts by splitting a hold-out test set. Then, the rest of the dataset is divided into four train-validation folds on which to perform hyper-parameter tuning and to obtain measures of dispersion. Then, the models are evaluated against the hold-out test, for which the metrics seen in Section~\ref{sec:Theory} are reported, with emphasis on accuracy as a measure of performance and IND, CON and LIP as the measures of fairness targeted by the proposed topologies, with the idea of checking whether or not the use of Fair Sheaf Diffusion truly behaves as expected. 

\subsection{Simulation study}
\label{subsec:Simulation}

The analysis begins by following the simple simulation scheme outlined in Section \ref{subsec:Setup}, running it $100$ times with hyper-parameters shown in Table \ref{tab:HyperParamSimul} found in the \ref{app:training}, along with the hyper-parameter choices for the rest of the experiments. The results are shown in Figures \ref{fig:SimulBoxplots} and \ref{fig:InfluenceSimul}.\\
Let us start with the analysis of the boxplots in Figure \ref{fig:SimulBoxplots}, which compiles the results of $100$ simulations for the $IND$, $CON$ and $ACC$ metrics. Although unreported in this Section, the rest of the fairness metrics shown in Section \ref{sec:Theory} are showcased in Figures \ref{app:FirstBox} to \ref{app:LastBox} in \ref{subapp:simulation}. Starting with the analysis of performance, all models seem to compromise accuracy, with a relative decrease in the median performance ranging from around $13\%$ for the subset topology to $3\%$ for the unit ball graph. On the other hand, $CON$ is improved after using a graph model, with the most resounding success found in both kNN configurations, which results in a relative median decrease in the already low $CON$ of $33\%$, thus cementing the intuition behind these methods. Moving onto group fairness, the results are less satisfactory. For starters, the global topology is the biggest underperformer, despite its theoretical properties. In general, mixed models seem to compromise independence, while purely local topologies have similar readings to logistic regression.

\begin{figure}[!ht]
  \centering
\includegraphics[width=.49\textwidth]{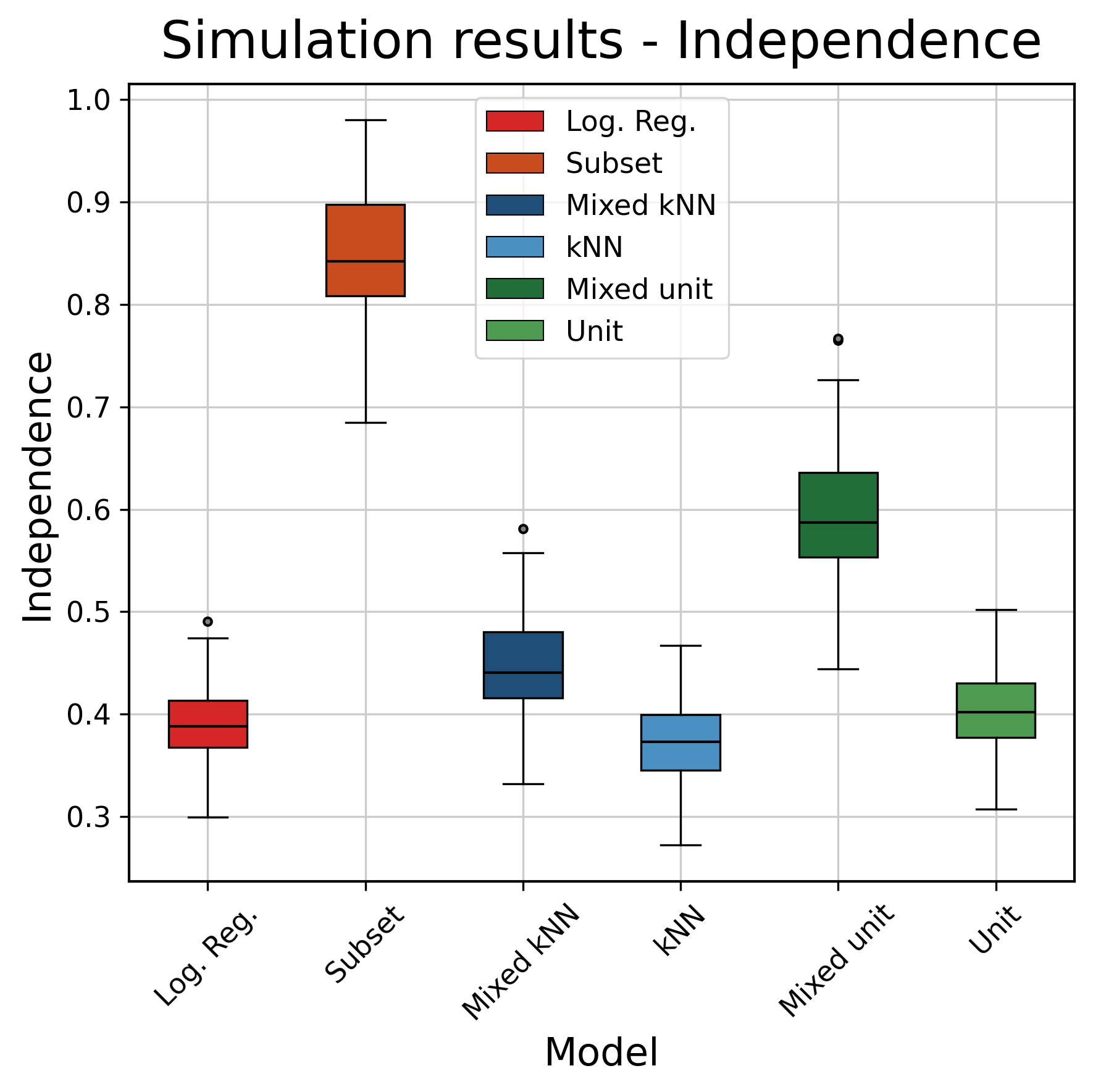} 
 \includegraphics[width=.49\textwidth]{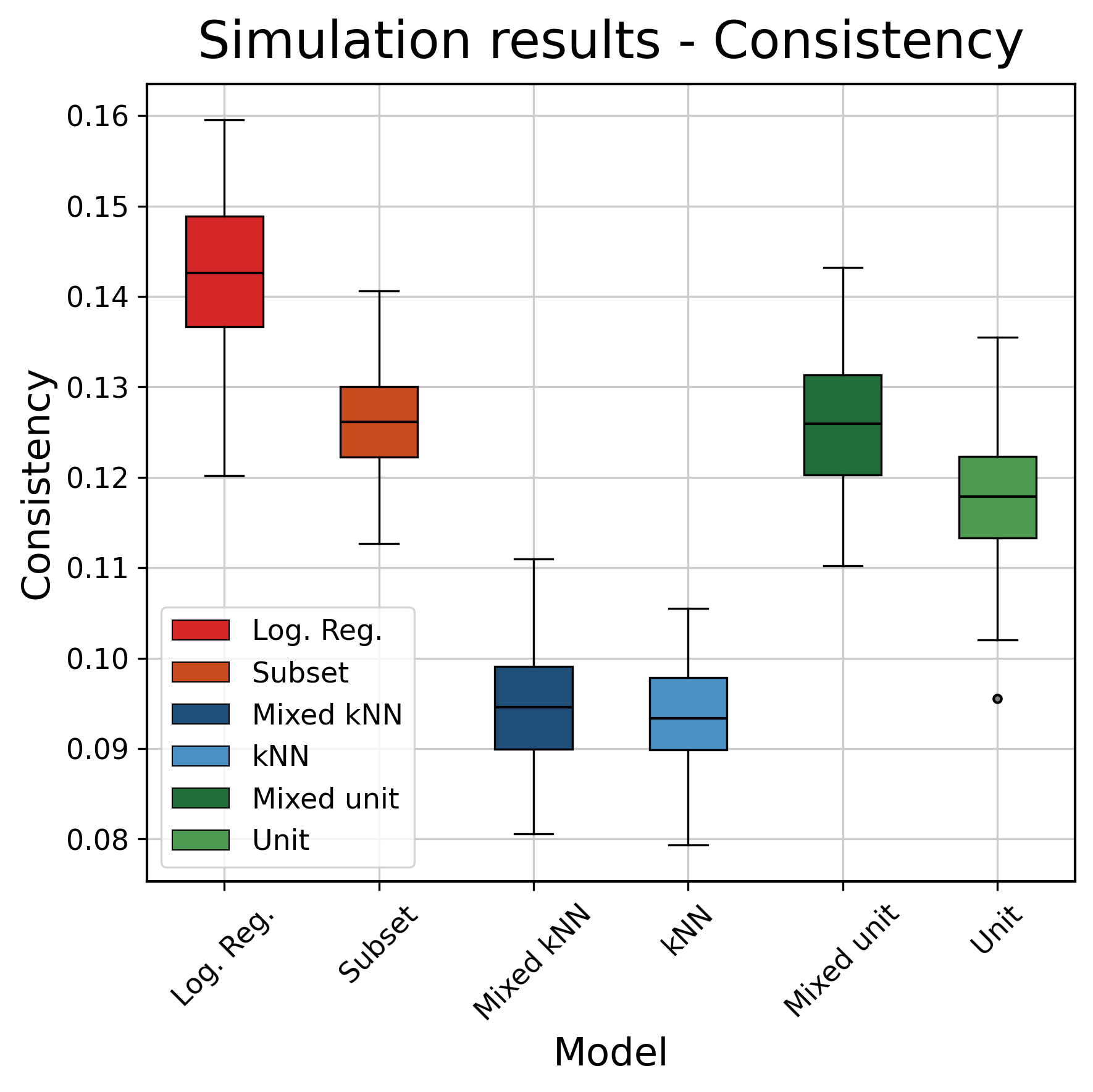}
  \includegraphics[width=.49\textwidth]{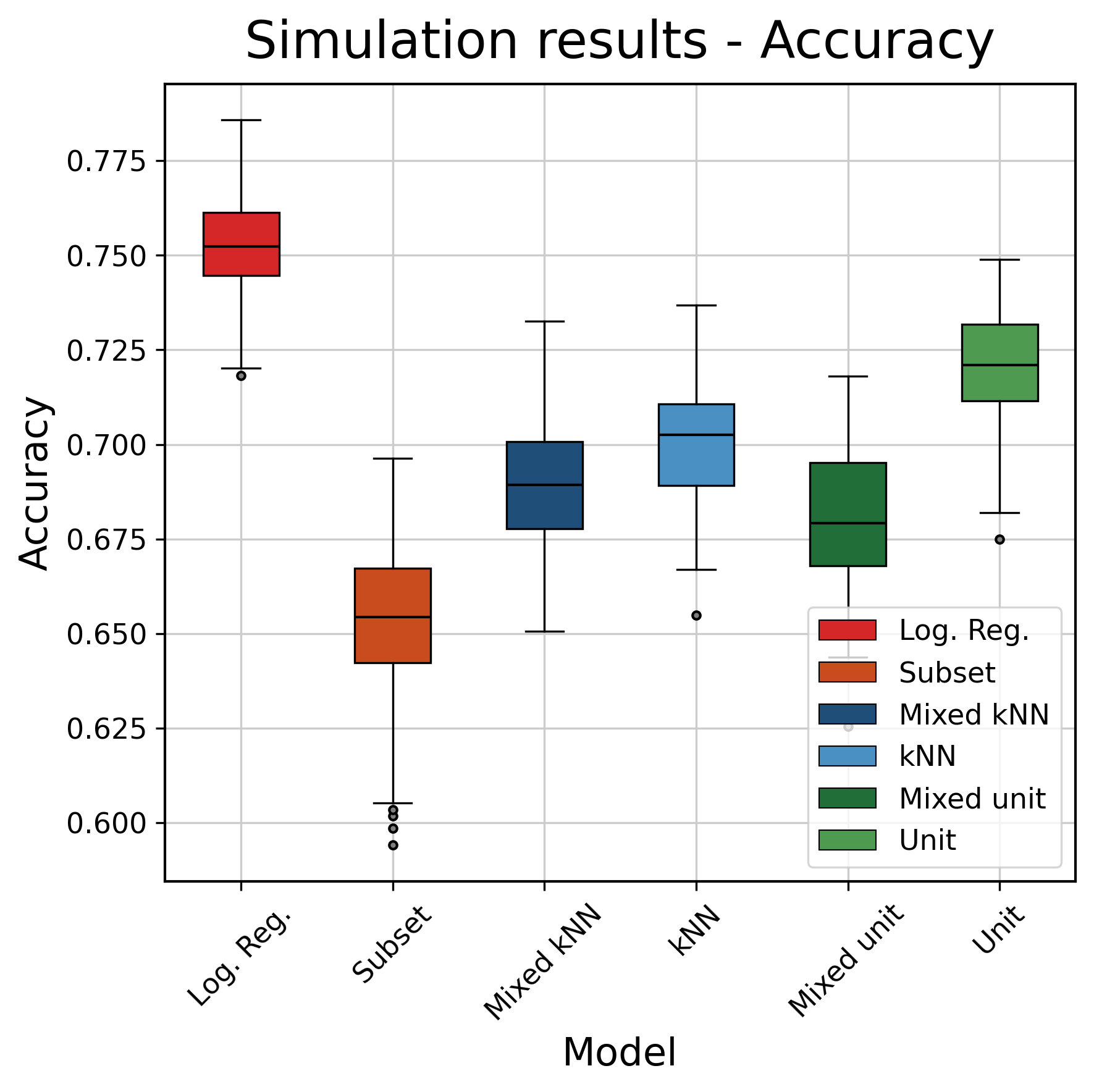}
  \caption{Boxplots with fairness metrics and accuracy for one hundred simulations. The top row displays fairness metrics while the bottom row shows accuracy.}
   \label{fig:SimulBoxplots}
\end{figure}

\begin{figure}[!ht]
    \centering
    \includegraphics[width=0.49\textwidth]{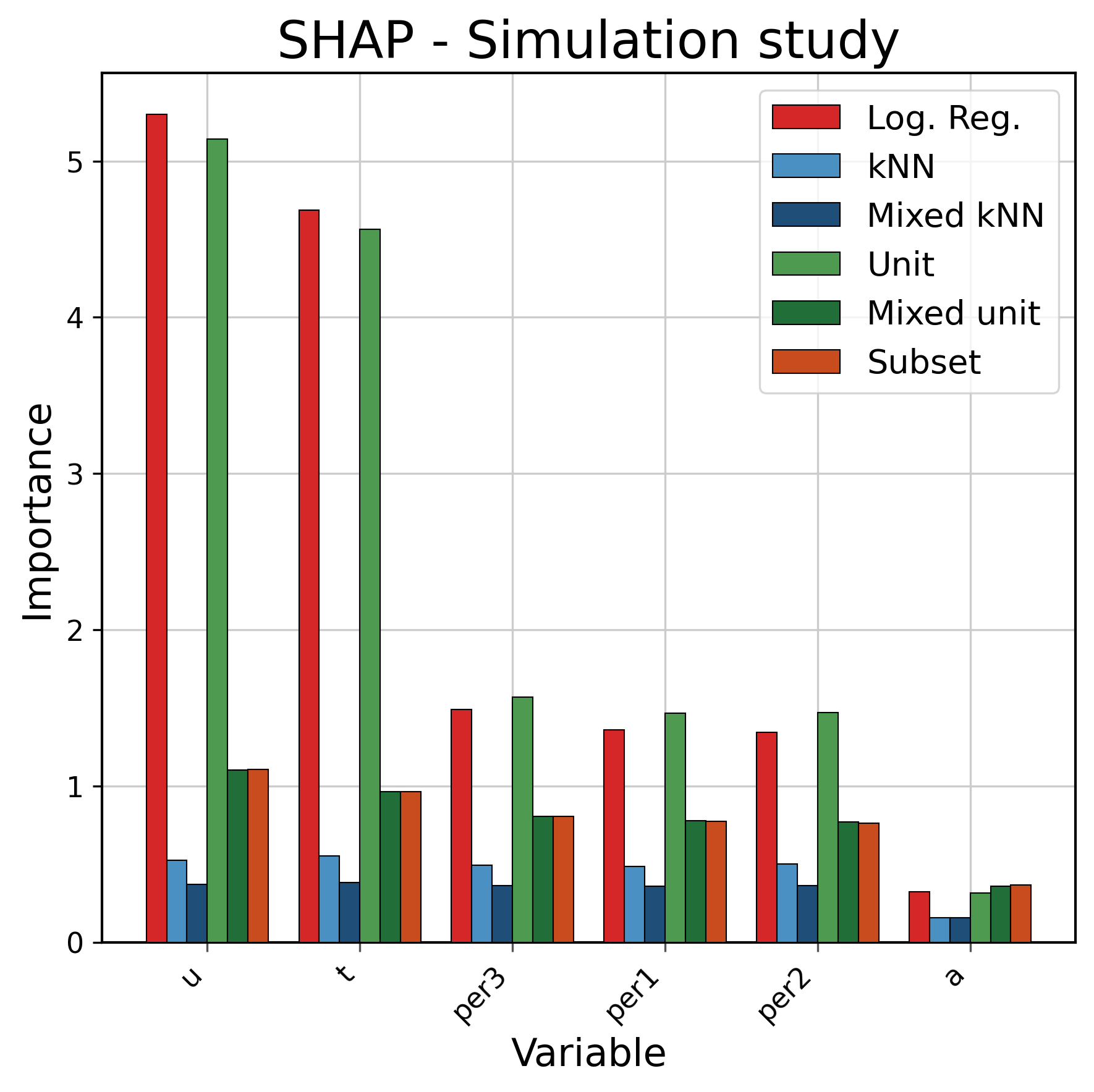}
    \includegraphics[width=0.49\textwidth]{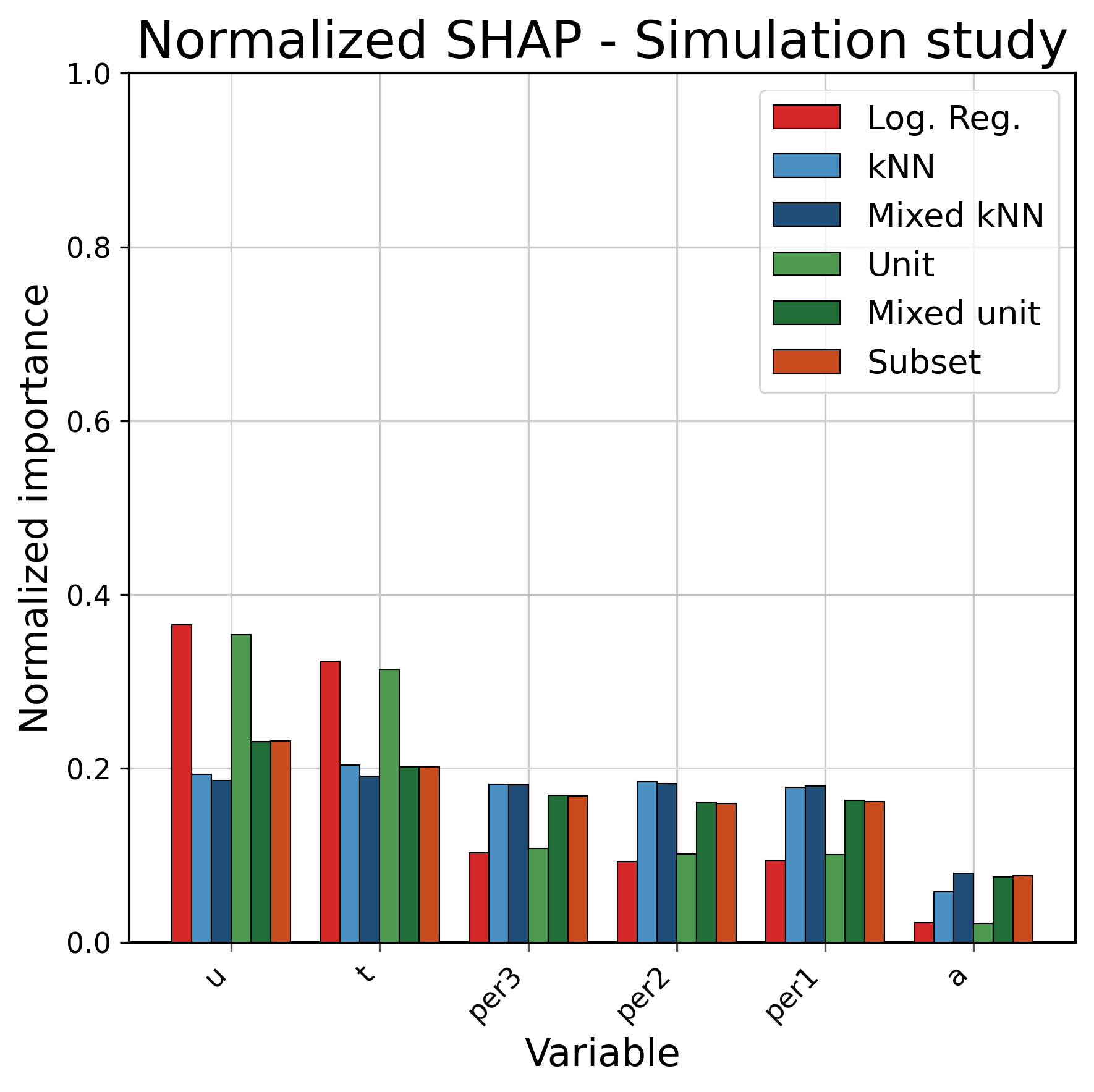}    
    \caption{Average SHAP variable influence of one hundred simulations. The left panel shows the average absolute SHAP contributions while the right panel influences are normalized.}
    \label{fig:InfluenceSimul}
\end{figure}

All in all, it is not a surprise that graph models perform better on individual fairness metrics due to their smoothing properties \citep{NSD}, although we will see in future experiments in Sections \ref{subsec:Design} to \ref{subsec:Cases} how this is a consequence of the underlying models and their associated graphs, whose topologies influence the different performance-fairness trade-offs. However, before delving into a more thorough discussion, a first-order approximation of the behavior of the models can be given by their SHAP values as seen in Figure \ref{fig:InfluenceSimul}, which compares the average variable influence of all simulations. First, it is clear that logistic regression gives more importance to $u$ and $t$ because they are directly related to the response variable, while neglecting the sensitive attribute $a$ and the personal covariates $per_i$. On the other hand, graph models assign less importance to all variables as suggested in Section \ref{subsec:Sheaves}. Moreover, they assign more weight to the personal features, which is key in achieving individual fairness as seen in the boxplots. This is similar to the toy model by \cite{ADVERSARIAL}, where the adversarial debiasing procedure assigned greater weight to the sensitive attribute in order to mitigate prejudice. However, the models fail at detecting group bias, and while they assign a greater importance to the sensitive attribute than the benchmark, it is not enough to mitigate prejudice as seen in the boxplots. Finally, note that the unit ball graph and the logistic regression model are really similar due to the sparsity of the underlying topology, hence cementing the relationship between topology and model performance. However, the results of Figure \ref{fig:SimulBoxplots} shows how the unit ball model achieves similar performance in the group metrics to logistic regression while achieving a great reduction in individual bias, hence becoming an attractive alternative to logistic regression while keeping similar interpretation.

\subsection{Effect of the hyper-parameters}
\label{subsec:Design}
After checking the overall capabilities of the model with the simulated data, we can now discuss the more concrete effects of modifying its hyper-parameters. This is done through a sensitivity analysis in which many different models are trained on the value ranges considered in Table \ref{tab:HyperParam} in the \ref{app:training}.\\
\begin{table}[ht!]
\centering
\begin{tabular}{*{8}{c}}
\toprule
Hyper-parameter & \multicolumn{7}{c}{Effect} \\
\cmidrule(lr){2-8}
 & ACC & IND & SEP & SUF & CON & LIP & ENT\\
\midrule
$\alpha$ & $\downarrow$ & $\downarrow$ & $\downarrow$ & $\downarrow$ & $\downarrow$ & $\rightarrow$ & $\downarrow$ \\
$n$ & $\downarrow$ & $\rightarrow$ & $\rightarrow$ & $\circlearrowleft$ & $\downarrow$ & $\downarrow$ & $\downarrow$ \\
$t$ & $\circlearrowleft$ & $\rightarrow$ & $\rightarrow$ & $\circlearrowleft$ & $\uparrow$ & $\circlearrowleft$ & $\circlearrowleft$ \\
$k$ & $\rightarrow$ & $\rightarrow$ & $\circlearrowleft$ & $\circlearrowleft$ & $\downarrow$ & $\downarrow$ & $\downarrow$ \\
$\delta$ & $\rightarrow$ & $\downarrow$ & $\rightarrow$ & $\circlearrowleft$ & $\downarrow$ & $\downarrow$ & $\rightarrow$ \\
$w_{subset} $ & $\uparrow$ & $\uparrow$ & $\uparrow$ & $\uparrow$ & $\downarrow$ & $\uparrow$ & $\uparrow$ \\
\bottomrule
\end{tabular}
\caption{Effect on each metric of increasing the hyper-parameters. The $\uparrow$ symbol means that the metric increases, $\downarrow$, that it decreases, $\rightarrow$, that it stays the same, and $\circlearrowleft$, that the effect is mostly unpredictable.}
\label{tab:HyperParamEffect}
\end{table}
The results are summarized in Table \ref{tab:HyperParamEffect}, which serves as a qualitative catalog of the broad effects on each metric that result from increasing each hyper-parameter. Basically, increasing the Sheaf Diffusion strength consistently compromises performance while improving fairness. The effect is similar with the number of layers of the discrete implementation, although in this case Group Fairness metrics are not improved, with IND and SEP staying the same while SUF varies unpredictably. Following with the integration time in continuous models, its effects are very chaotic. This, combined with longer training times, is why we recommend to avoid this implementation. Turning our attention to the defining parameters of the local topologies, in general Individual Fairness metrics are improved as $k$ and $\delta$ increase, while the effect is less clear on Group Fairness metrics and accuracy. This cements the intimate relationship between these topologies and Individual Fairness, and further suggests the robustness of these models with respect to their respective parameters. Finally, the weight of the global topology on mixed configurations improves accuracy while compromising all fairness metrics with the exception of consistency.\\
Although the sensitivity analysis is relegated to \ref{subapp:design}, we encourage interested readers to delve into this study in order to better understand the stability of the models when their hyper-parameters change. The main conclusions that can be drawn from the full analysis can be summarized in two points: the proposed global configurations are unstable and the results for continuous models exhibit chaotic behavior. Therefore, we encourage the use of local topologies with the discrete implementation.\\

\subsection{Case study}
\label{subsec:Cases}

The analysis concludes with the study of the performance of the proposed methods on a series of standard datasets used to benchmark fairness methods. The details of the datasets used are shown in Table \ref{tab:Datasets}. The discussion begins with a discussion of the performance-fairness trade-offs available with FSD models, followed by an analysis of the Pareto-optimal solutions of each topology and concluding with an examination of the variable importance and the interpretation of the methods. The hyper-parameter grids chosen for the grid search can be found in Table \ref{tab:Grid} in \ref{app:training}.

\subsubsection{The cost of fairness}
\label{subsubsec:Cost}

After running the grid search specified by Table \ref{tab:Grid}, a solution must be chosen. However, at the heart of fair ML lies a multi-objective optimization problem: balancing fairness with performance. In particular, we will measure group fairness using $IND$ and individual fairness through $CON$ because of their good behavior in both the simulation study and the hyper-parameter test. Meanwhile, we will use $ACC$ as a measure of performance. However, the three-fold choice we have to make is not straightforward and it can be aided by the visualization of available trade-offs in the Pareto-efficient frontier. Figure \ref{fig:Pareto2D} shows the Pareto fronts for independence and accuracy on the one hand, representing Group Fairness and performance trade offs, and the efficient frontier for consistency and accuracy, representing Individual Fairness and performance trade-offs. These objects allow us to quantify the cost of fairness. For example, in the German dataset we can compromise accuracy by just $2\%$ while achieving an improvement of nearly $50\%$ in independence or $33\%$ in consistency.

\begin{figure}[!t]
    \centering
    \includegraphics[width=0.49\textwidth]{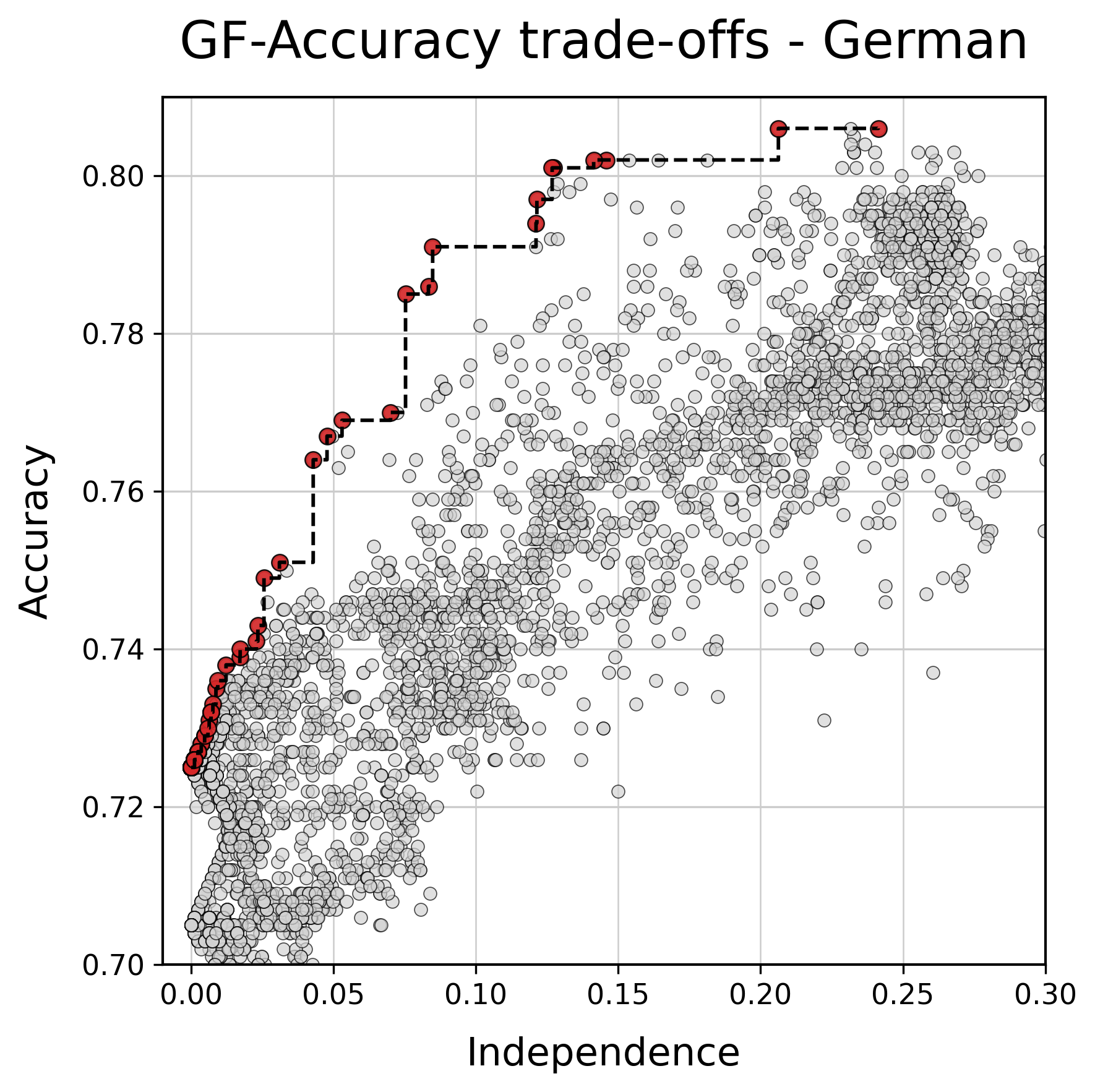}
    \includegraphics[width=0.49\textwidth]{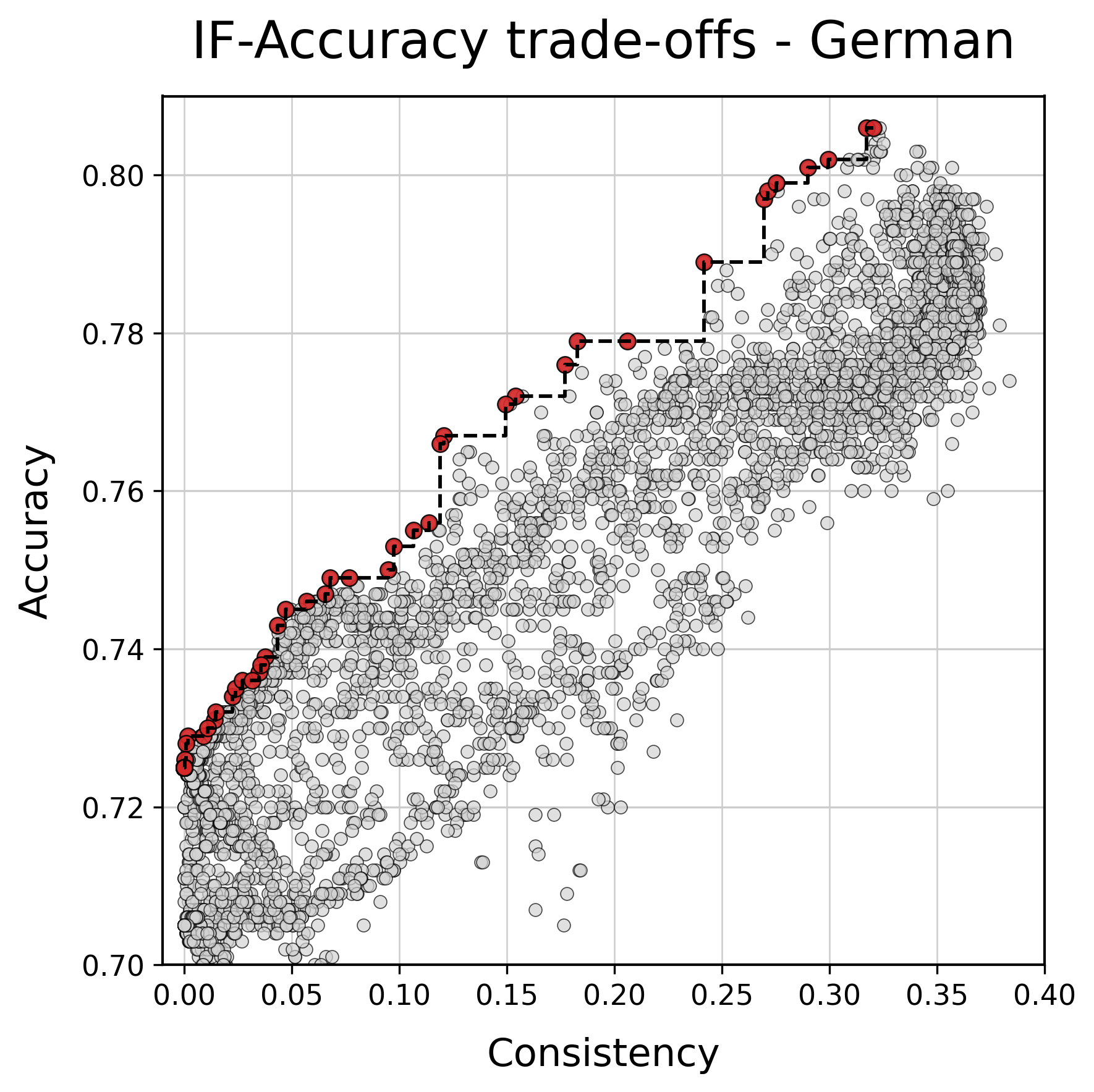}   
    \caption{Fairness-accuracy trade-offs in the Pareto frontier for the grid search on the German dataset. The left figure shows trade-offs between independence and accuracy, while the right displays consistency and accuracy.}
    \label{fig:Pareto2D}
\end{figure}

However, we are interested in giving a unified perspective between individual and group fairness. For this reason, Figure \ref{fig:Pareto3D} paints a more complete picture integrating both independence and consistency in a single plot, showcasing a three-dimensional Pareto quantifying the trade-offs between accuracy and both fairness metrics. This allows us to better understand the interaction between individual and group fairness metrics, thus making it easier to select an optimal solution. For example, for the German dataset, a compromise of just $2\%$ in accuracy leads to a combined improvement of $33\%$ in both independence and consistency.

\begin{figure}[!t]
    \centering
    \includegraphics[width=.75\textwidth]{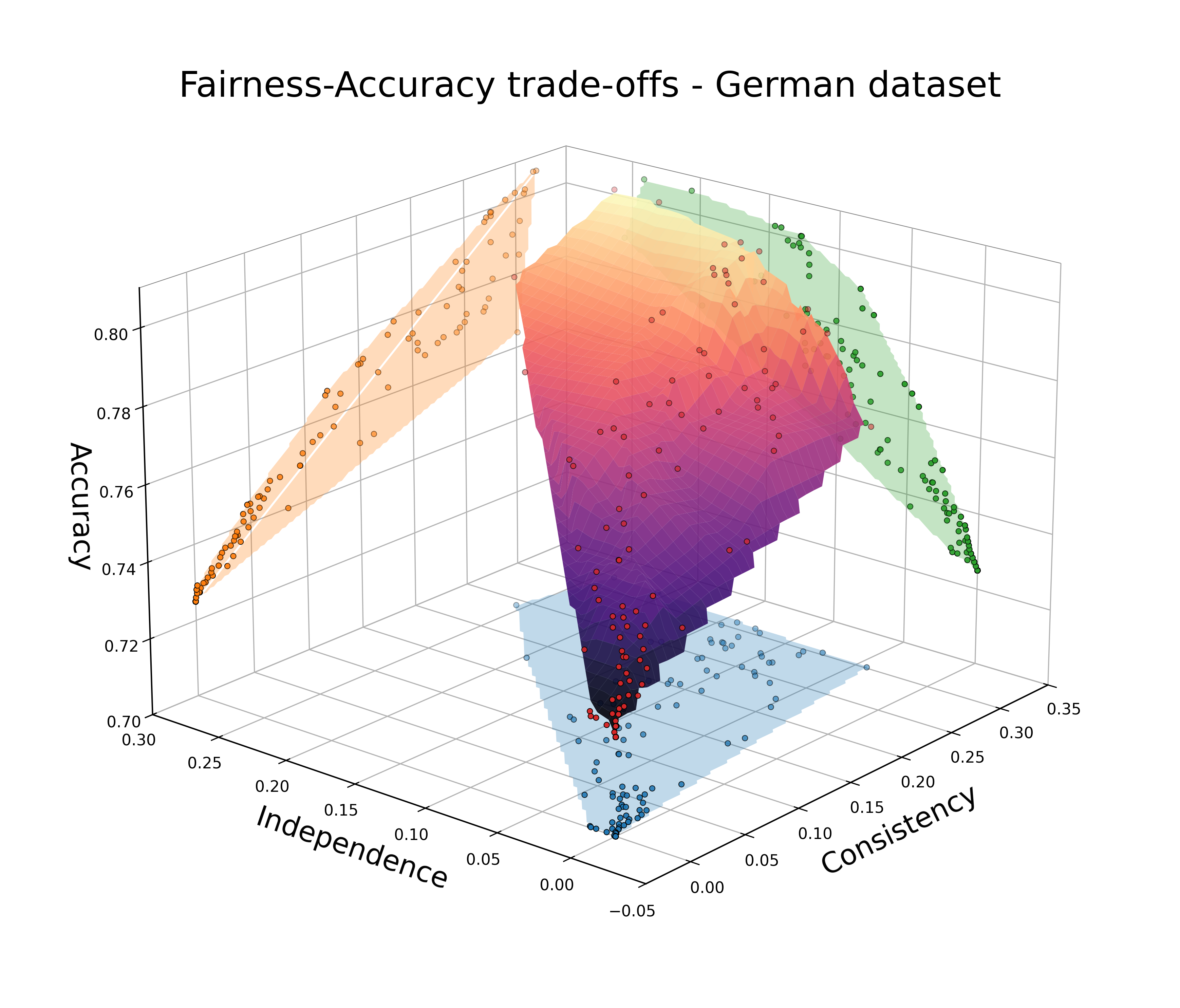}
    \caption{Fairness-accuracy trade-offs in the Pareto frontier accounting for both independence and consistency at the same time.}
    \label{fig:Pareto3D}
\end{figure}

With this picture in mind, we can now choose a Pareto efficient solution out of the grid search. In particular, for every model and its corresponding grid we will consider the model that maximizes $ACC - IND - CON$, thus giving equal weight to all three metrics. Finally, although this Section has only shown the efficient frontiers for the German dataset, readers itnerested in the trade-offs available for the rest of the datasets are refered to \ref{subapp:cost}.

\subsubsection{Results of the grid search}
\label{subsubsec:Results}

A grid search is run with the goal of maximizing the combined quantity $ACC - IND - CON$. The results for the optimal model for each dataset are then aggregated and shown in Table \ref{tab:DiffBenchmark}, while the results for each particular dataset can be found in \ref{app:results} in Tables \ref{tab:GermanResults} to \ref{tab:AdultResults}.

\begin{table}[ht!]
\centering
\resizebox{\linewidth}{!}{
\begin{tabular}{lccccccc}
\toprule
Model & ACC (\%) & IND (\%) & SUF (\%) & SEP (\%) & CON (\%) & LIP & ENT \\
\midrule
kNN & {\footnotesize -1.700}{\scriptsize $\pm$1.483} & \textbf{{\footnotesize -13.214}{\scriptsize $\pm$4.790}} & {\footnotesize -3.391}{\scriptsize $\pm$7.706} & \textbf{{\footnotesize -8.036}{\scriptsize $\pm$5.030}} & \textbf{{\footnotesize -4.480}{\scriptsize $\pm$2.466}} & \textbf{{\footnotesize -0.006}{\scriptsize $\pm$0.005}} & {\footnotesize -0.178}{\scriptsize $\pm$0.161} \\
Mixed kNN & {\footnotesize -1.600}{\scriptsize $\pm$1.233} & {\footnotesize 0.887}{\scriptsize $\pm$5.000} & {\footnotesize -0.665}{\scriptsize $\pm$7.577} & {\footnotesize 0.049}{\scriptsize $\pm$4.660} & {\footnotesize -2.700}{\scriptsize $\pm$3.185} & {\footnotesize 0.002}{\scriptsize $\pm$0.005} & {\footnotesize -0.171}{\scriptsize $\pm$0.135} \\
Unit & {\footnotesize -1.400}{\scriptsize $\pm$1.183} & {\footnotesize -5.774}{\scriptsize $\pm$4.292} & \textbf{{\footnotesize -3.892}{\scriptsize $\pm$7.330}} & {\footnotesize -7.142}{\scriptsize $\pm$5.522} & {\footnotesize -0.780}{\scriptsize $\pm$2.005} & {\footnotesize -0.002}{\scriptsize $\pm$0.003} & {\footnotesize -0.151}{\scriptsize $\pm$0.131} \\
Mixed unit & \textbf{{\footnotesize -0.100}{\scriptsize $\pm$1.192}} & {\footnotesize 6.350}{\scriptsize $\pm$4.909} & {\footnotesize -0.489}{\scriptsize $\pm$8.548} & {\footnotesize 2.537}{\scriptsize $\pm$6.158} & {\footnotesize -0.540}{\scriptsize $\pm$3.407} & {\footnotesize 0.002}{\scriptsize $\pm$0.004} & {\footnotesize -0.009}{\scriptsize $\pm$0.144} \\
Subset & {\footnotesize -2.700}{\scriptsize $\pm$1.095} & {\footnotesize -10.981}{\scriptsize $\pm$4.033} & {\footnotesize 1.817}{\scriptsize $\pm$7.482} & {\footnotesize -2.814}{\scriptsize $\pm$5.800} & {\footnotesize -1.680}{\scriptsize $\pm$7.368} & {\footnotesize -0.003}{\scriptsize $\pm$0.009} & \textbf{{\footnotesize -0.278}{\scriptsize $\pm$0.114}} \\
\bottomrule
\end{tabular}
}
\caption{Average differences between the proposed methods and logistic regression for all datasets.}
\label{tab:DiffBenchmark}
\end{table}

Starting with accuracy, it is easy to see that all models compromise performance, with the worst performing model being found in the subset configuration while the mixed unit ball topology seems to achieve the smaller reduction in accuracy. Mixing a local graph with the global configuration leads to enhanced performance, with the improvement being very small for the kNN graph but significant for the unit ball topology. Moving on to group fairness metrics, it seems like these metrics are compromised when using a mixed configuration, as seen in both independence and separation. The purely local and global topologies seem to perform better in this front, with kNN being the biggest winner in this respect both in terms of $IND$ and $SEP$. The reason behind this is found in the underlying topology: kNN manages to create diverse neighbors including members of both sensitive groups, thus encouraging group fairness. If this was not the case then group fairness might not improve; this is the case of the unit ball topology, whose area of influence is reduced due to the presence of a fixed radius, thus its neighbors might not be as diverse. Finally, sufficiency is improved across the board when using graph models, with the biggest winner being the unit ball configuration. However, this metric has been found inadequate in many contexts (see, for instance, \cite{EJORSurvey}), so these results should be taken with a grain of salt. Moving on to individual fairness metrics, consistency is improved across the board in all topologies as a consequence of the smoothing behavior of graph models (see, for instance, \cite{NSD}). However, the best performer in terms of $CON$ is the pure kNN topology, which aligns with our theoretical expectations. On the other hand, the lipschitz constant is barely changed after using graph models, the biggest improvement is given by the kNN topology once again, thus cementing its position as one of the best models for both individual and group fairness. Mixed configurations seem to perform worse, although in general changes in this metric are not significant. Finally, the generalized entropy is improved with all topologies, once again as a consequence of the smoothing capabilities of graph models. The best model in this front is given by the global configuration by virtue of its topology being the one that most strongly incorporates group fairness, which when combined with the smoothing properties of graph models leads to the biggest improvement in this unified metric.\\
All in all, the analysis suggests that graph models provide an attractive option for ML practitioners interested in mitigating individual and group bias in a combined effort, with the best models being those involving the kNN topology in general. Finally, readers interested in the concrete results for each dataset are refered to \ref{subapp:results}.

\subsubsection{Understanding fairness}
\label{subsubsec:Understanding}
The study of the models concludes with the analysis of the SHAP values, found in Figure \ref{fig:Influence}, which shows the normalized and absolute aggregated SHAP values for all predictions for the Compas dataset. Those interested in the analysis of all SHAP values are refered to \ref{subapp:understanding}.\\
Although graph models use the same underlying variables to make decisions as the reference found in logistic regression, the importance they assign can be dampened or exacerbated by the underlying topology, although this effect is less pronounced on sparse graphs like the unit ball configuration. One example of dampening happens in the \emph{c\_charge\_desc} variable, while one striking example of exacerbation is found in the race variable, whose relative importance is increased. However, this falls in line with the findings of \cite{ADVERSARIAL}, whose model also increases the influence of the sensitive variable in order to mitigate prejudice and is a manifestation of fairness through awareness. Finally, we find that FSD models tend to exacerbate low-variance variables, which contributes to the reduction in individual bias seen in the previous Subsection.

\begin{figure}[!t]
    \centering
    \includegraphics[width=0.49\textwidth]{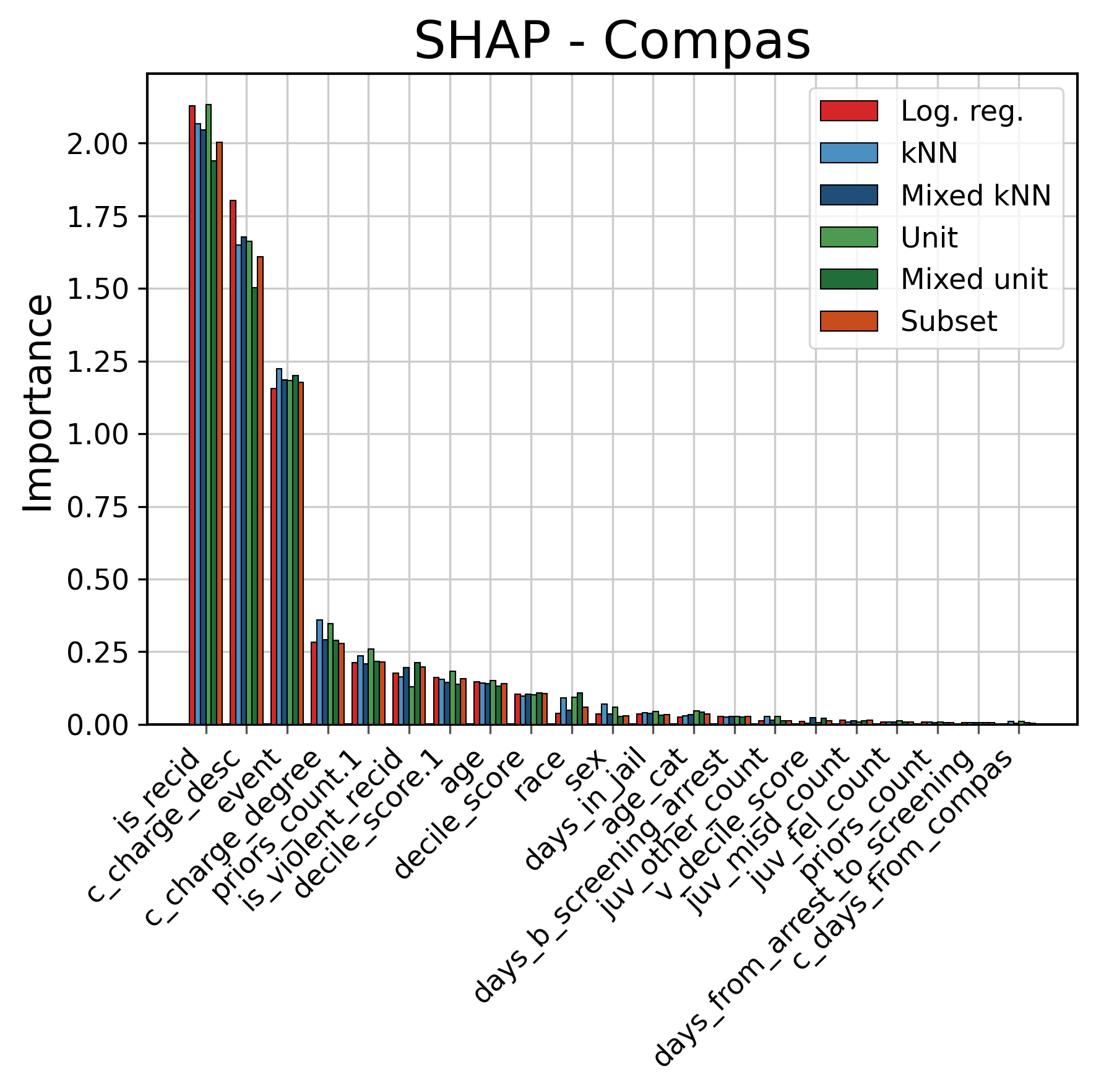}
    \includegraphics[width=0.49\textwidth]{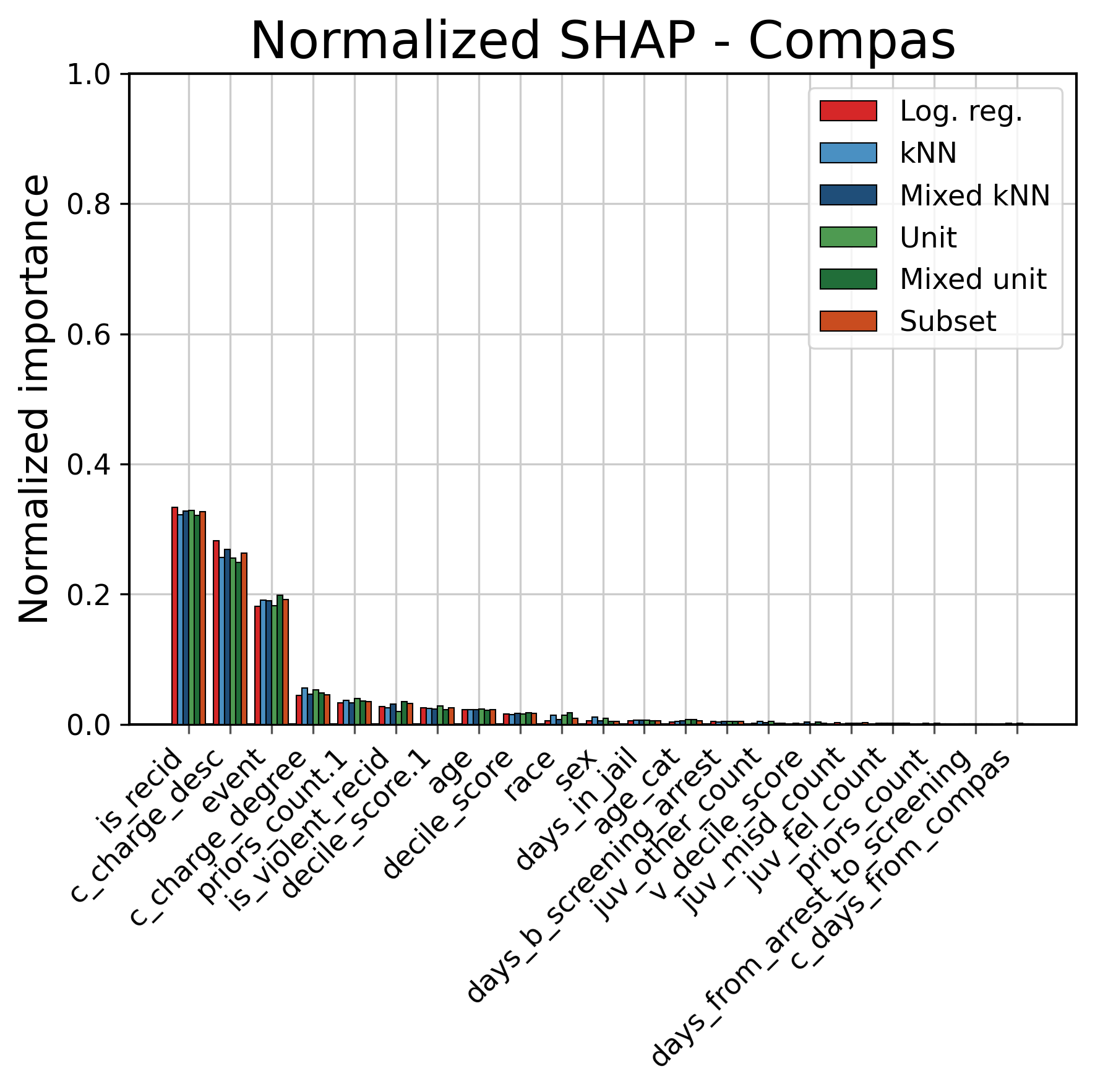} 
    \caption{SHAP variable influence. The left image shows average absolute SHAP importance, while the right image is normalized.}
    \label{fig:Influence}
\end{figure}

%\begin{figure}[!ht]
%    \centering
%    \includegraphics[width=\textwidth]{FigCompassInfluence.png}
%    \includegraphics[width=\textwidth]{FigCompassInfluenceNorm.png}    
%    \caption{SHAP variable influence for the Compass dataset. The left panel shows the average absolute SHAP contributions while the right panel influences are normalized.}
%    \label{fig:InfluenceCompass}
%\end{figure}

\section{Conclusion}
\label{sec:Conclusion}
This paper has introduced a novel technique in the realm of fair ML in Fair Sheaf Diffusion, which has opened the door for attractive innovations in the literature.\\
The main contribution is found in the modeling of fairness constraints through graph configurations even in the absence of traditional edges. Furthermore, the combination of different network topologies allows to simultaneously address individual and group bias, extending existing work by \cite{ENTROPY} and providing a combined perspective on the matter. The classifiers built on these ideas provide fine control over the existing trade-offs between accuracy and individual and group fairnes through a series of hyper-parameters whose effects are comprehensively cathegorized, and they admit closed-form expressions for their predictions, facilitating interpretations and explanations. Moreover, these models can be placed in any part of the ML pipeline, thus becoming a pre-processor, in-processor or post-processor depending on the needs of the practitioner thus offering extra flexibility, or they could be concatenated to create a hybrid multistage processor, which is a future line of research which could lead to enhanced fairness. \\
These ideas are materialized through a thorough series of case studies which range from simulations to real datasets. The experiments show how the models achieve notable results both in performance and fairness, and they allow to quantify the effect of the different implementations, graphs and hyper-parameters, providing a comprehensive catalog for practitioners. Finally, the analysis of their SHAP values allows practitioners with readily-available explanations when those are needed in critical contexts.\\
The most inmediate ways of generalizing our work consist of adapting the proposed methodologies to more general settings and sheaf models, the most straighforward being found in extensions of sheaf models like those found for directed graphs \citep{DIRECTED} or hypergraphs \citep{HYPERGRAPH}, which could cement the use of topological flavoured methods in fairness literature for more complex data structures. On the other hand, some more complex lines of future research involve the use of non-linear sheaf diffusion \citep{NONLINEAR} to model non-linear constraints and metrics, the adoption of the wave equation \citep{SHEAFSUR} to propose an alternative method based on its properties related to energy conservation, and the use of local topologies to achieve counterfactual fairness \citep{JASA_COUNTERFACTUAL, knnCounterfactual}. Finally, the framework introduces the language of algebraic topology, associating fairness with a cohomology sequence. We conjecture that the use of the theoretical tools of algebraic topology could lead to an interesting new direction which could provide an alternative foundation to fairness.

\section*{Acknowledgments}
This research is part of the I + D + i projects PID2022-137243OB-I00 and  PID2022-137818OB-I00 funded by Ministerio
de Ciencia, Innovación y Universidades/AEI/10.13039/501100011033 and European Union NextGenerationEU/PRTR. This initiative has also been partially carried
out within the framework of the Recovery, Transformation and Resilience Plan funds, financed by the European Union (Next Generation) through the grant ANTICIPA
and the ENIA 2022 Chairs for the creation of university-industry chairs in AI-AImpulsa: UC3M-Universia.

\bibliographystyle{apalike}
\bibliography{cas-refs}

\appendix
\section{Training procedure}
\label{app:training}

The choice of values or grids used for different hyper-parameters are shown in Tables \ref{tab:HyperParamSimul} to \ref{tab:Grid}. Note that $Q(z)$ represents the $z$ quantile of the distribution of observed distances, $Q(z) = Q(z; \{\|x_i-x_j\|_2 \mid x_i, x_j \in X\})$. All training was done using the AdamW optimizer and performed on an NVIDIA Tesla T4 GPU with $16$GB of VRAM. The system was running NVIDIA driver version $565.57.01$ and CUDA $12.7$.

\begin{table}[ht!]
    \centering
    \resizebox{\linewidth}{!}{
    \begin{tabular}{ccc}
    \toprule
         Parameter & Value & Description\\
         \midrule
         $\alpha$ & $0.3$ & Strength of the Sheaf Diffusion.\\ 
         $n$ & $10$ & Number of layers.\\          
         $k$ & $5$ & Number of neighbors for the knn topology. \\
         $\delta$ & $Q(0.1; \{d(x_i, x_j)\mid x_i, x_j \in X\})$ & Distance value for the unit ball topology. \\
          $w_{subset} $ & $0.5$ & Convex weight of the subset topology. \\
         \bottomrule
    \end{tabular}
    }
    \caption{Values considered for the simulation study.}
    \label{tab:HyperParamSimul}
\end{table}

\begin{table}[ht!]
    \centering
    \resizebox{\linewidth}{!}{
    \begin{tabular}{cccc}
    \toprule
         Parameter & Range of values & Standard value & Description\\
         \midrule
         $\alpha$ & $0.1-1.3$ & $0.3$ & Strength of the Sheaf Diffusion.\\ 
         $n$ & $1-40$ & $10$ & Number of layers of the discrete model.\\          
         $t$ & $0.1-1.3$ & $0.6$ & Integration time of the continuous model. \\ 
         $k$ & $1-15$ & $5$ & Number of neighbors for the knn topology. \\
         $\delta$ & $Q(50/N)-Q
         (500/N)$ & $Q(500/N)$ & Distance value for the unit ball topology. \\
          $w_{subset} $ & $0.1 - 0.9$ & $0.5$ & Convex weight of the subset topology. \\
         \bottomrule
    \end{tabular}
    }
    \caption{Values considered for the study of the effect of the different hyper-parameters of the graph configuration and SD process.}
    \label{tab:HyperParam}
\end{table}

\begin{table}[ht!]
    \centering
    \resizebox{\linewidth}{!}{
    \begin{tabular}{ccc}
    \toprule
         Parameter & Value or grid & Description\\
         \midrule
         $\mathcal{P}$ & $\{\{x_i | A = a\}_{a}\}$ & Partition for the subset topology.\\ 
         $d$ & $\| \cdot \|_2$ & Distance function for the local topologies \\
         $k$ & $\{1, 5, 15\}$ & Number of neighbors for the knn topology. \\
         $\delta$ & $\{Q(50/N), Q(100/N), Q(500/N)\}$ & Distance value for the unit ball topology. \\
         $\alpha$ & $\{0.05, 0.1, 0.3\}$ & Strenght of the Sheaf Diffusion.  \\ 
         $n$ & $\{5, 10, 20\}$ & Number of layers of the discrete model. \\   
         $t$ & $\{0.1, 0.3, 1.0\}$ & Integration time of the continuous model. \\ 
         $WD$ & Log-uniform$[10^{-6}, 10^{-2}]$ & Weight-decay of the AdamW optimizer.\\
        $LR$ & Log-uniform$[10^{-4}, 10^{-1}]$ & Learning rate of the AdamW optimizer.\\
        $\beta_1$ & Uniform$[0.8, 0.95]$ & First moment decay rate of the AdamW optimizer. \\
        $\beta_2$ & Uniform$[0.95, 0.9999]$ & Second moment decay rate of the AdamW optimizer. \\
         \bottomrule
         
    \end{tabular}
    }
    \caption{Grids and values for the grid search.}
    \label{tab:Grid}
\end{table}

\section{Extended results}
\label{app:results}

We now show the results which did not fit in the main paper along with their analysis in order to paint a more complete picture of the discussed models.

\subsection{Simulation study}
\label{subapp:simulation}

First, we show the full results of the simulation study outlined in Section \ref{subsec:Setup}, this time showing all metrics. Remember that the simulation is ran $100$ times with the hyper-parameters shown in Table \ref{tab:HyperParamSimul}. The results are shown in Figures \ref{app:FirstBox} to \ref{app:LastBox} which represent the distribution of accuracy and the fairness metrics reviewed in Section \ref{sec:Theory}. For starters, all models seem to compromise accuracy, with a relative decrease in the median performance ranging from around $13\%$ for the subset topology to $3\%$ for the unit ball graph. Likewise, all individual fairness metrics are improved after using a graph model, with the most resounding success found in consistency with both kNN configurations, which results in a relative median decrease in the already low $CON$ metric of $33\%$, thus cementing the intuition behind these methods. The processors also show promise for $ENT$, with a consistent decrease of around $33\%$. Finally, with $LIP$, although the metric is also improved, the results are not as dramatic. Moving onto group fairness metrics, the results are less satisfactory. For starters, the global topology is the biggest underperformer, compromising all group fairness metrics despite the presented ideas. The rest of the models underperform in this front when compared to logistic regression with the exception of the purely local topologies with the $IND$ metric, where the fairness readings are similar.\\

\begin{figure}[ht]
  \centering
\includegraphics[width=.49\textwidth]{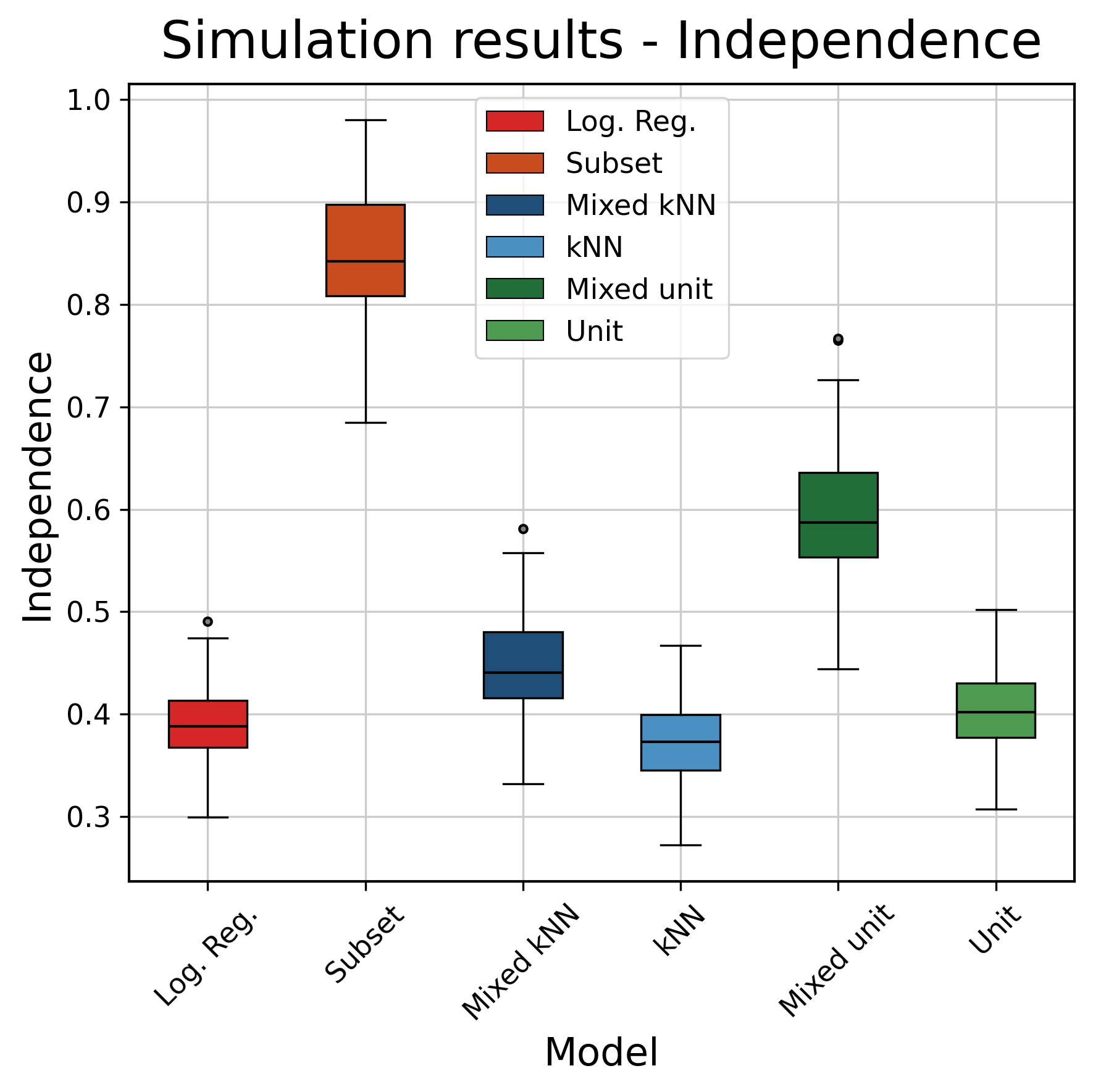}
\includegraphics[width=.49\textwidth]{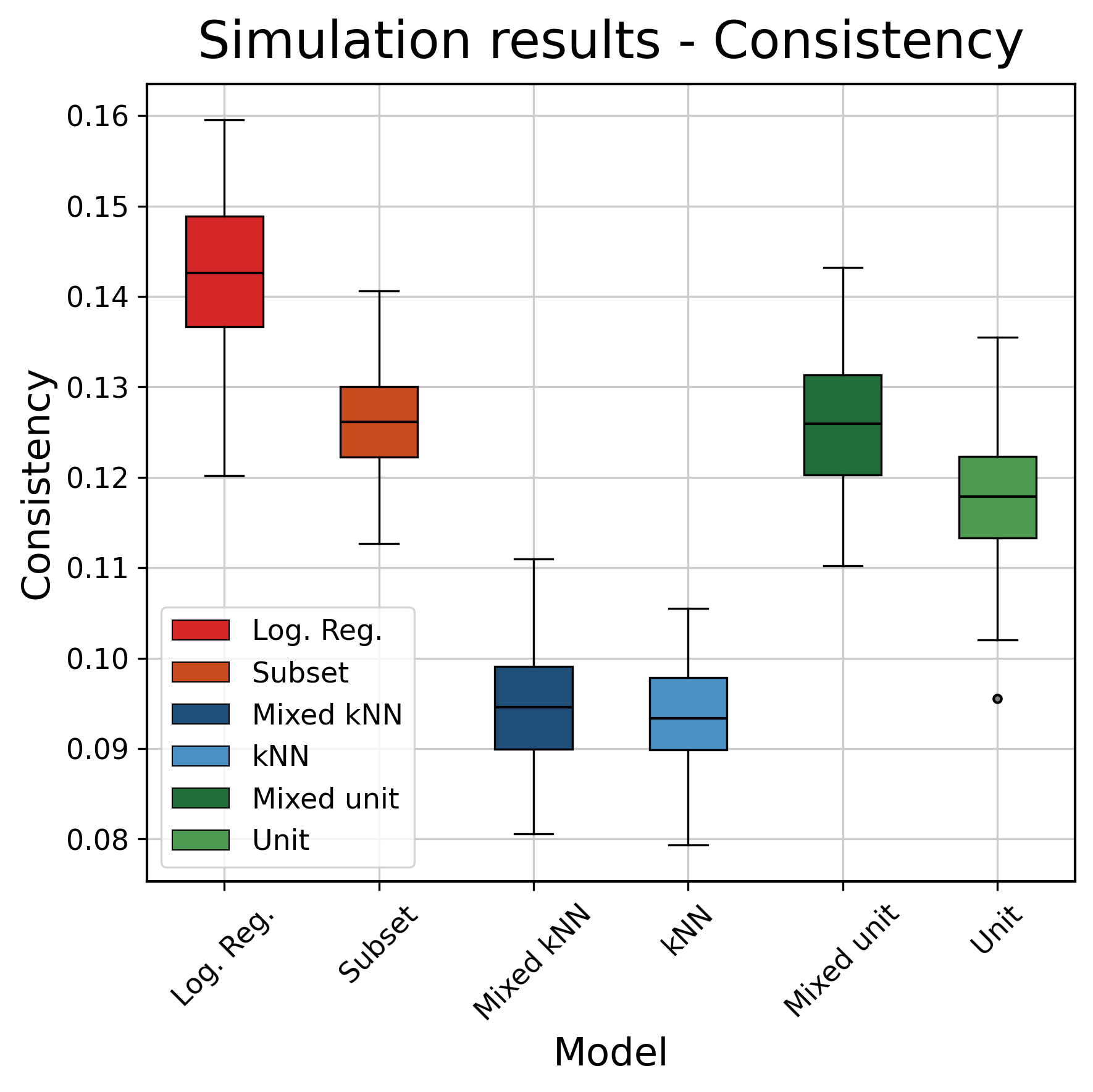}
   \caption{}
\label{app:FirstBox}
\end{figure}

\begin{figure}[ht]
  \centering
\includegraphics[width=.49\textwidth]{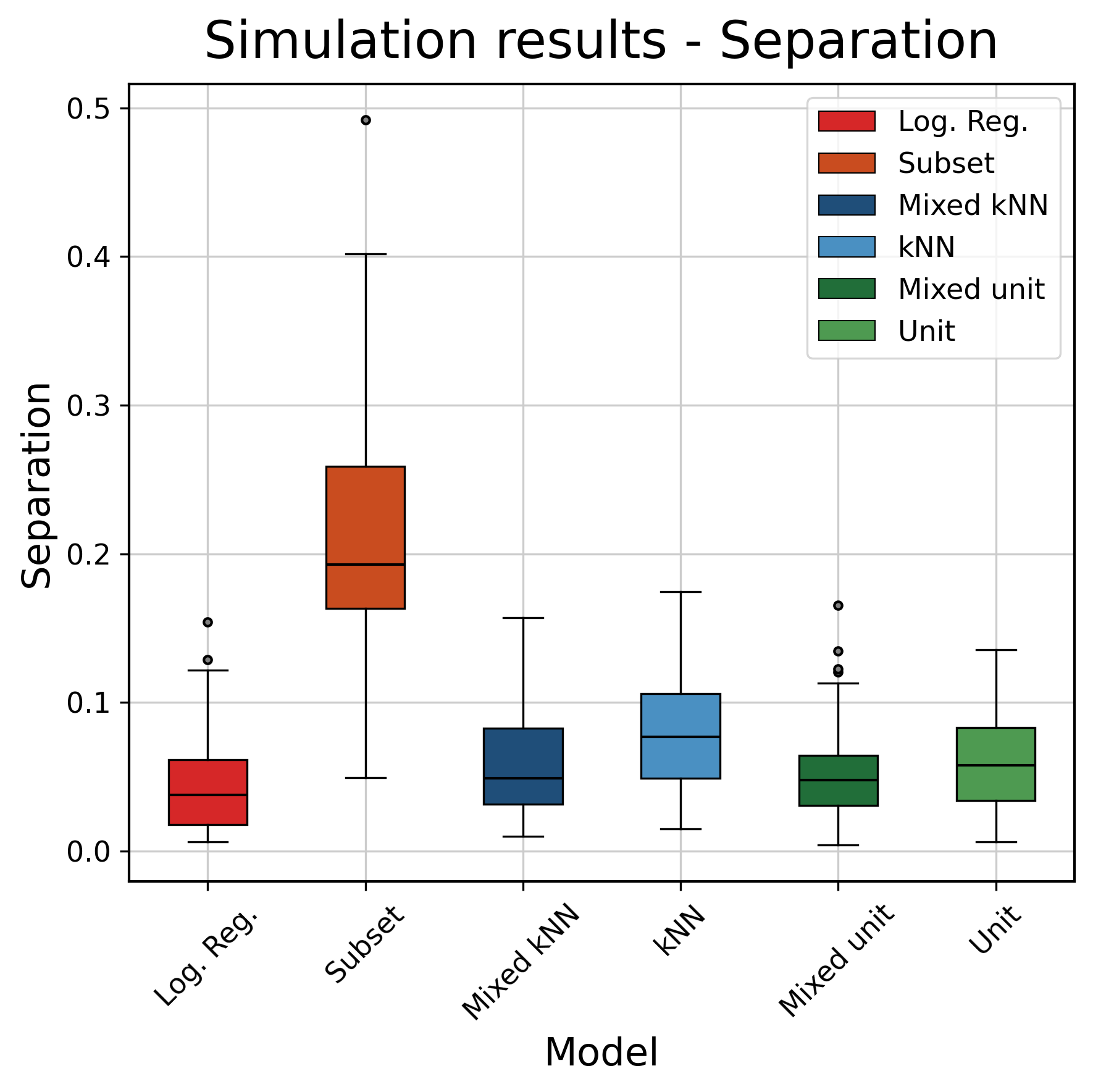}
\includegraphics[width=.49\textwidth]{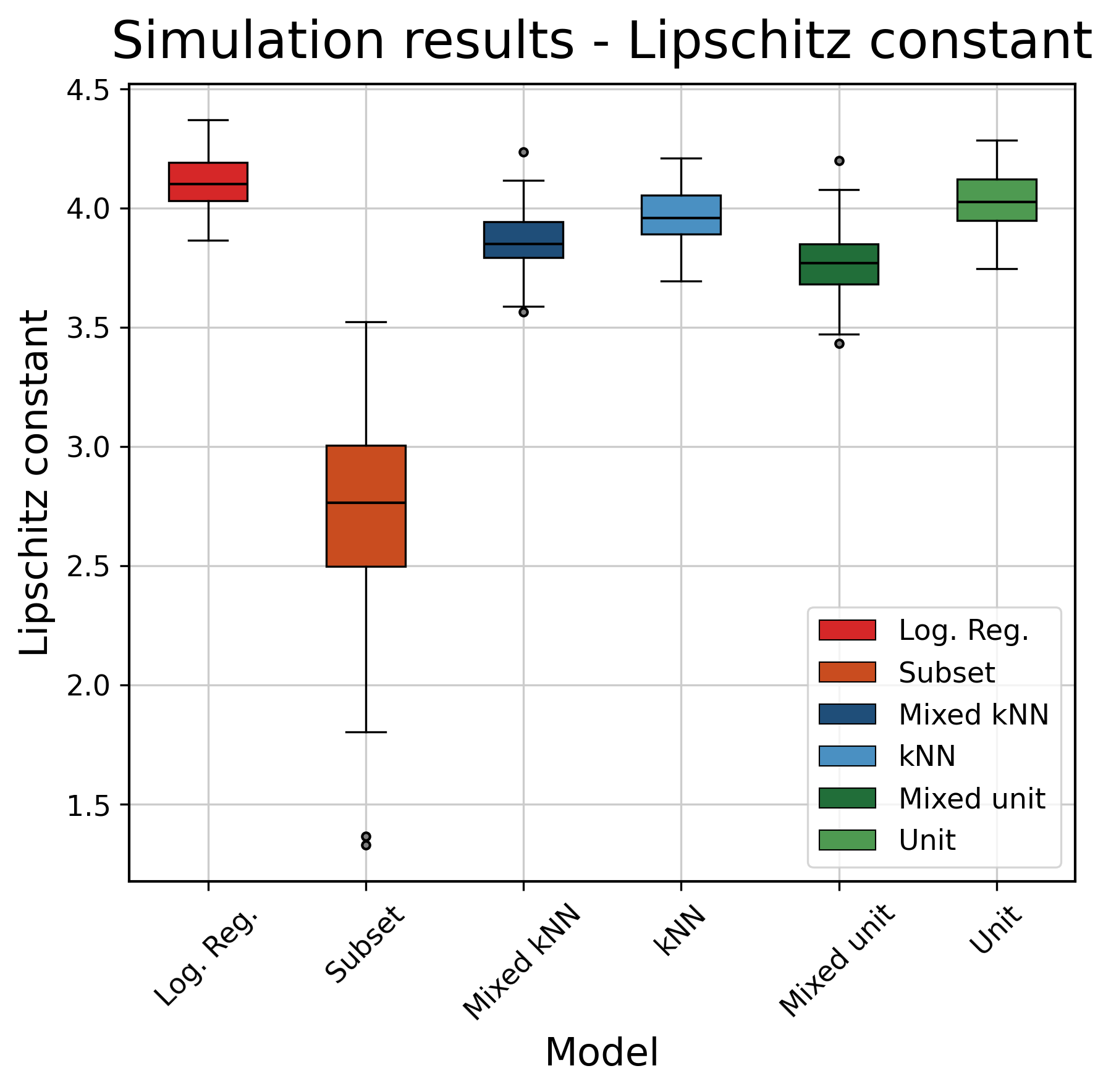}
   \caption{}
\end{figure}

\begin{figure}[ht]
  \centering
\includegraphics[width=.49\textwidth]{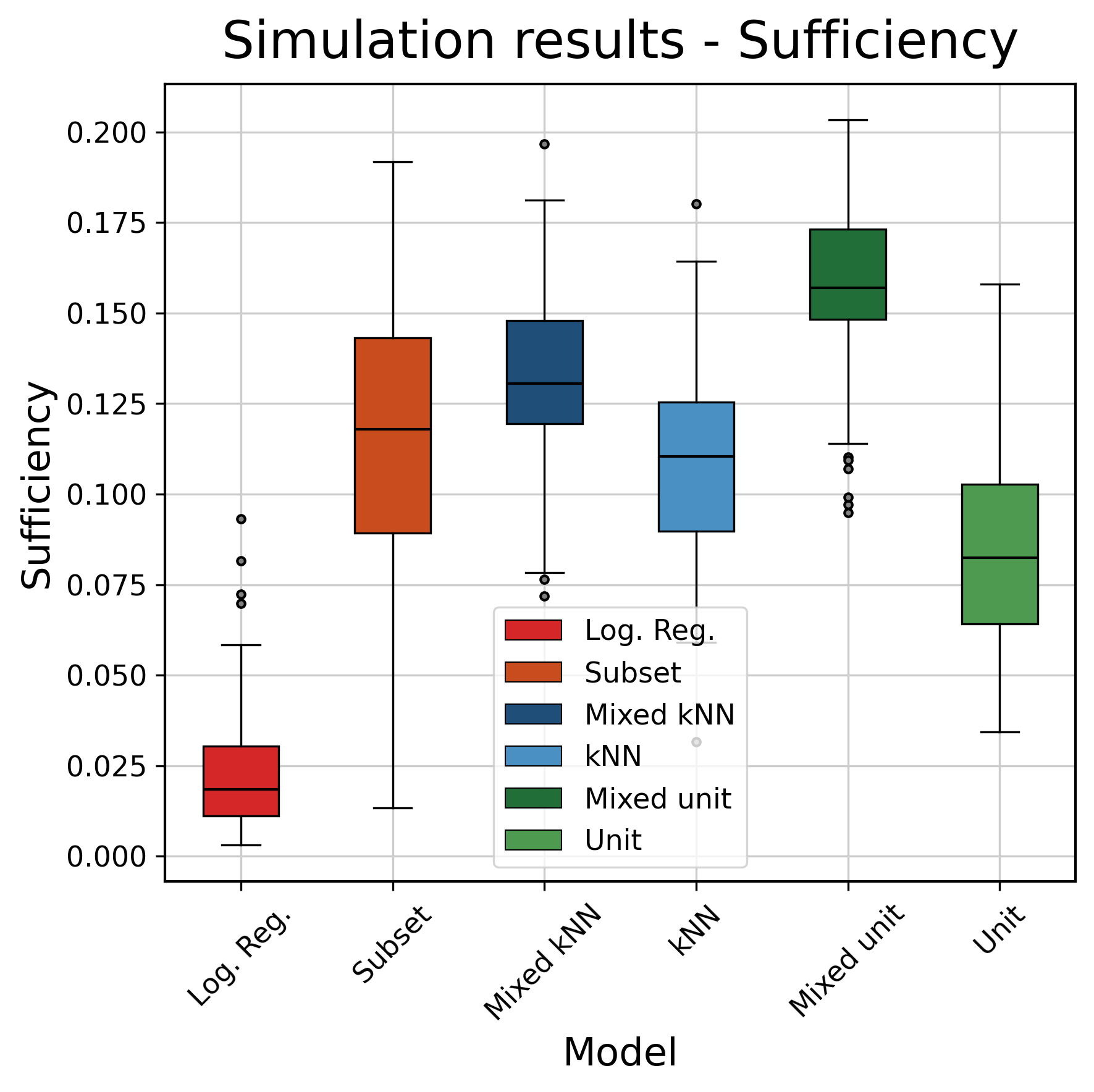}
\includegraphics[width=.49\textwidth]{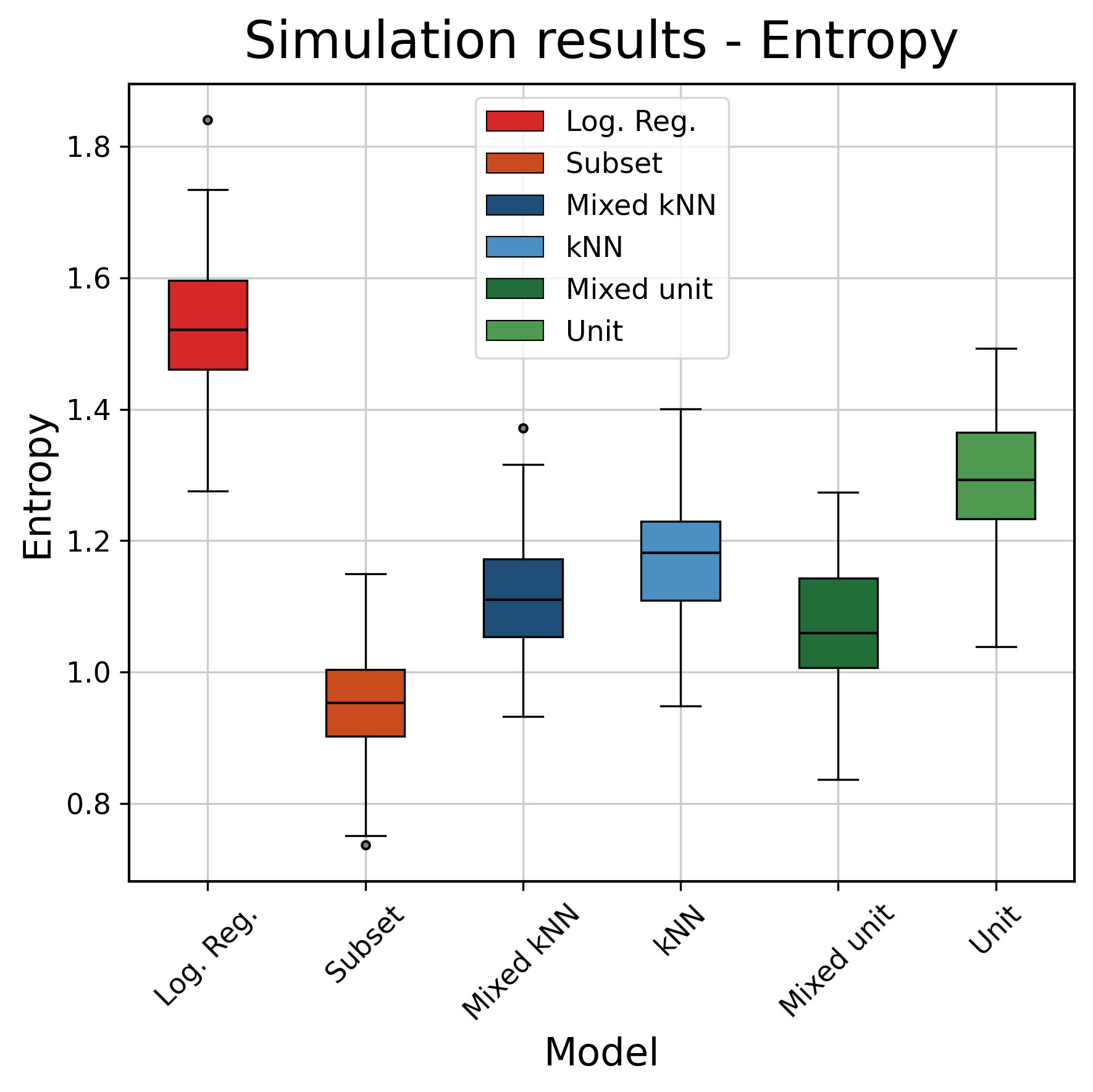}
   \caption{}
\end{figure}

\begin{figure}[ht]
  \centering
 \includegraphics[width=.49\textwidth]{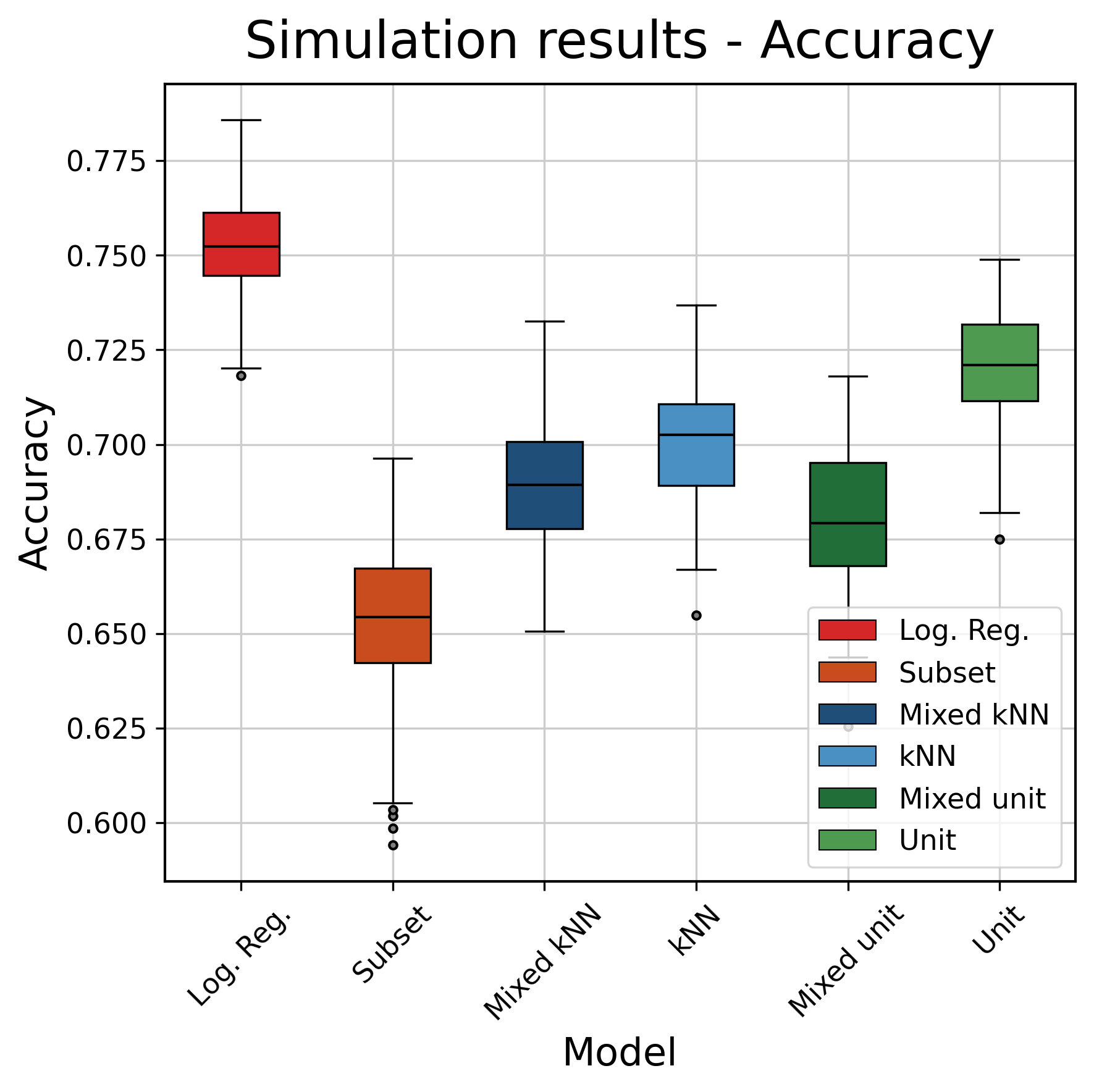}
    \caption{}
 \label{app:LastBox}
\end{figure}

\subsection{Sensitivity analysis}
\label{subapp:design}

This Section showcases the full sensitivity analysis which Section \ref{subsec:Design} and Table \ref{tab:HyperParamEffect} build upon, delving into a more detailed, microstructural analysis of the different hyper-parameters and topologies with the aim of easing the understanding of the different design choices behind FSD models. All considered value ranges can be found in Table \ref{tab:HyperParam} in the \ref{app:training}. 

\subsubsection{Effect of $\alpha$, $n$ and $t$}
The first parameters to be studied are those common to all models, found in the strenght of the diffusion, $\alpha$, and its length, given by the number of layers, $n$, for the discrete implementation and by the integration time, $t$, in the continuous case. The analytical expression \eqref{SheafDif} suggests that increasing all of these factors should lead to improved fairness at the cost of worse performance, because increasing $\alpha$ increases the magnitude of the negative exponent of the solution, while the integration time or the number of layers is equivalent to letting time approach infinity, hence getting closer to the fair time limit. The results are shown in Figures \ref{fig:FirstCoef} to \ref{fig:LastTime}.

\begin{figure}[!t]
  \centering
  \includegraphics[width=.49\textwidth]{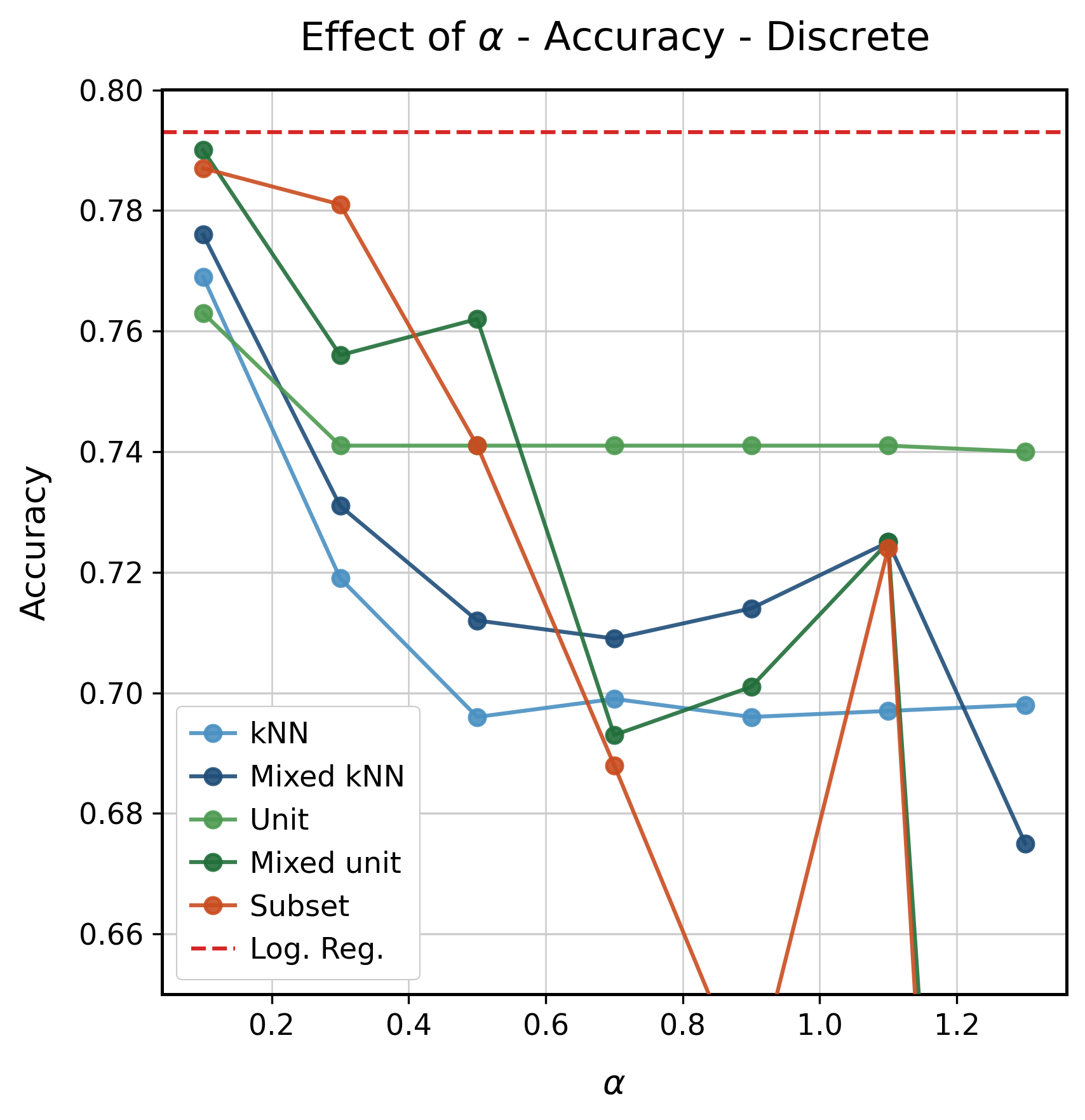}
  \includegraphics[width=.49\textwidth]{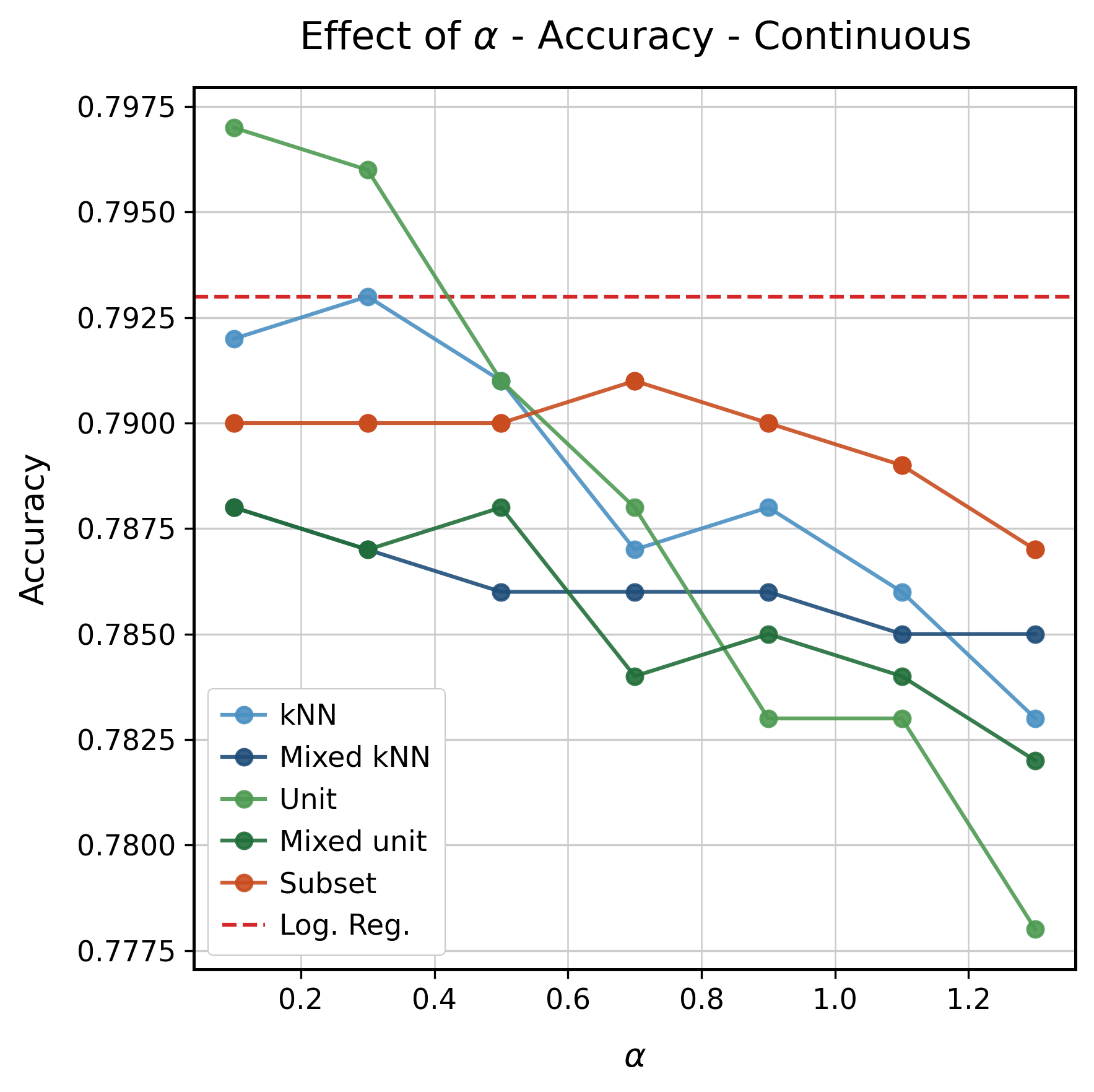}
  \caption{}
  \label{fig:FirstCoef}
\end{figure}

\begin{figure}[!t]
  \centering
  \includegraphics[width=.49\textwidth]{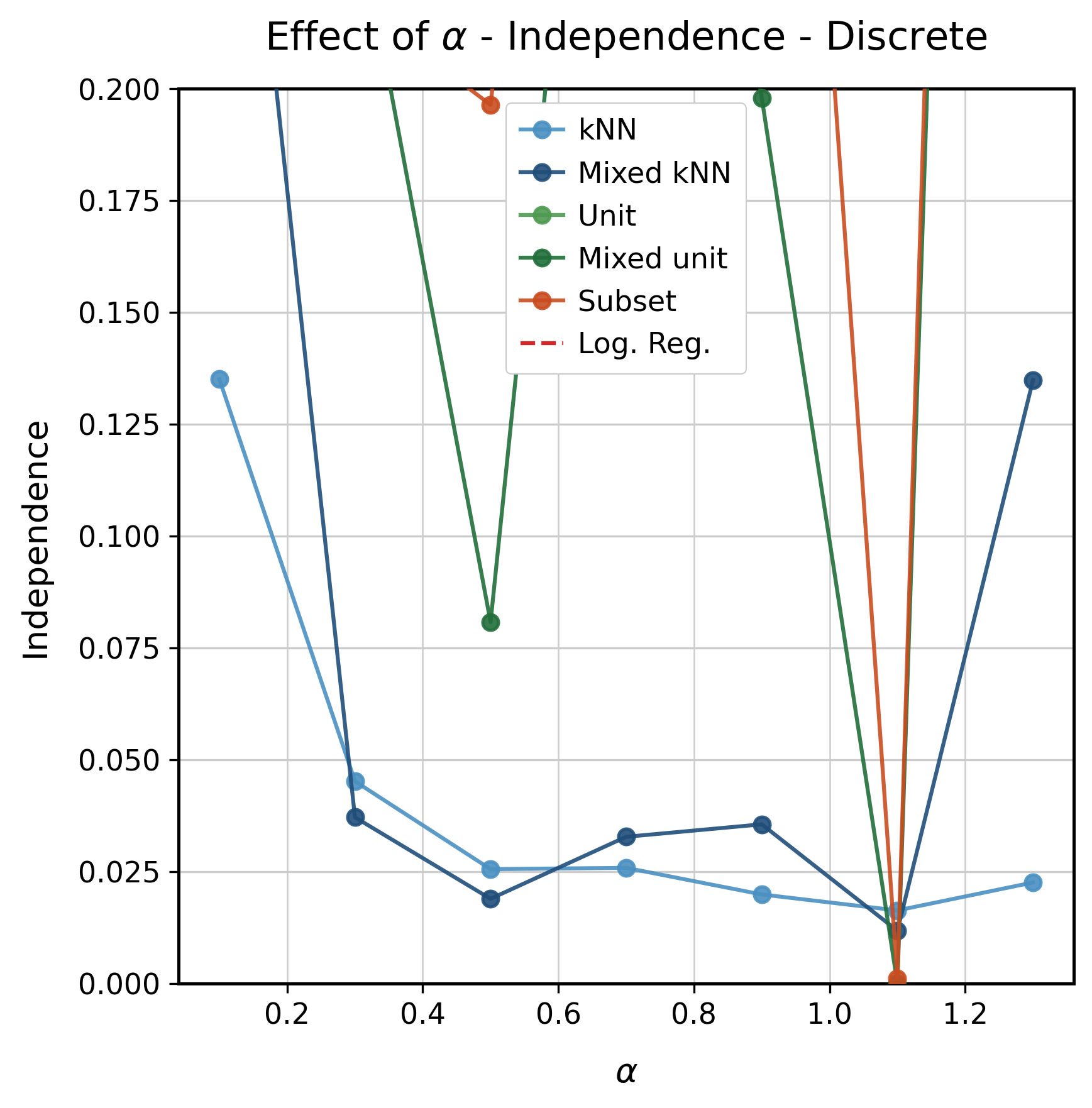}
  \includegraphics[width=.49\textwidth]{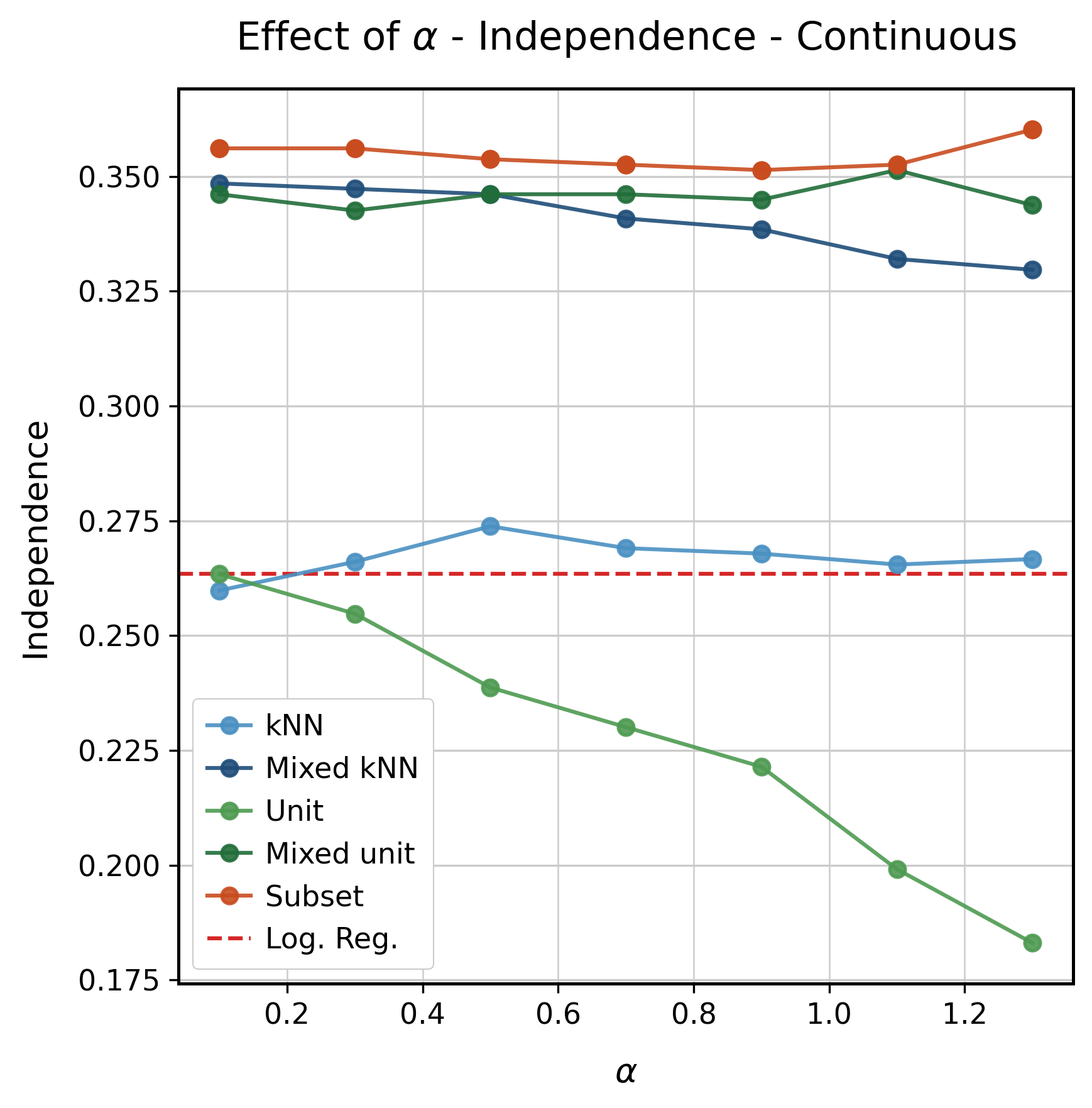}
    \caption{}
\end{figure}

\begin{figure}[!t]
  \centering
  \includegraphics[width=.49\textwidth]{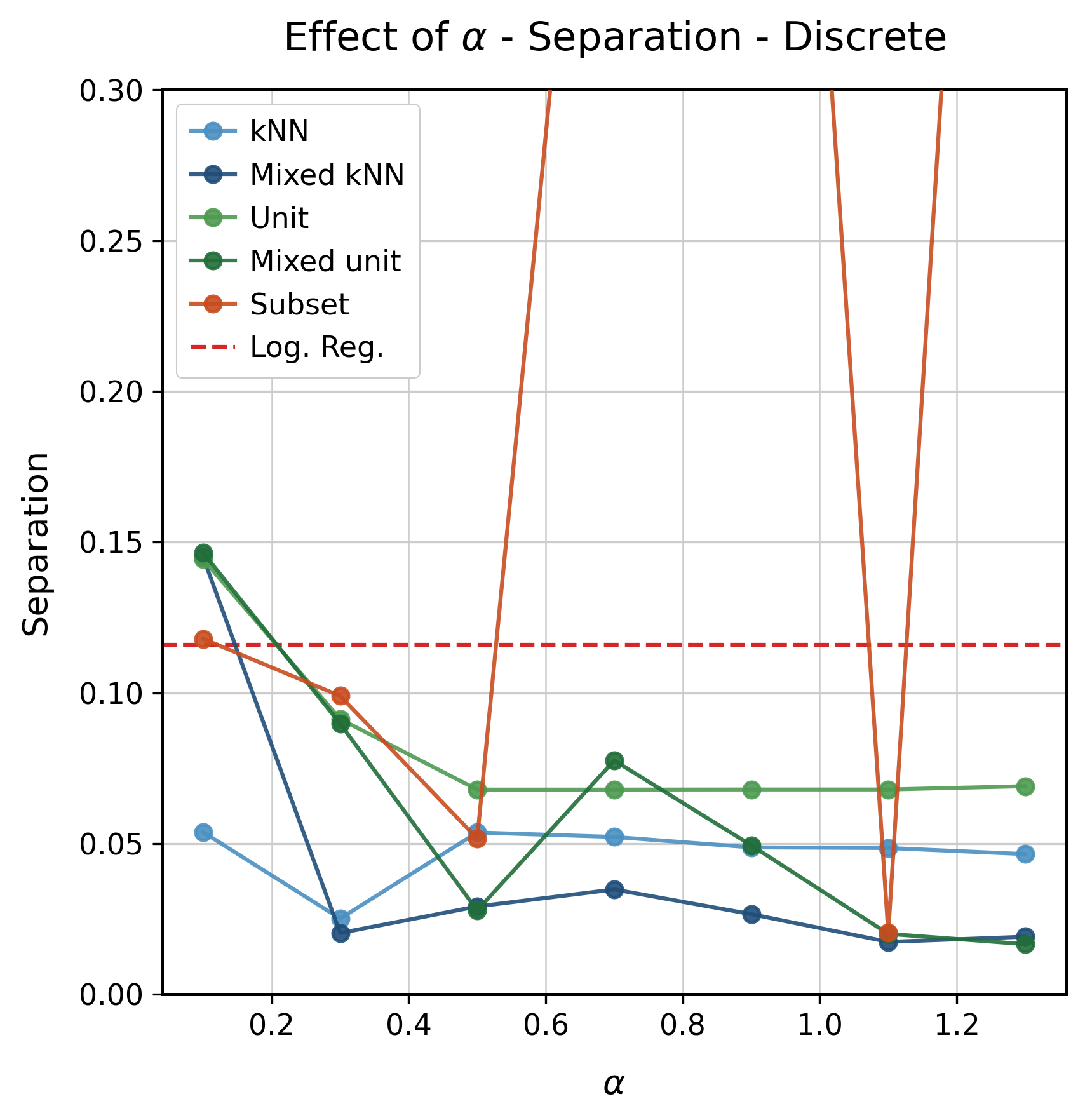}
  \includegraphics[width=.49\textwidth]{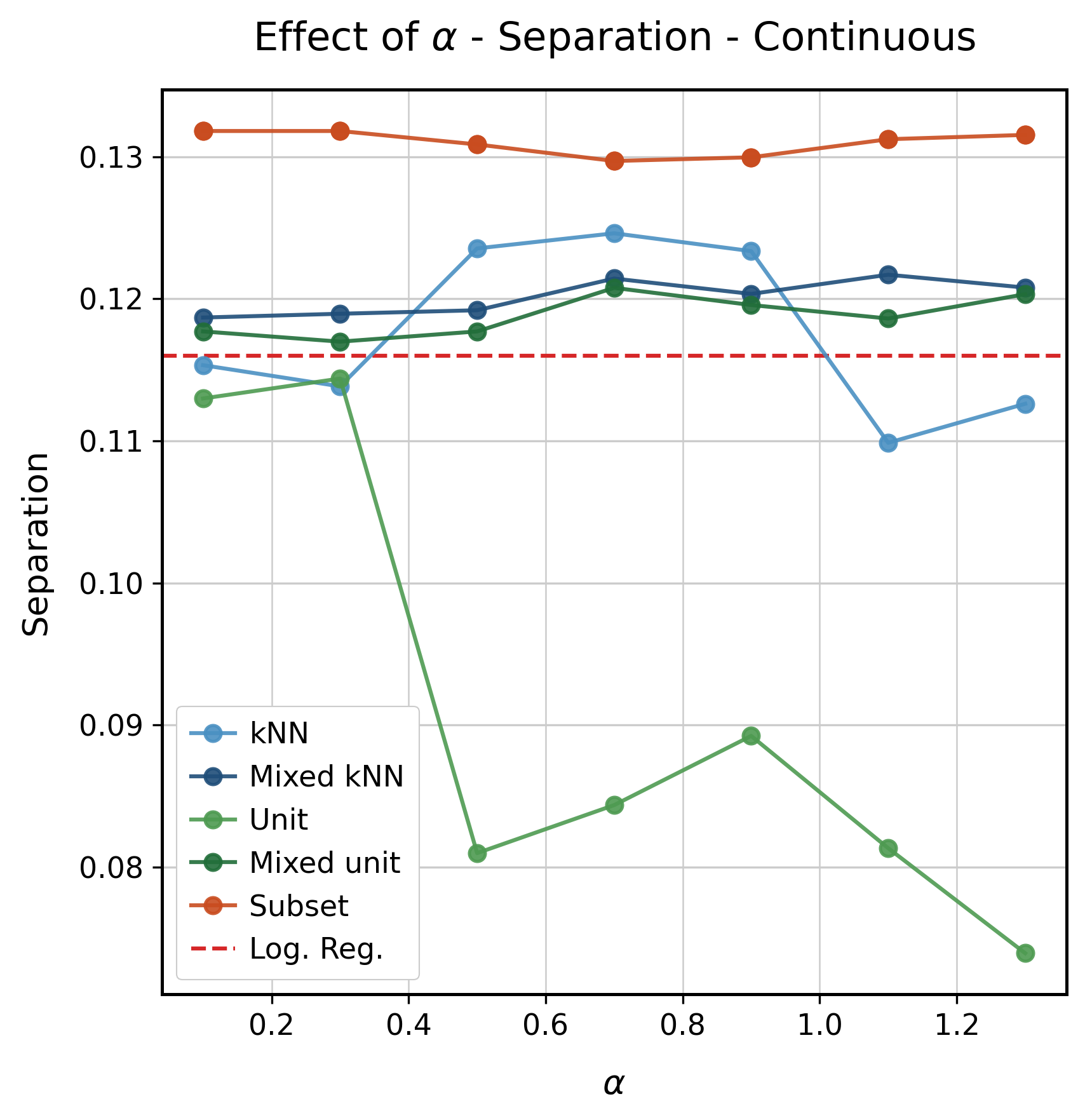}
    \caption{}
\end{figure}

\begin{figure}[!t]
  \centering
  \includegraphics[width=.49\textwidth]{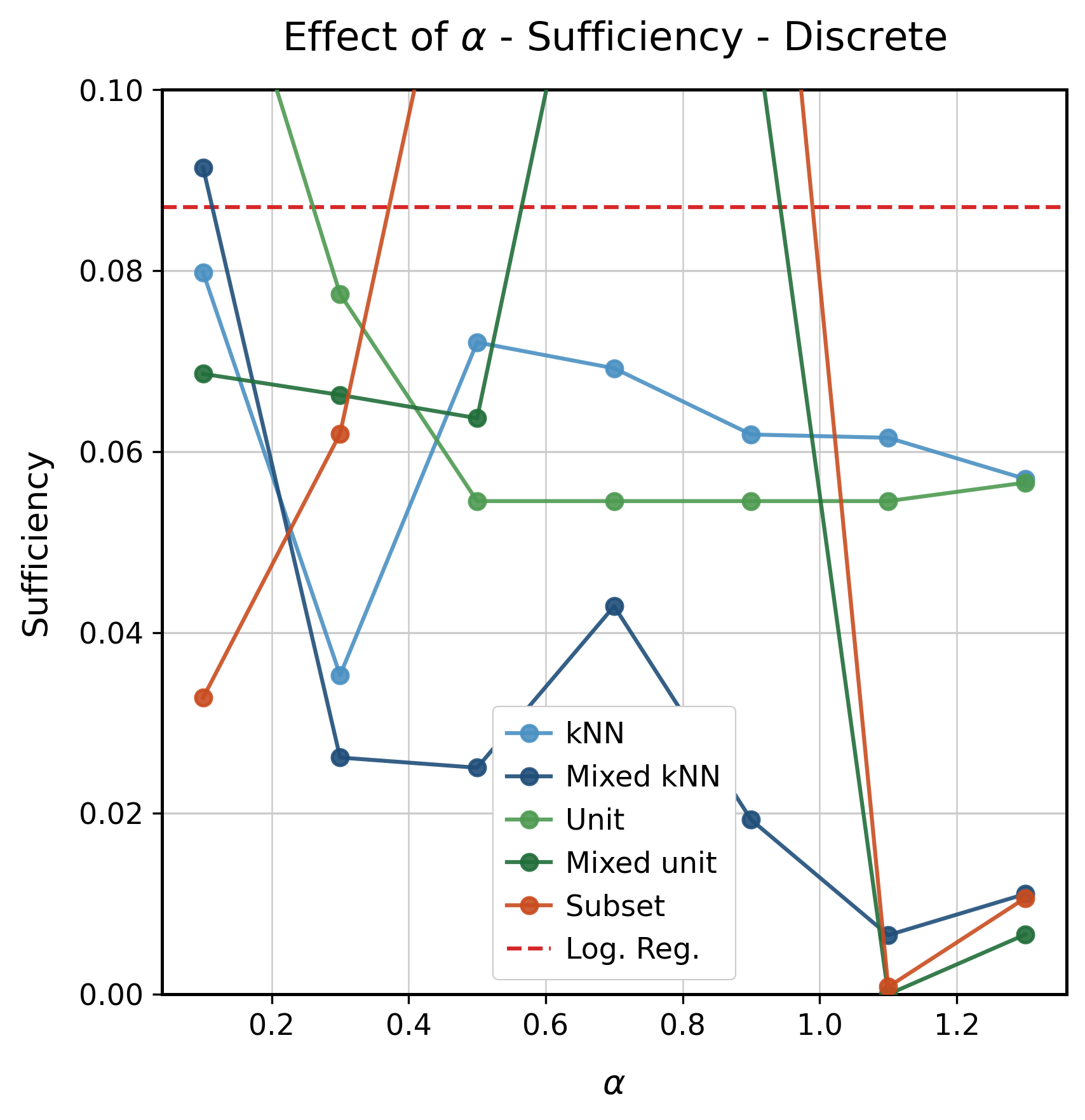}
  \includegraphics[width=.49\textwidth]{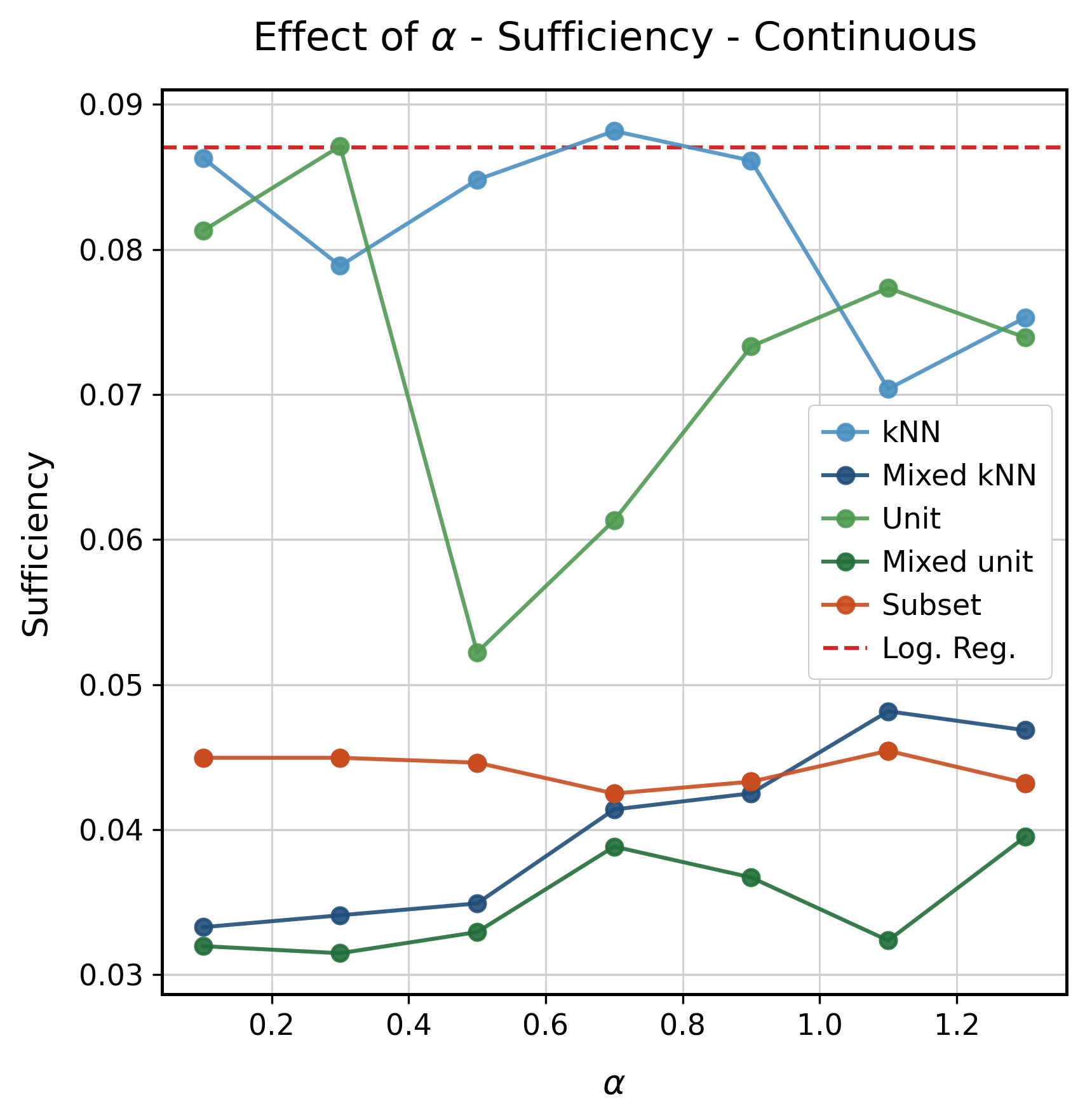}
    \caption{}
\end{figure}

\begin{figure}[!t]
  \centering
  \includegraphics[width=.49\textwidth]{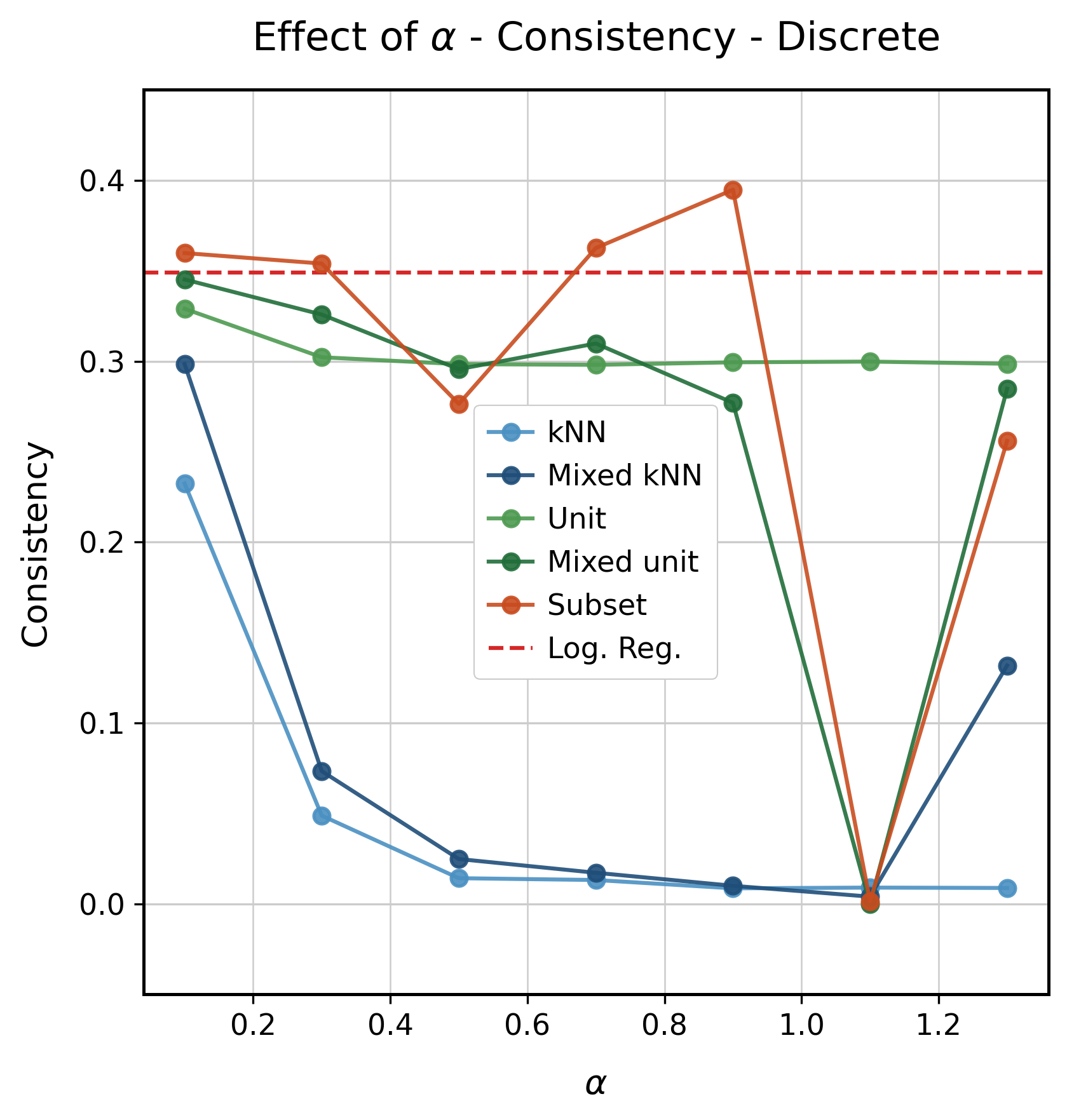}
  \includegraphics[width=.49\textwidth]{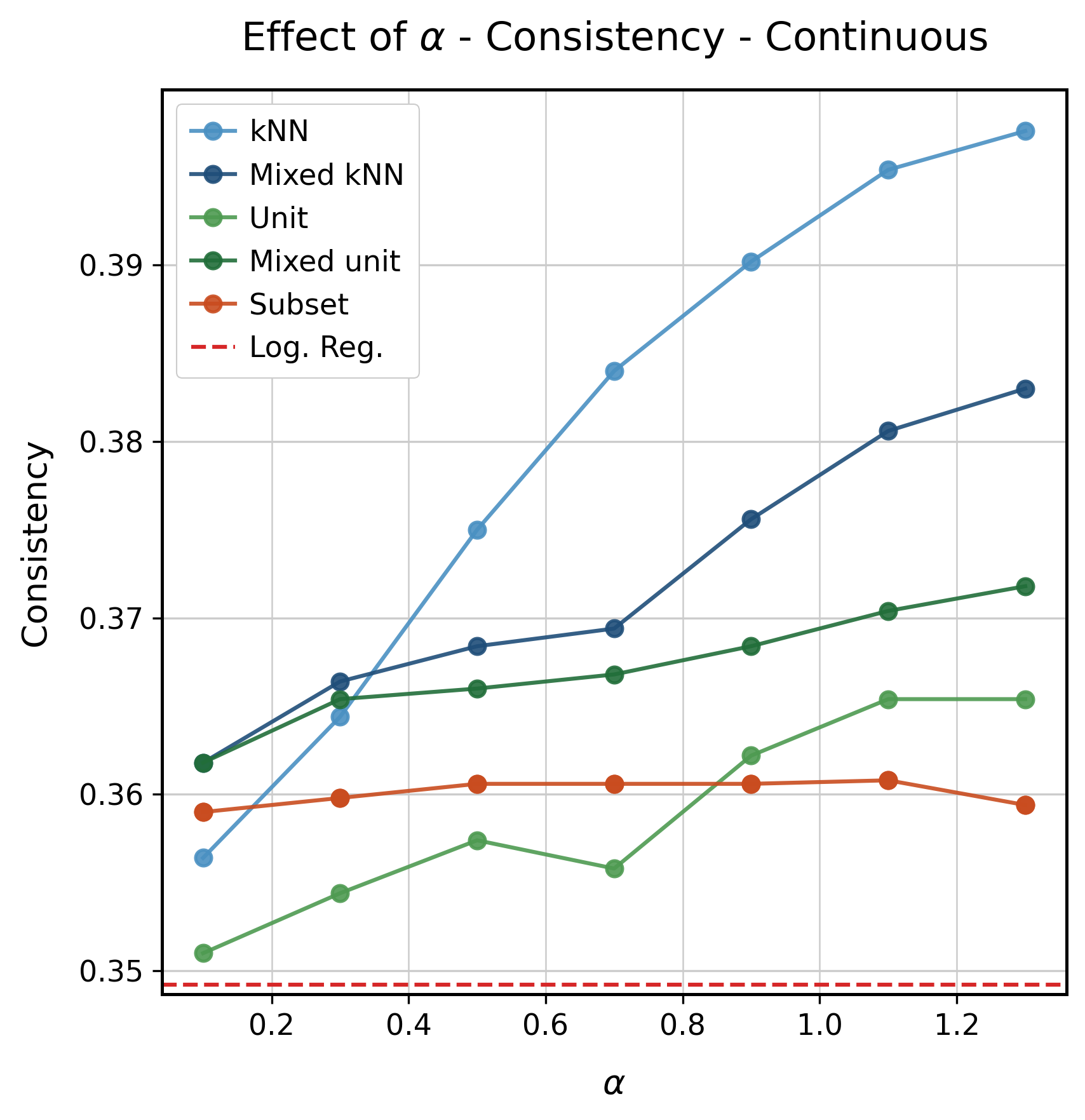}
    \caption{}
\end{figure}

\begin{figure}[!t]
  \centering
  \includegraphics[width=.49\textwidth]{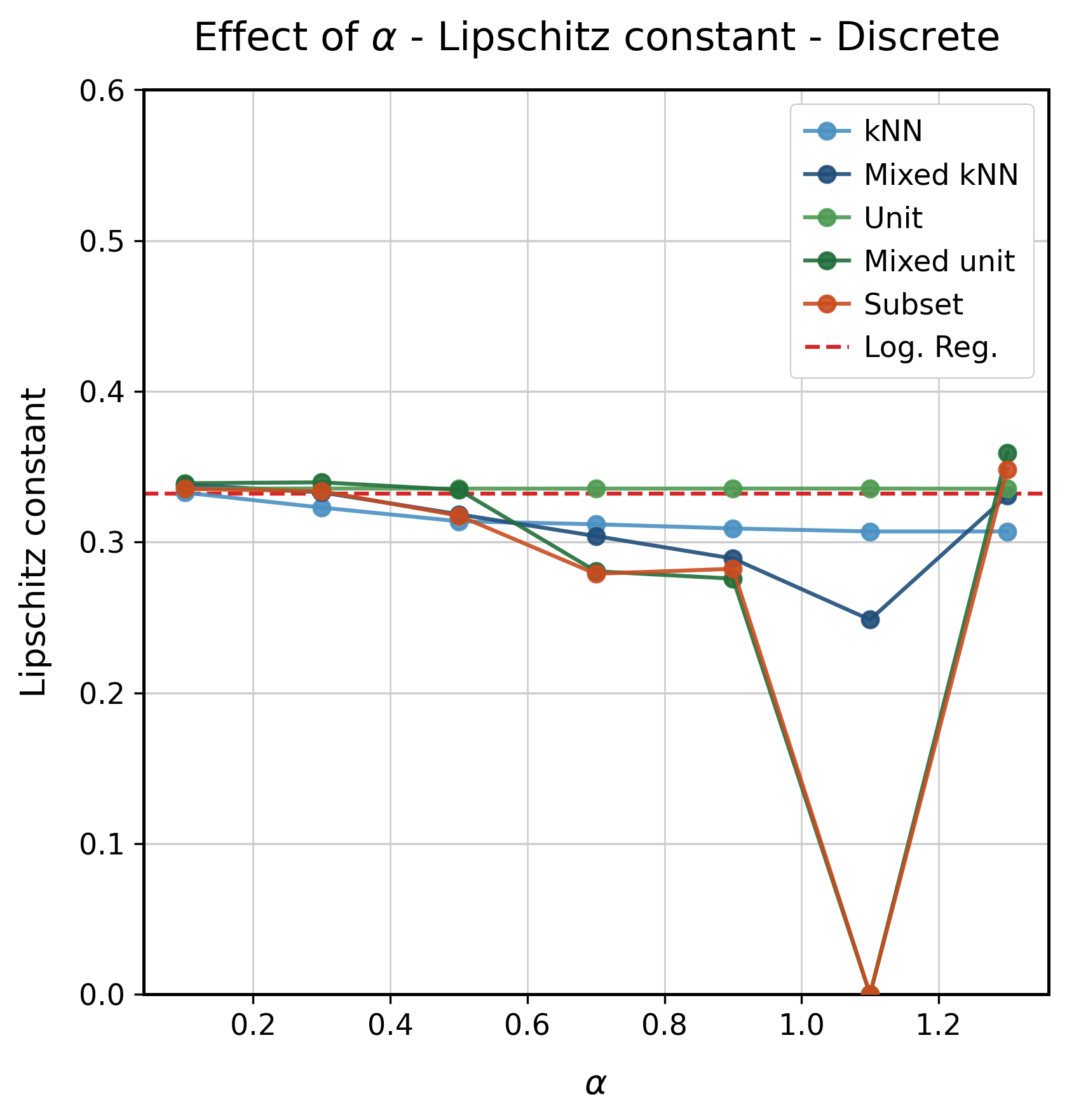}
  \includegraphics[width=.49\textwidth]{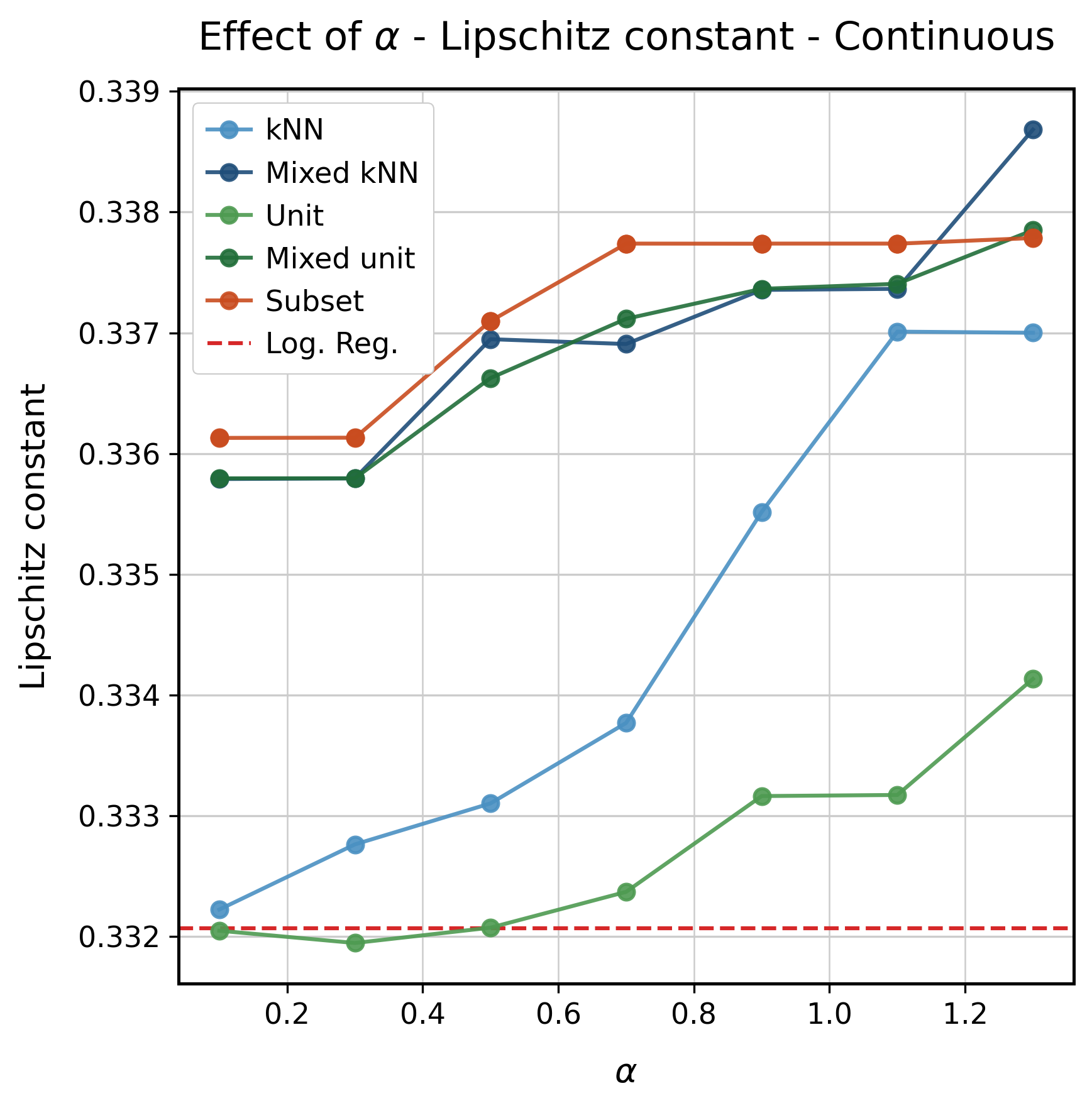}
    \caption{}
\end{figure}

\begin{figure}[!t]
  \centering
  \includegraphics[width=.49\textwidth]{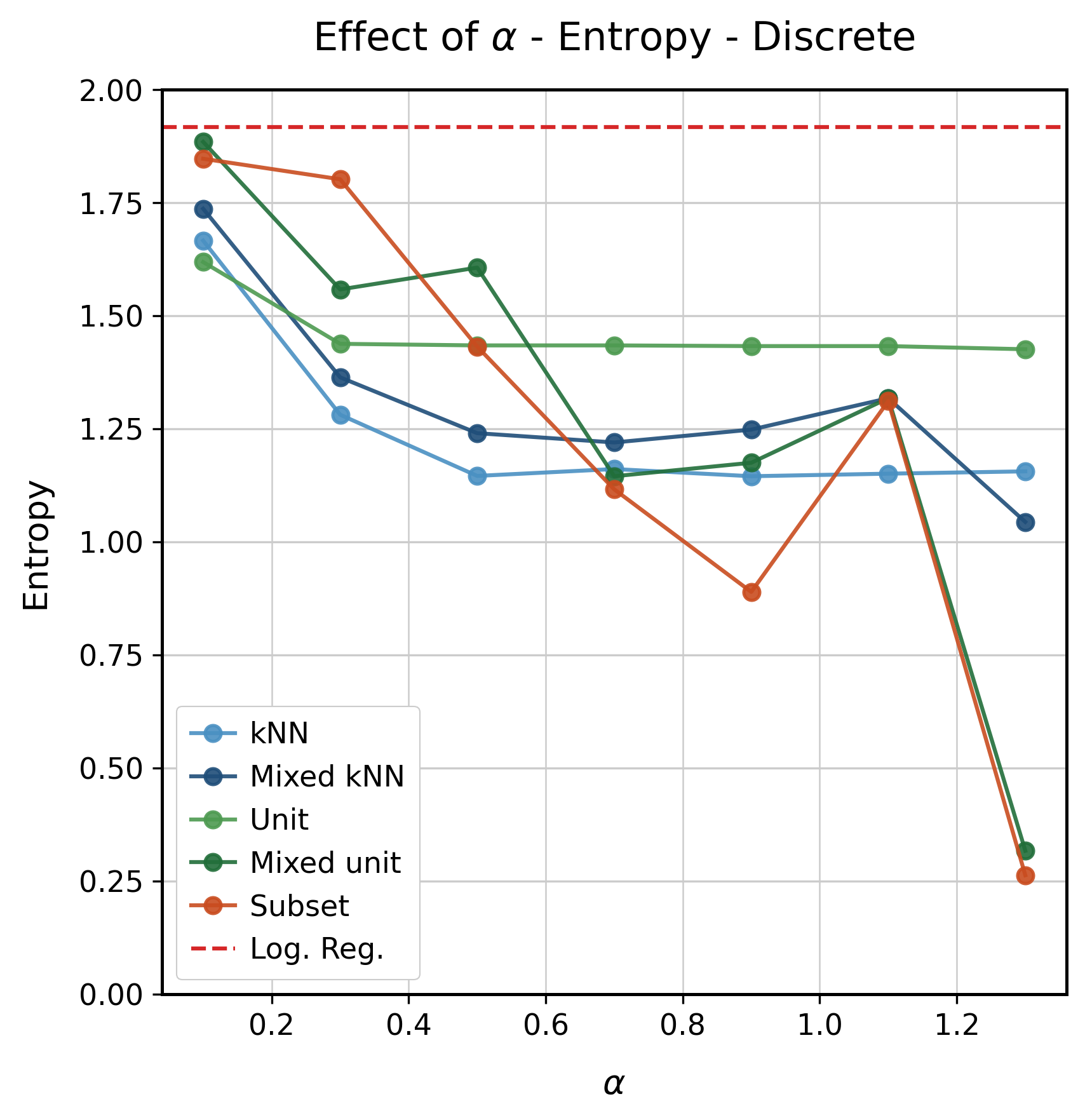}
  \includegraphics[width=.49\textwidth]{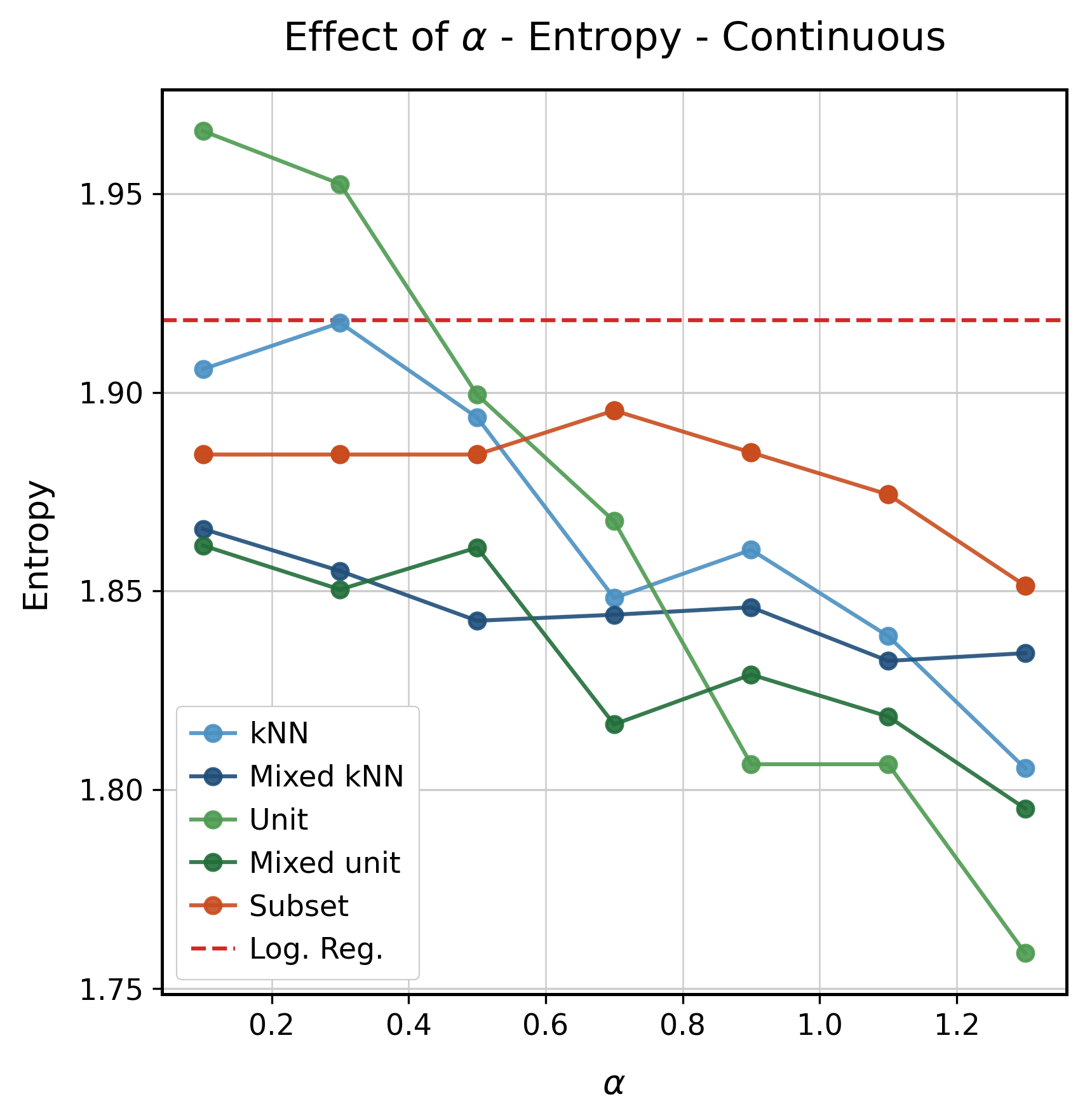}
  \caption{}
    \label{fig:LastCoef}
\end{figure}

The model sensitivity to $\alpha$ is shown in Figures \ref{fig:FirstCoef} to \ref{fig:LastCoef}. Starting with the discrete implementation, it is clear that the models get closer to logistic regression as the diffusion strength tends to zero.y in terms of accuracy and entropy. However, it seems like there are two different outcomes as $\alpha$ becomes bigger than one: Either the model converges as seen in the local topologies or it becomes unstable. This instability is mostly related to the the global topology as a consequence of the high degree of the aggregators \citep{HIGH-DEGREE}, although this issue is somewhat mitigated in the mixed kNN topology. The mixed unit ball graph, however, is also a victim to these unstabilities due to the sparsity of the underlying graph, once again hinting at the relatinship between the topology  of the data and the behavior of the model. This effect and the analysis for all topologies will remain a constant for the rest of the discussion. In any case, this phenomenom conditions the rest of the results, which nonetheless show  that up to a certain threshold accuracy, separation, sufficiency consistency and entropy become lower as $\alpha$ incerases, thus becoming patent the role of this parameter on the fairness-performance trade-off. The kNN graph also exhibits this behavior for independence and, to a lesser degree, in the lipschitz constant. on the other hand, changing $\alpha$ in continuous models results in a more chaotic pattern which is not explained by the theory. For this reason we encourage the use of the discrete implementation.

\begin{figure}[!t]
  \centering
  \includegraphics[width=.49\textwidth]{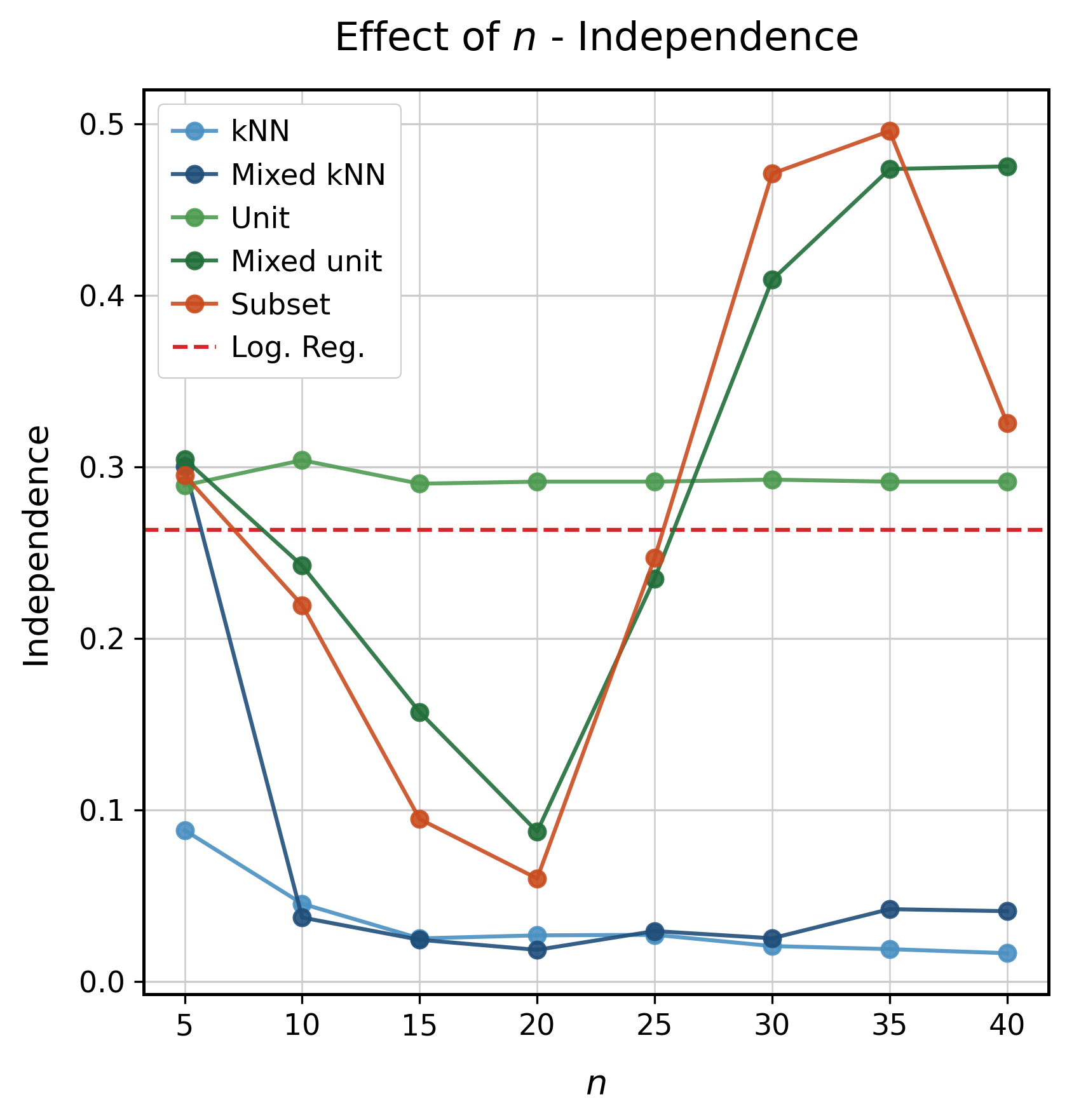}
  \includegraphics[width=.49\textwidth]{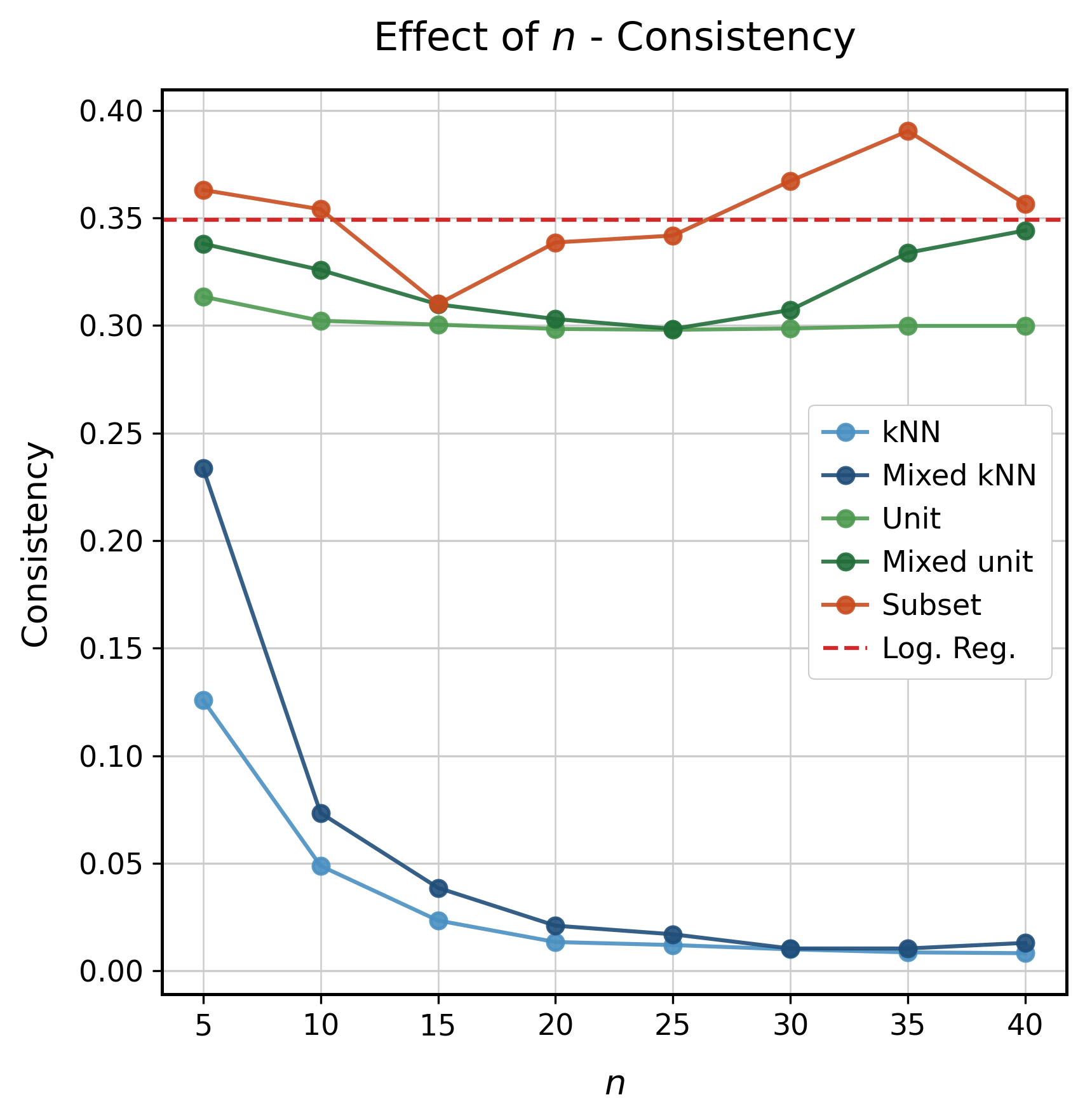}
    \caption{}
  \label{fig:FirstLayers}
\end{figure}

\begin{figure}[!t]
  \centering
  \includegraphics[width=.49\textwidth]{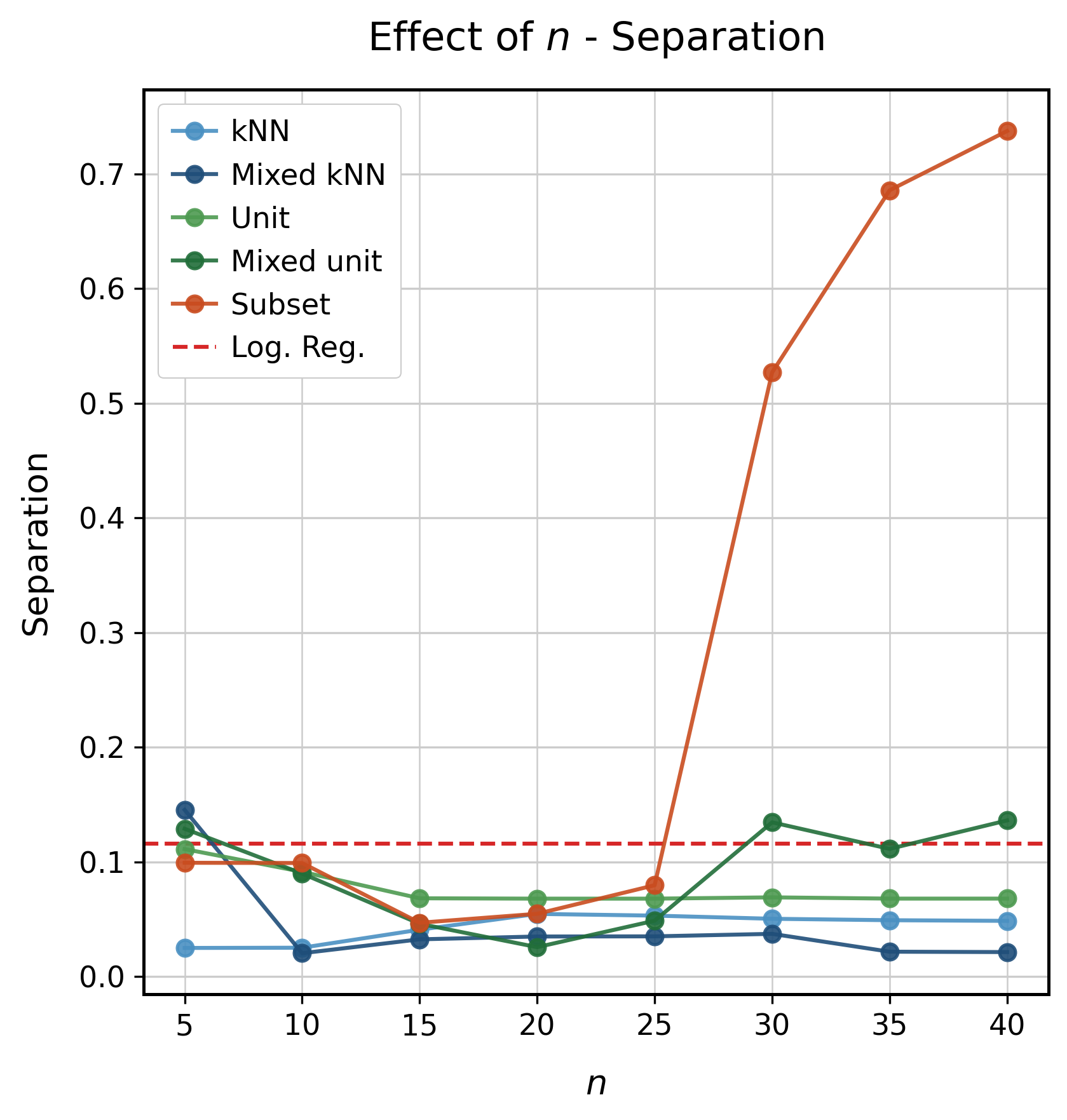}
  \includegraphics[width=.49\textwidth]{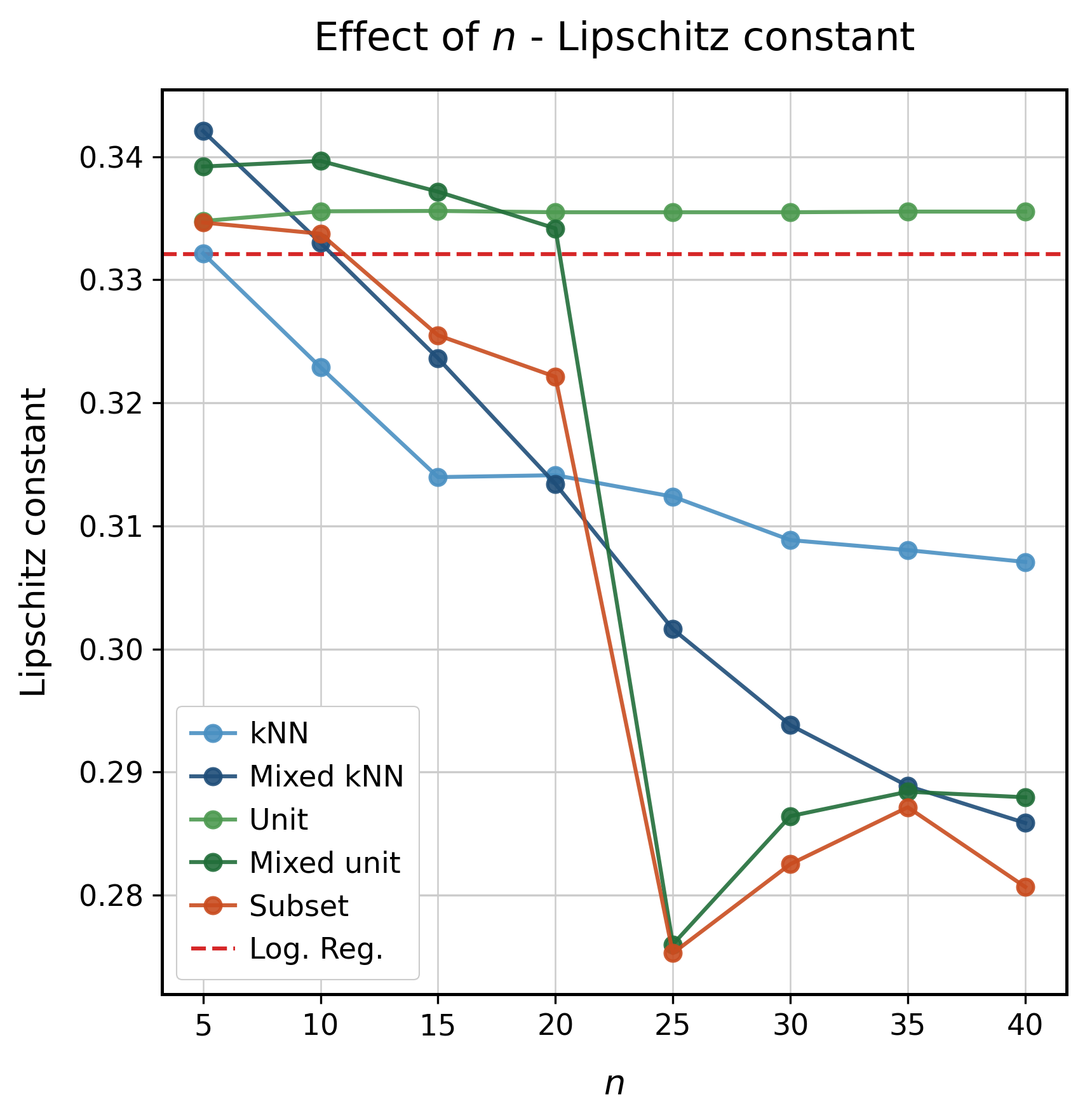}
    \caption{}
\end{figure}

\begin{figure}[!t]
  \centering
  \includegraphics[width=.49\textwidth]{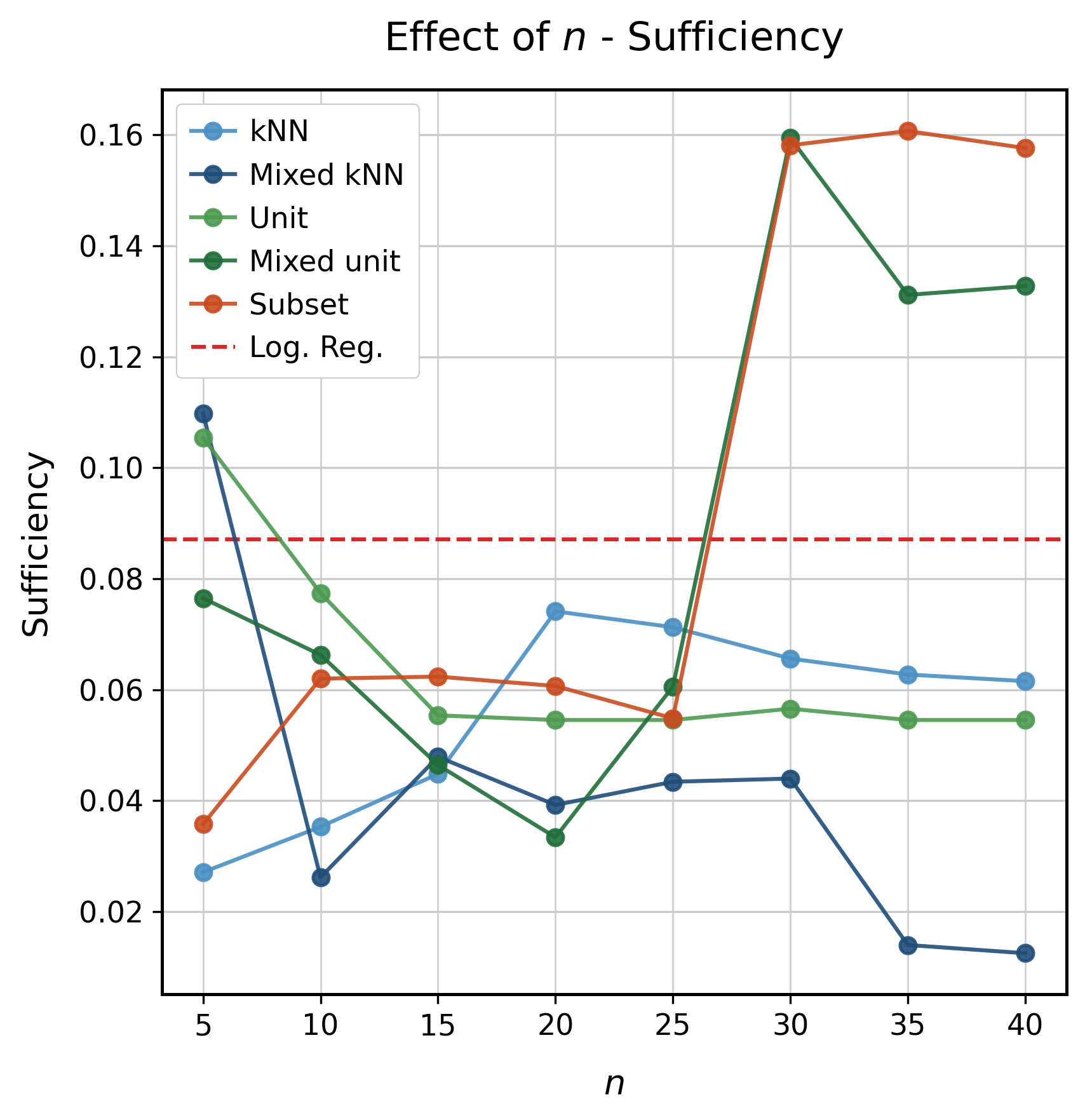}
  \includegraphics[width=.49\textwidth]{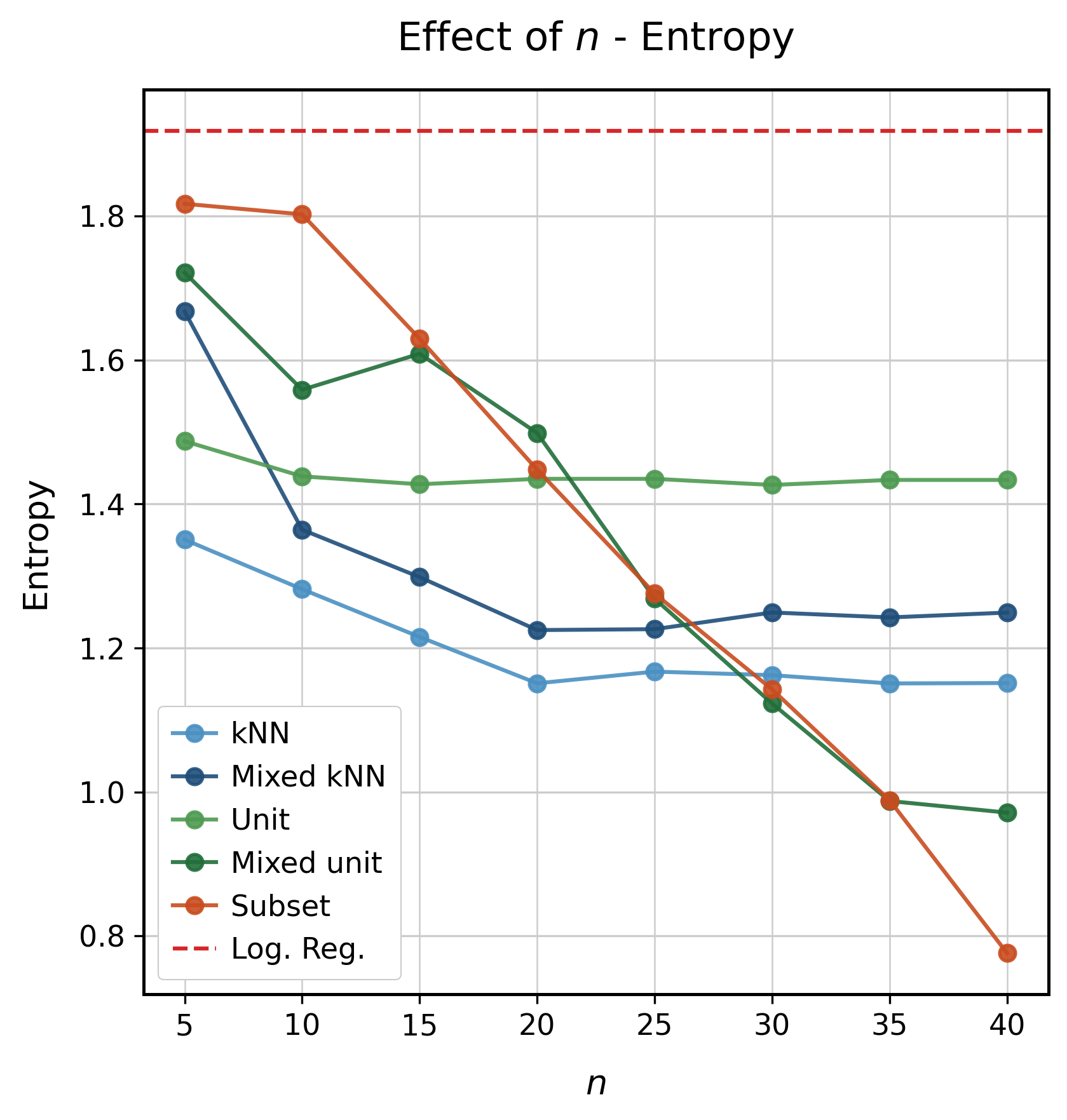}
    \caption{}
\end{figure}

\begin{figure}[!t]
  \centering
  \includegraphics[width=.49\textwidth]{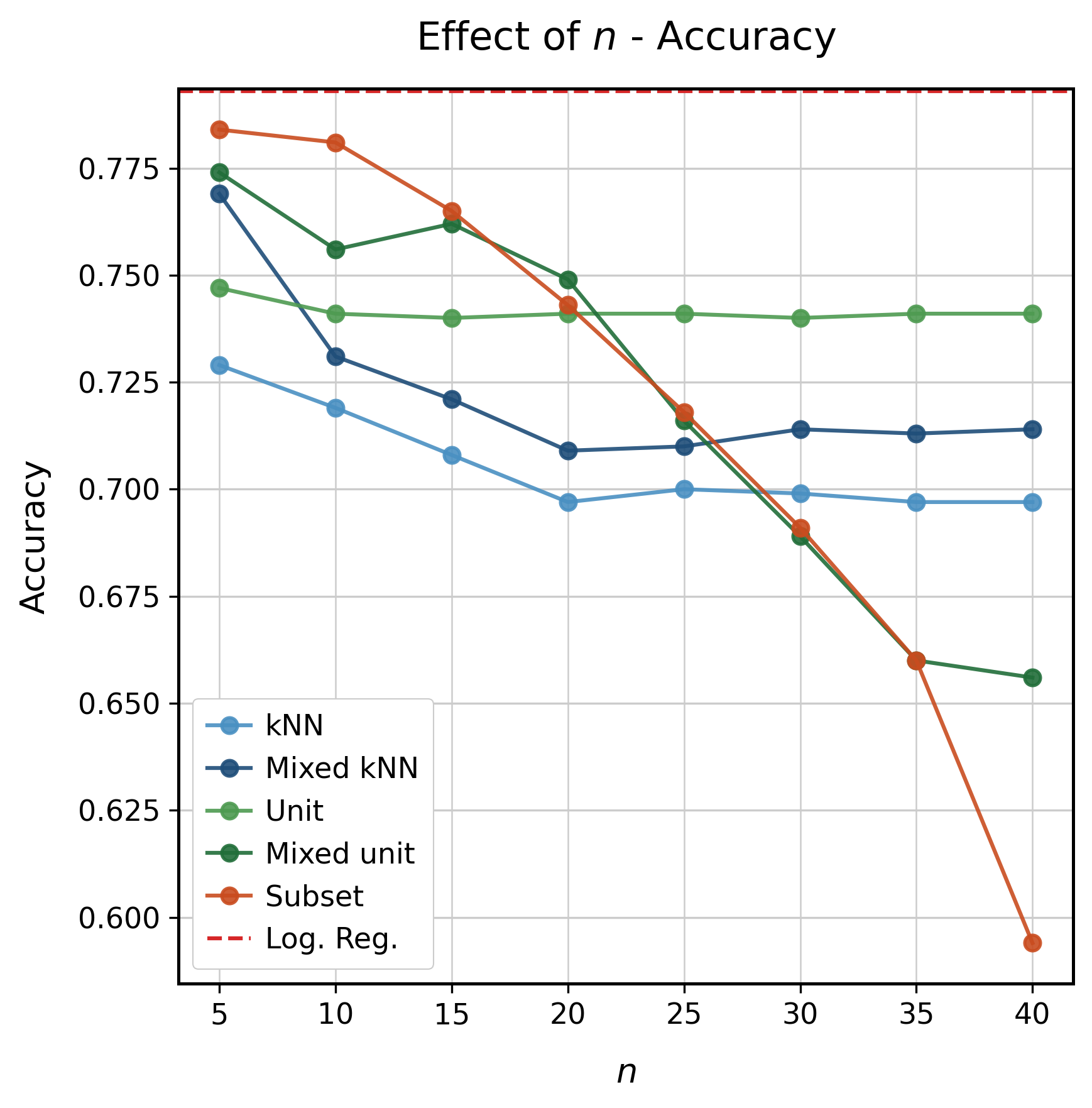}
    \caption{}
  \label{fig:LastLayers}
\end{figure}

\begin{figure}[!t]
  \centering
  \includegraphics[width=.49\textwidth]{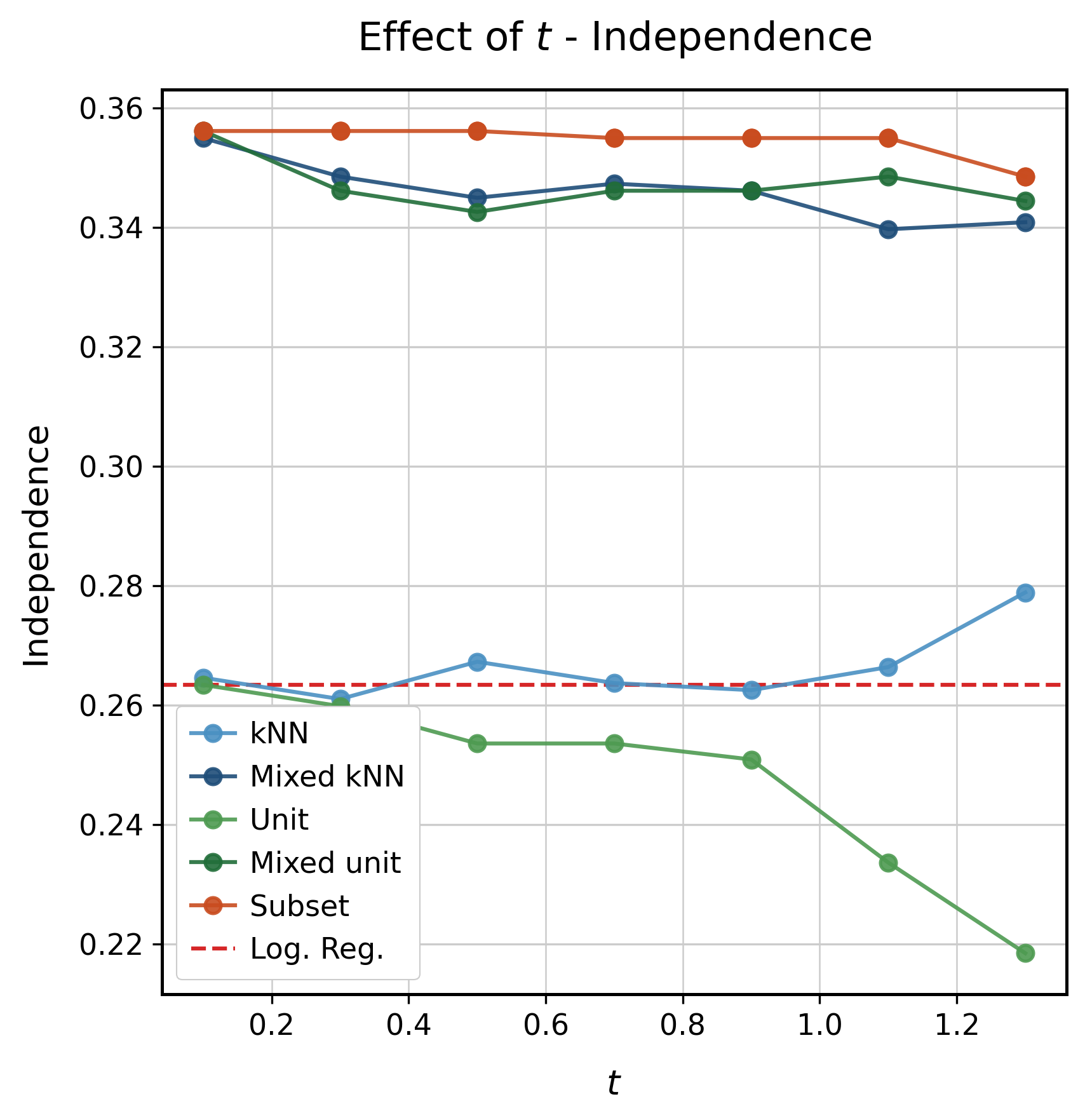}
  \includegraphics[width=.49\textwidth]{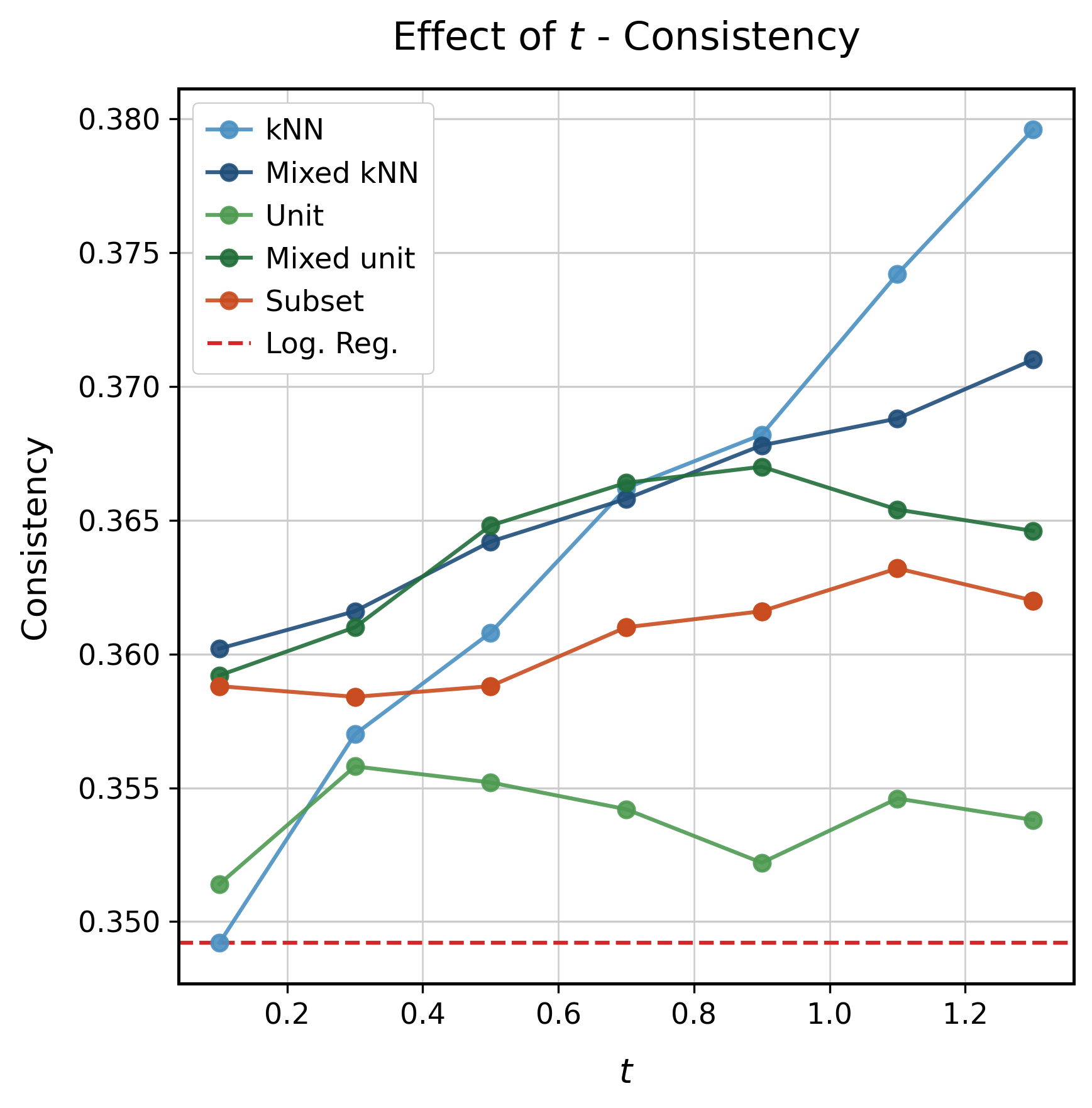}
    \caption{}
  \label{fig:FirstTime}
\end{figure}

\begin{figure}[!t]
  \centering
  \includegraphics[width=.49\textwidth]{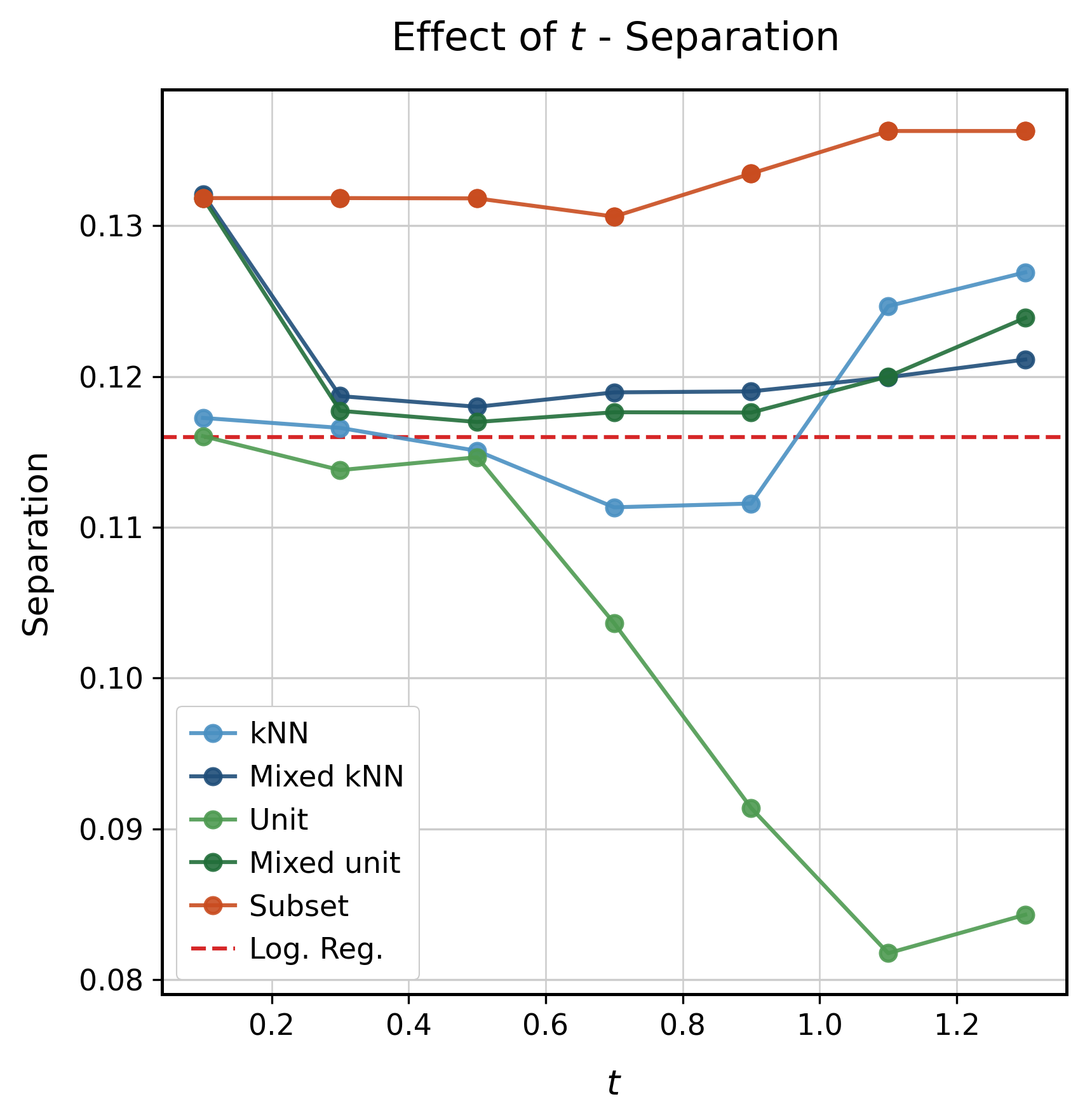}
  \includegraphics[width=.49\textwidth]{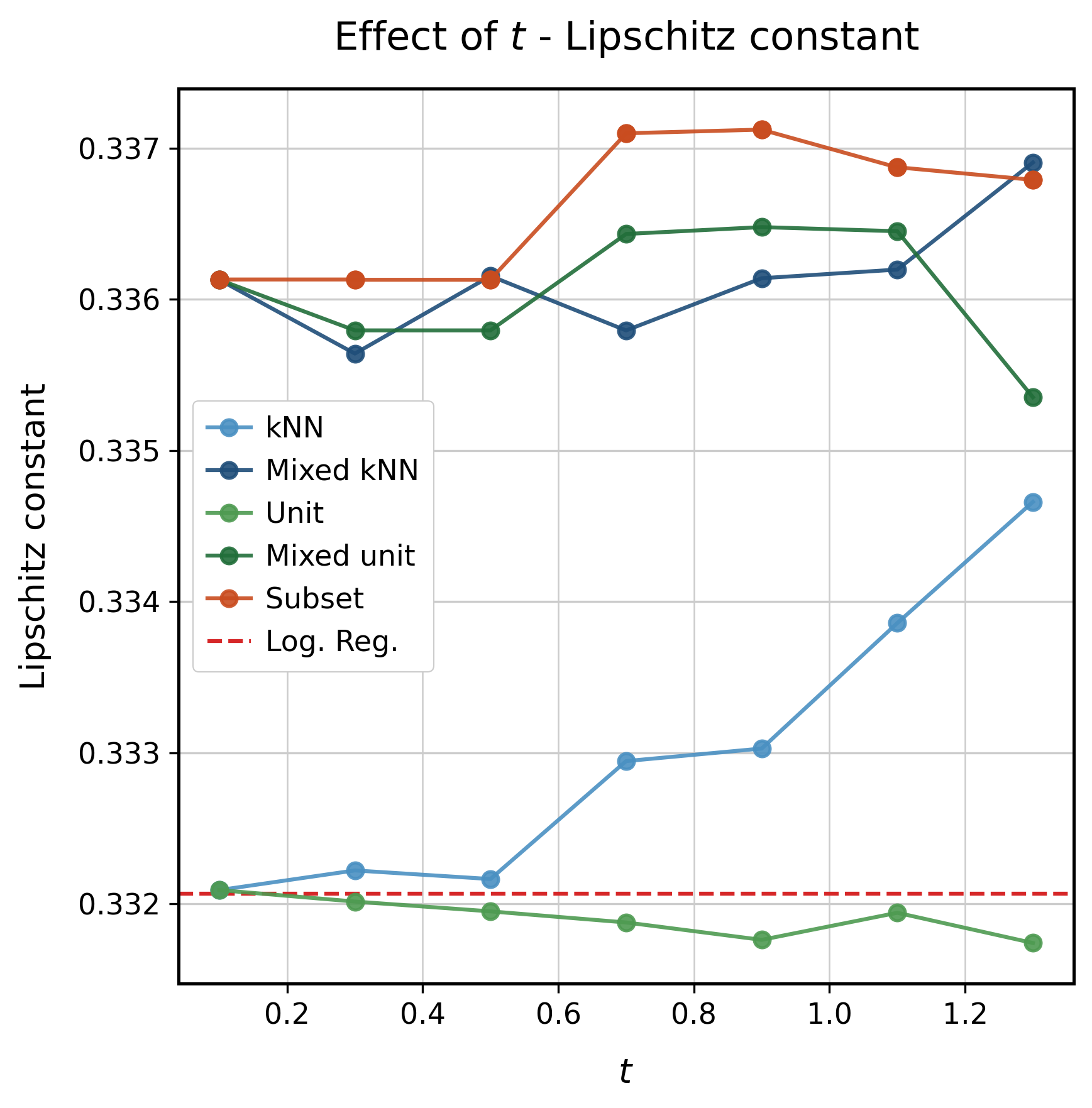}
    \caption{}
\end{figure}

\begin{figure}[!t]
  \centering
  \includegraphics[width=.49\textwidth]{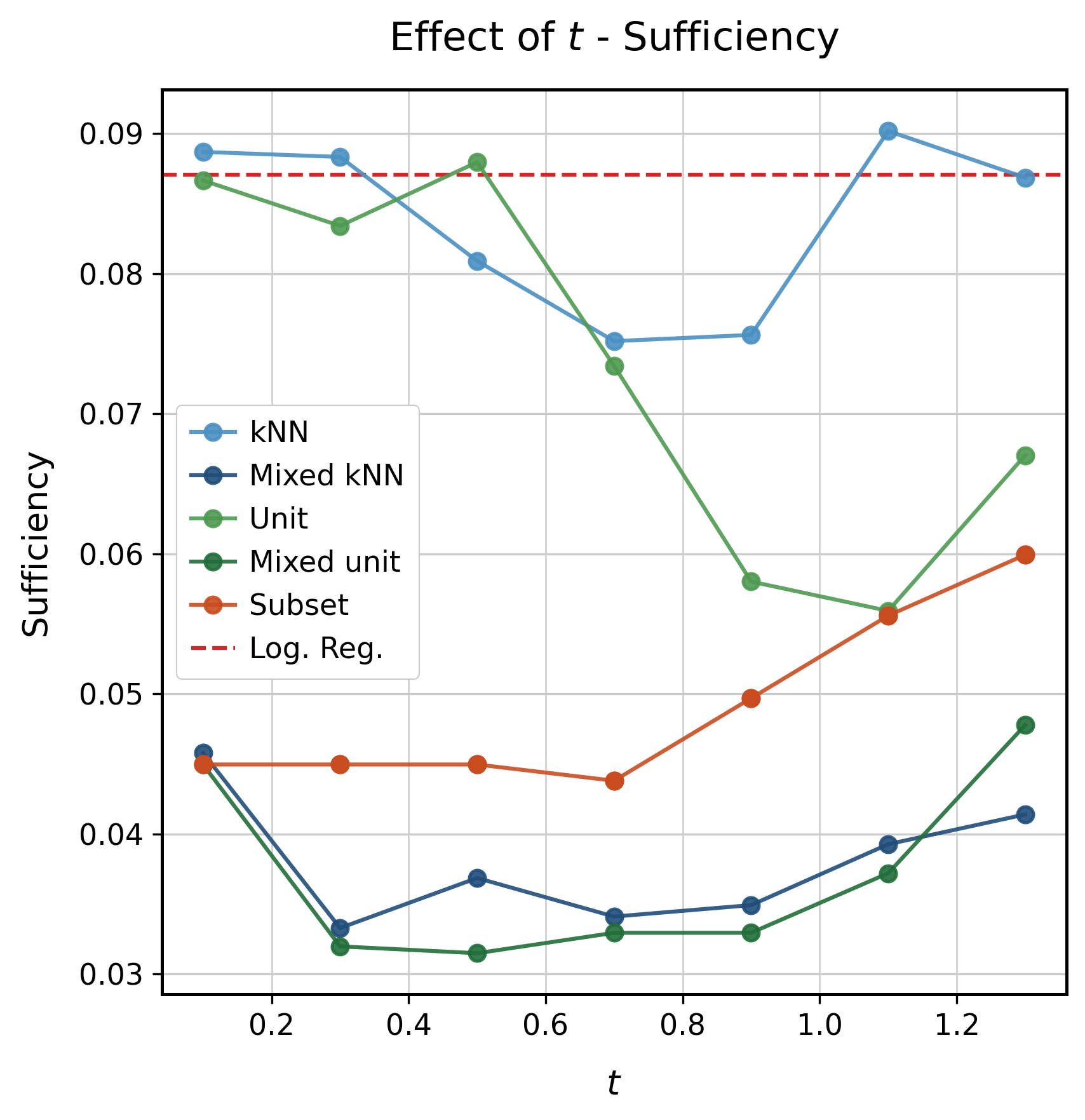}
  \includegraphics[width=.49\textwidth]{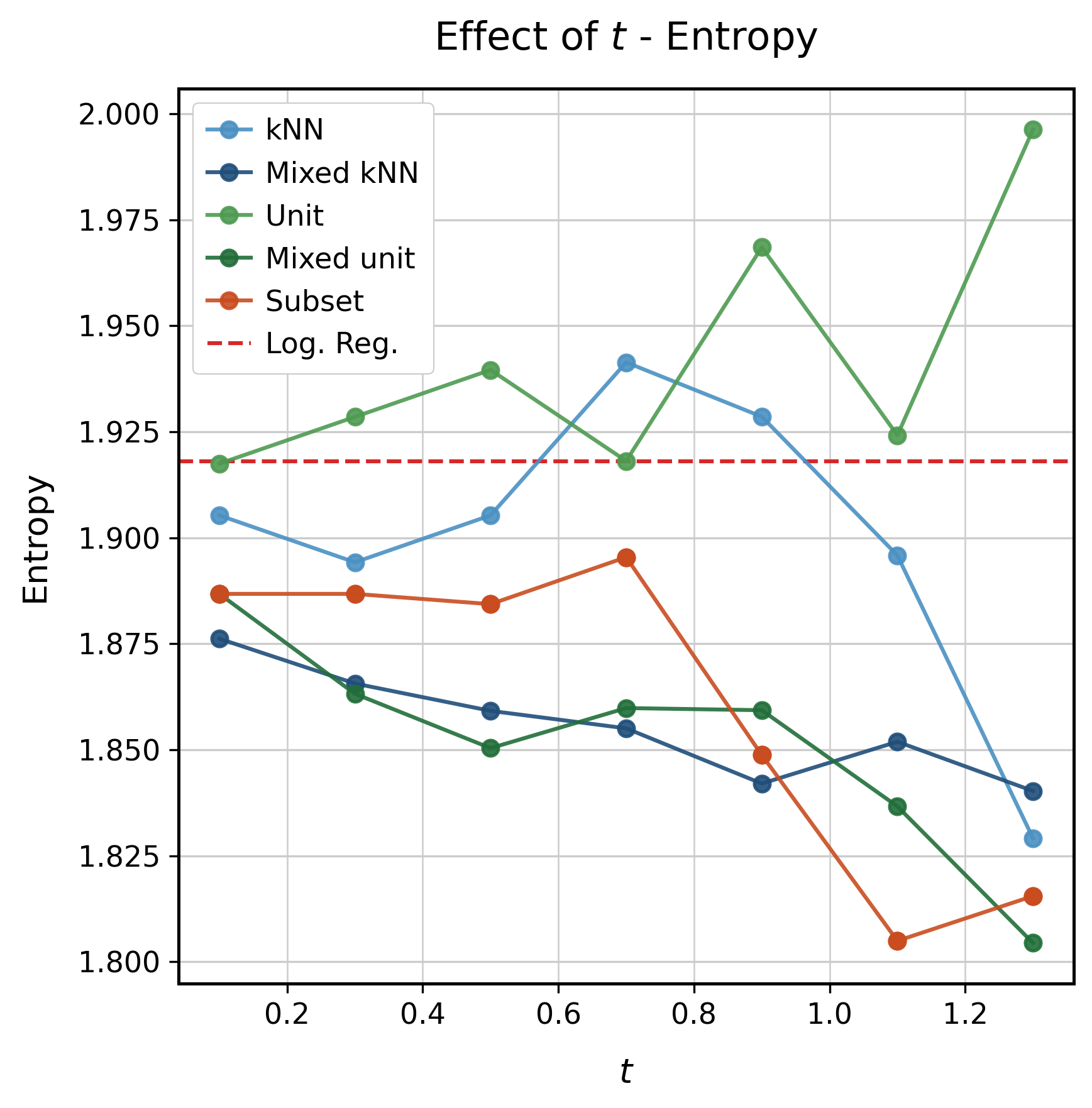}
    \caption{}
\end{figure}

\begin{figure}[!t]
  \centering
  \includegraphics[width=.49\textwidth]{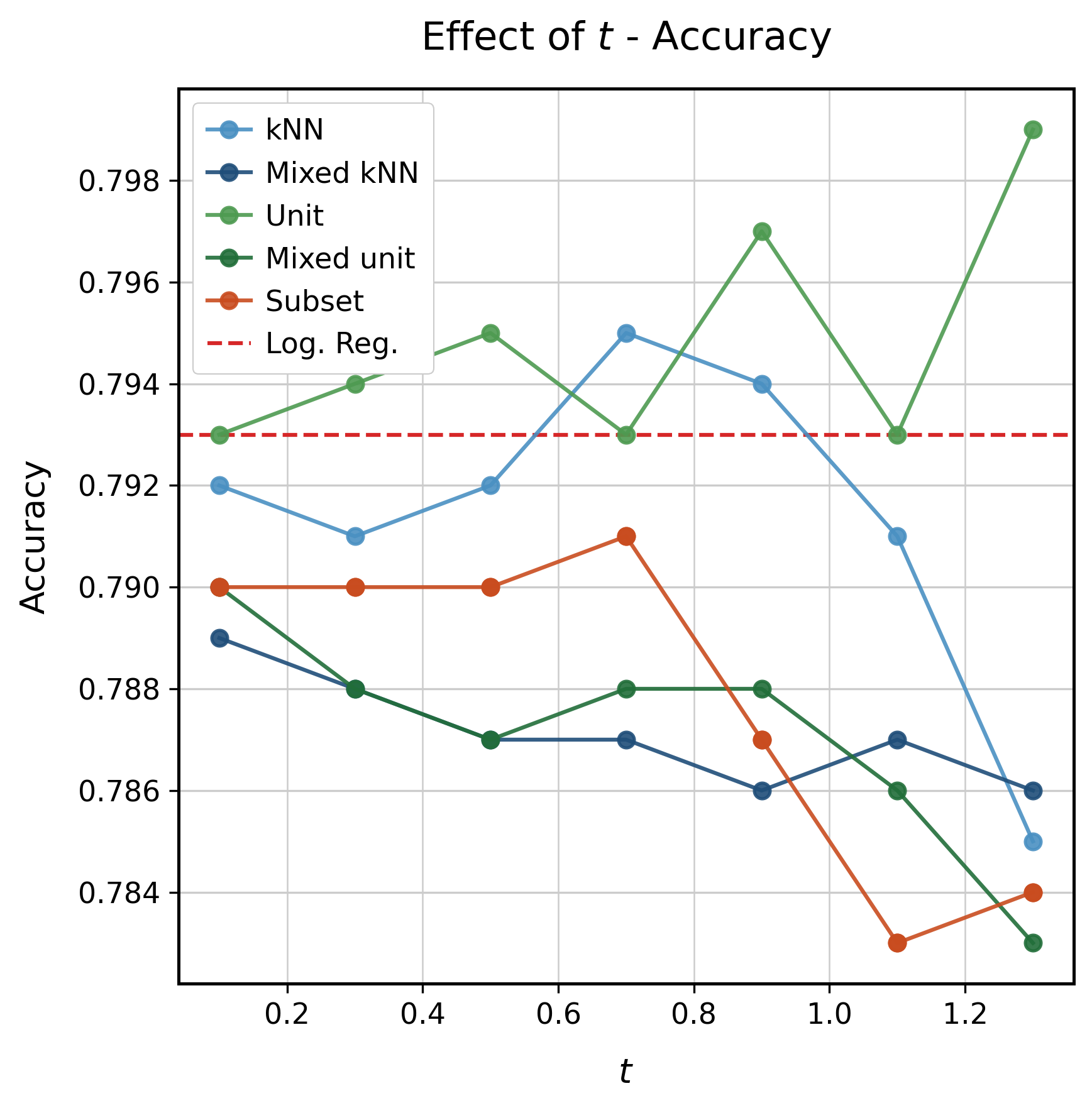}
    \caption{}
  \label{fig:LastTime}
\end{figure}

Moving on to the analysis of the length of the diffusion process, Figures \ref{fig:FirstLayers} to \ref{fig:LastLayers} show the more consistent effect of the number of layers, which, similarly to $\alpha$, exacerbates the fairness-accuracy trade-offs. The higher the number of layers, the lower most metrics, the only exception that does not involve the global topology being found in sufficiency, although this is not surprising due to the shortcomings of this metric in actually measuring fairness \citep{EJORSurvey}. Once again, the subset configuration seems to become unstable for high values of $n$, while the local graphs converge to the trivial solution. However, using low values of $n$ leads to enhanced fairness while barely compromising performance. Therefore, in order to enforce fairness constraints we recommend tweaking $n$ rather than $\alpha$ for discrete models. Integration time, whose effects are shown in Figures \ref{fig:FirstTime} to \ref{fig:LastTime}, shows some deviations with respect to the expected behavior. Namely, while accuracy is degraded as $t$ becomes higher, fairness metrics barely improve if at all. Moreover, individual fairness metrics are compromised as the time is increased with the exception of entropy.\\
Finally, the previous analysis shows that, even in the presence of instabilities, continuous models are more unpredictable, with behaviors deviating from those dictated from the theoretical framework and, although unreported, their training times exceed those of their discrete counterparts. All in all, these findings make the use of discrete implementations more attractive.

\subsubsection{Effect of $k$ and $\delta$ on local topologies}
The focus is now switched to the parameters that define the local topologies, those being the number of neighbors in the case of the kNN graph and the radius in the unit ball topology. Results are shown in Figures \ref{fig:FirstK} to \ref{fig:LastDelta}, which analyze the effect of each parameter on both pure and mixed configurations. Importantly, the parameter $\delta$ that we used is not the distance per se, but a quantile of the distribution of distances in the dataset. Concretely, the distance we are using is the quantile $Q(\delta/N; \{d(x_i, x_j) \mid x_i, x_j \in X\})$, which means that we are considering approximately $N \delta$ edges. This facilitates comparison between different datasets.\\

\begin{figure}[!t]
  \centering
  \includegraphics[width=.49\textwidth]{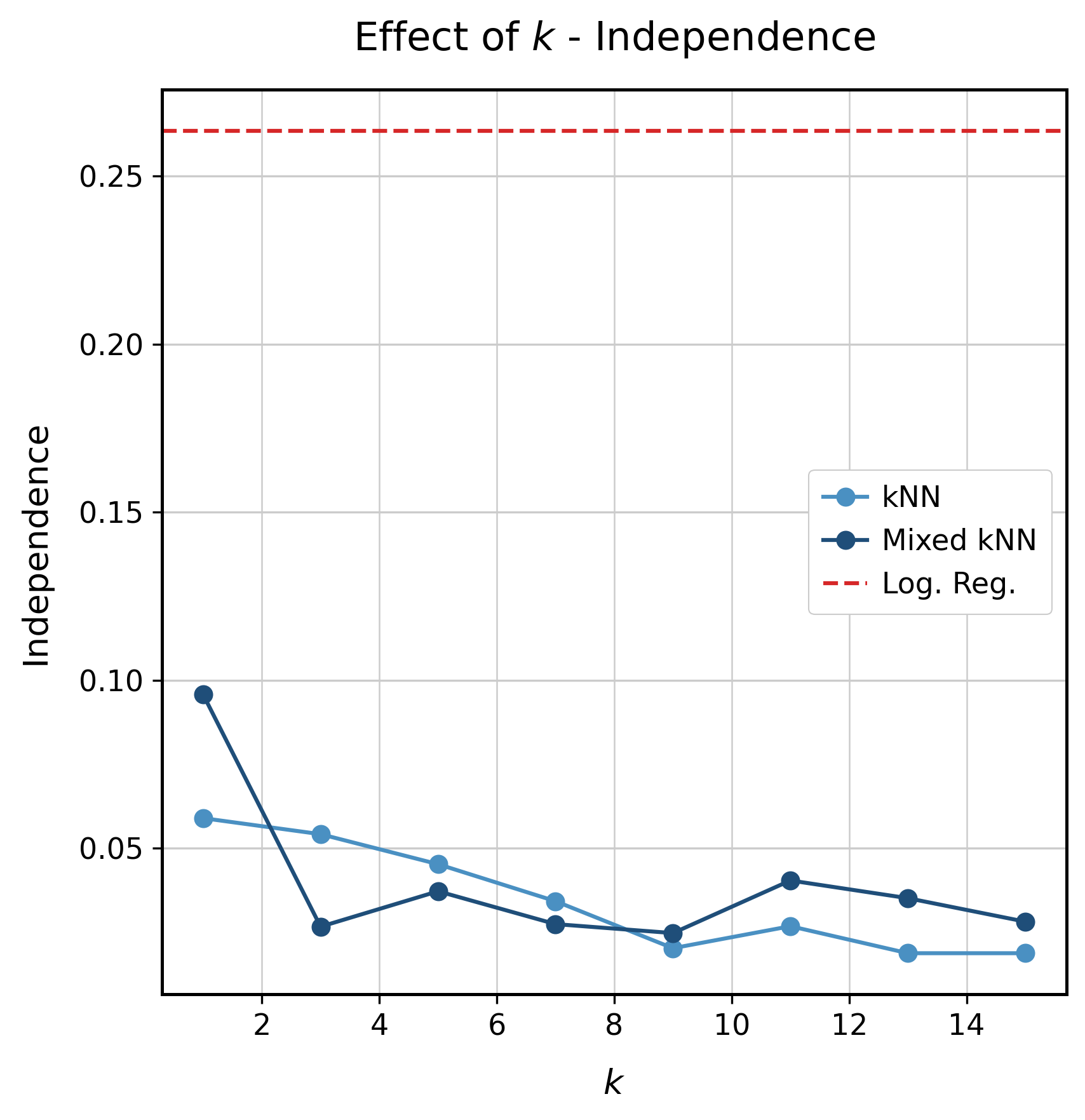}
  \includegraphics[width=.49\textwidth]{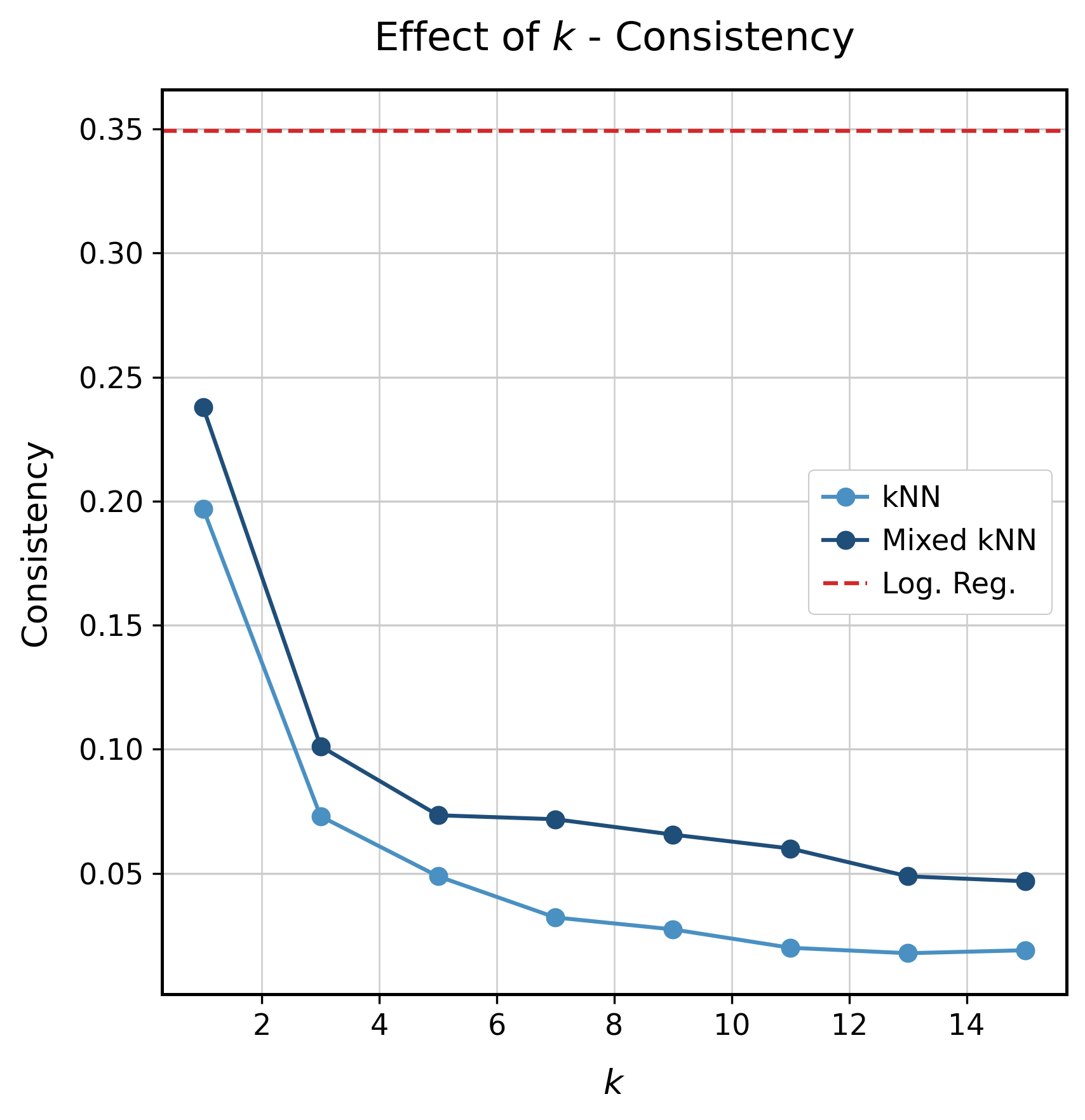}
    \caption{}
  \label{fig:FirstK}
\end{figure}

\begin{figure}[!t]
  \centering
  \includegraphics[width=.49\textwidth]{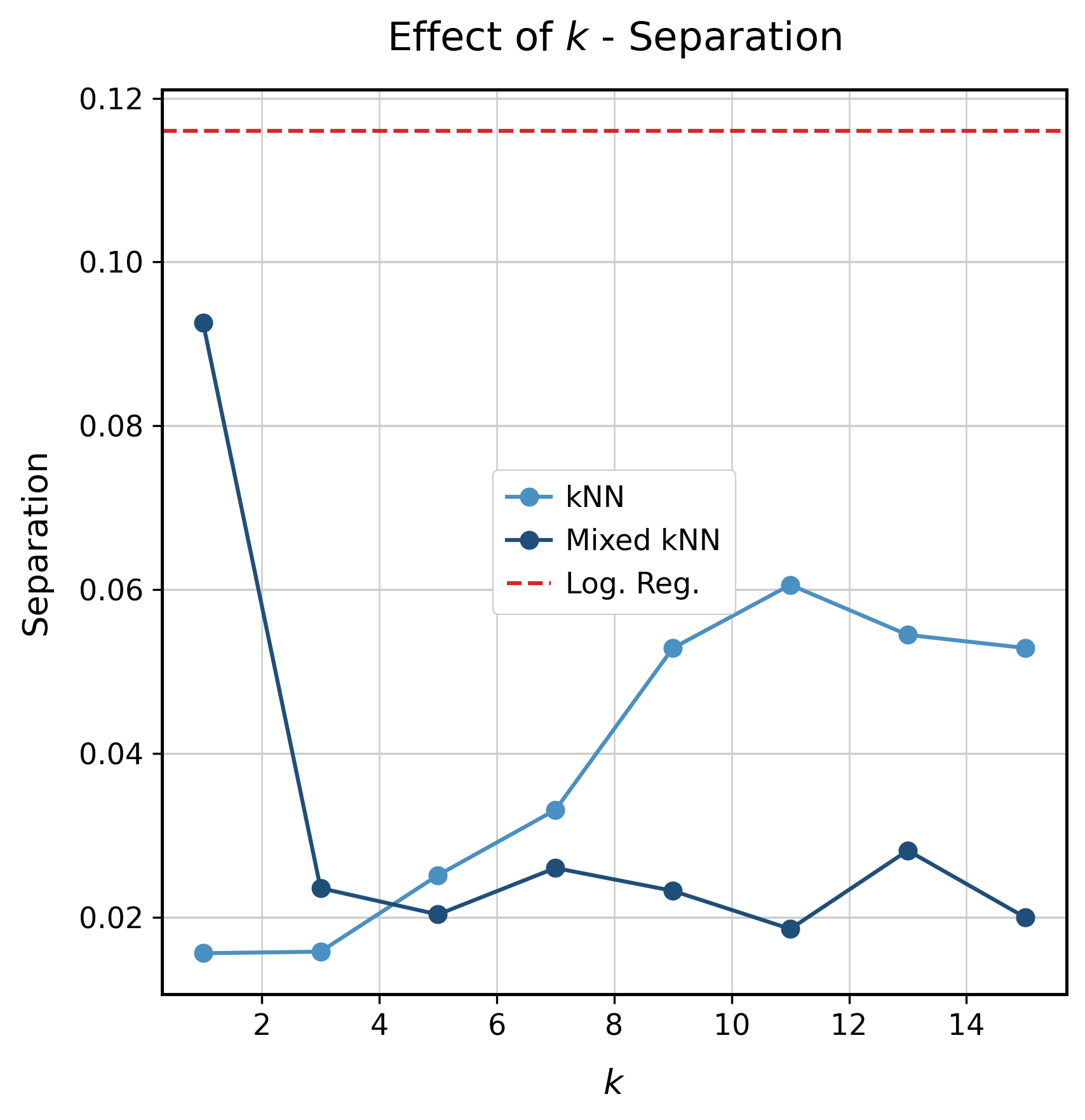}
  \includegraphics[width=.49\textwidth]{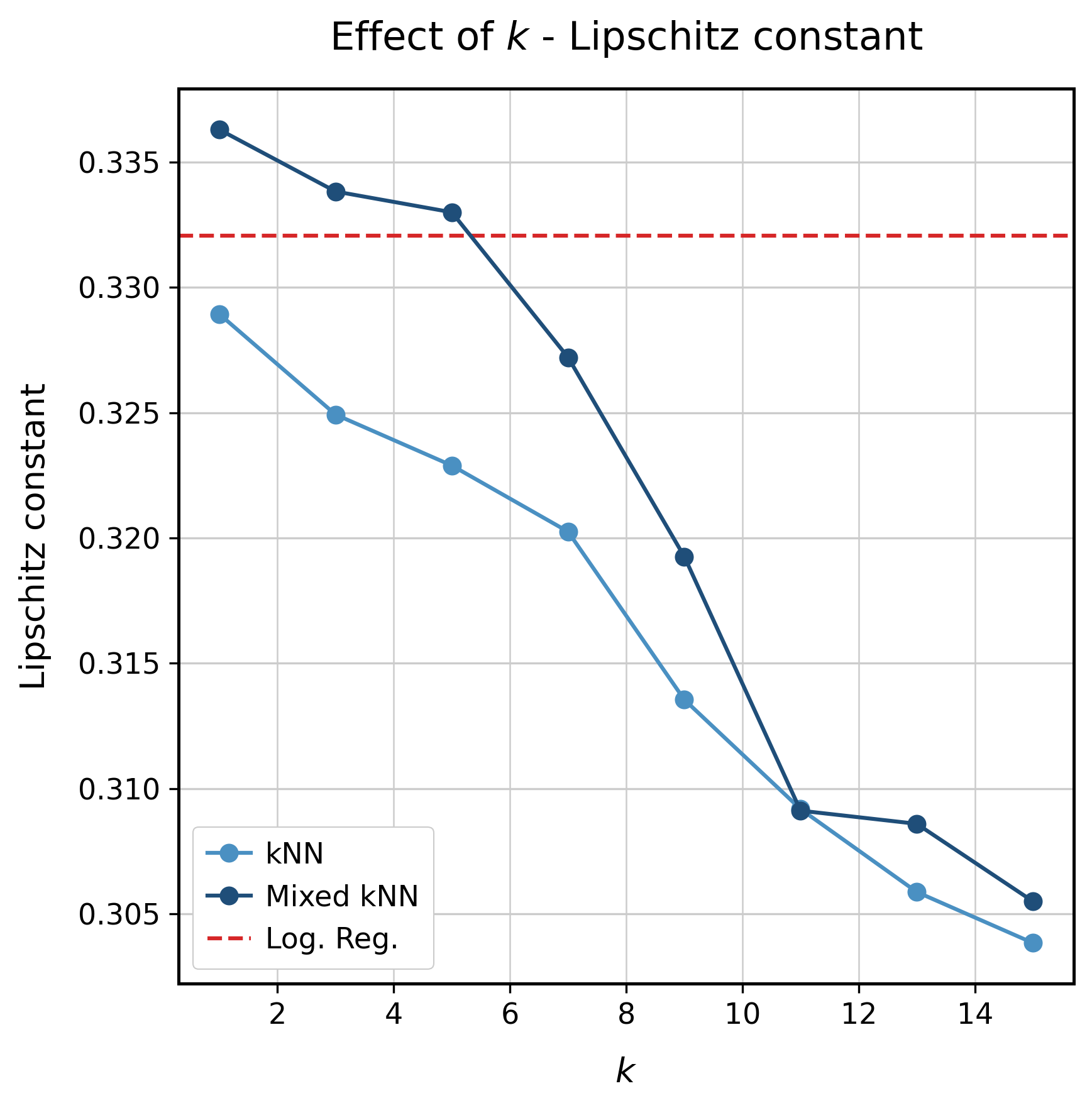}
    \caption{}
\end{figure}

\begin{figure}[!t]
  \centering
  \includegraphics[width=.49\textwidth]{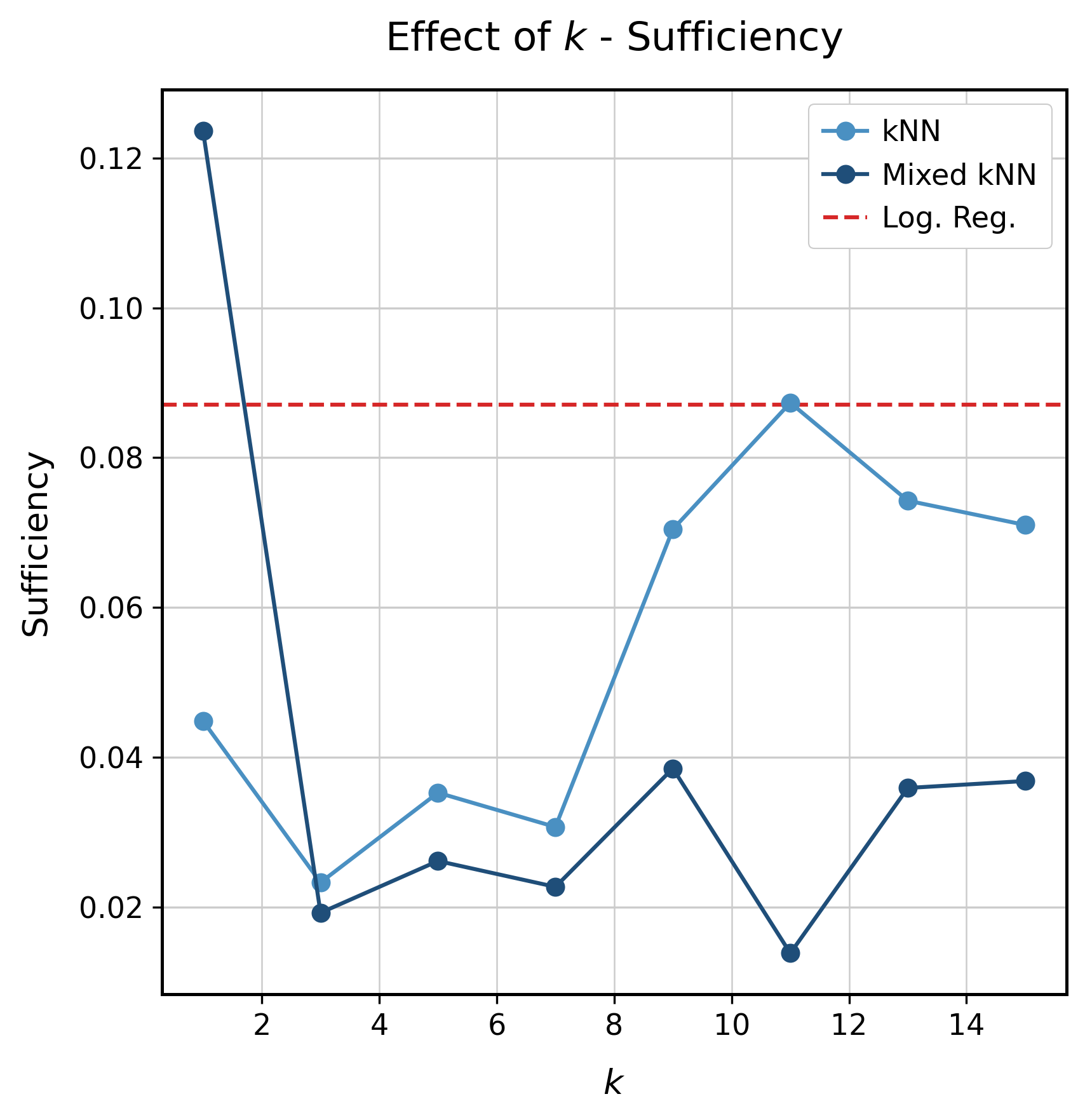}
  \includegraphics[width=.49\textwidth]{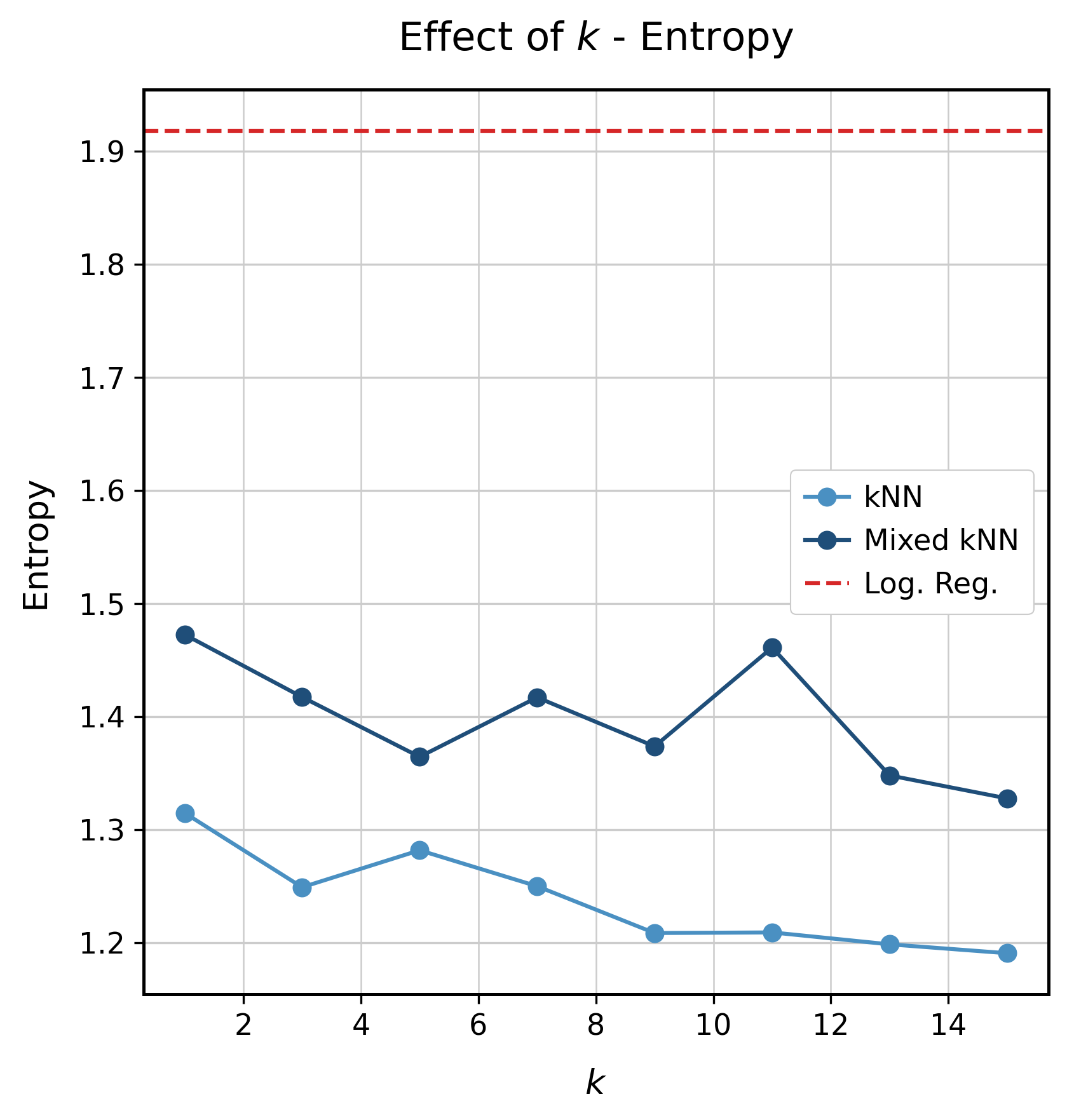}
    \caption{}
\end{figure}

\begin{figure}[!t]
  \centering
  \includegraphics[width=.49\textwidth]{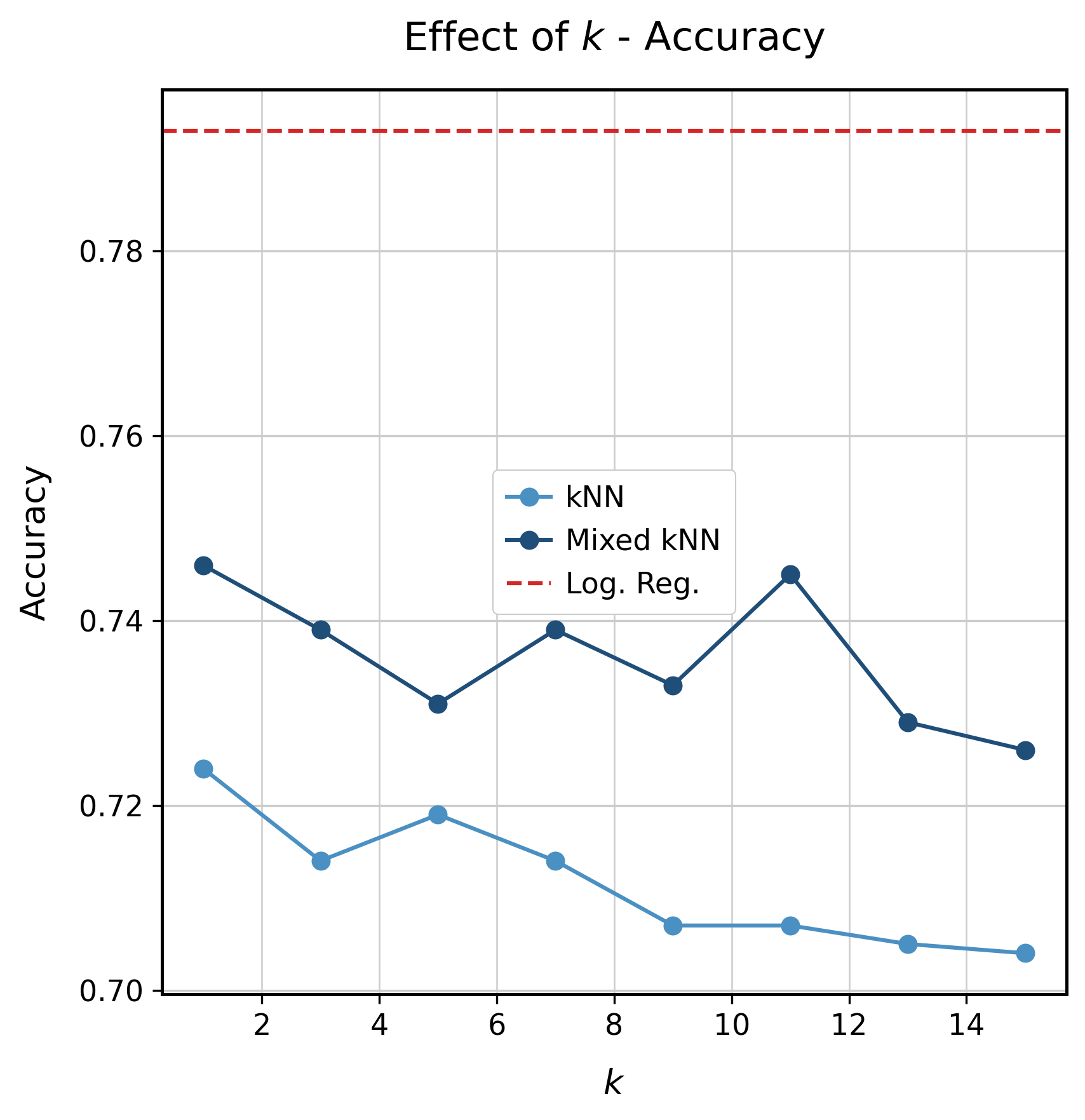}
    \caption{}
  \label{fig:LastK}
\end{figure}

\begin{figure}[!t]
  \centering
  \includegraphics[width=.49\textwidth]{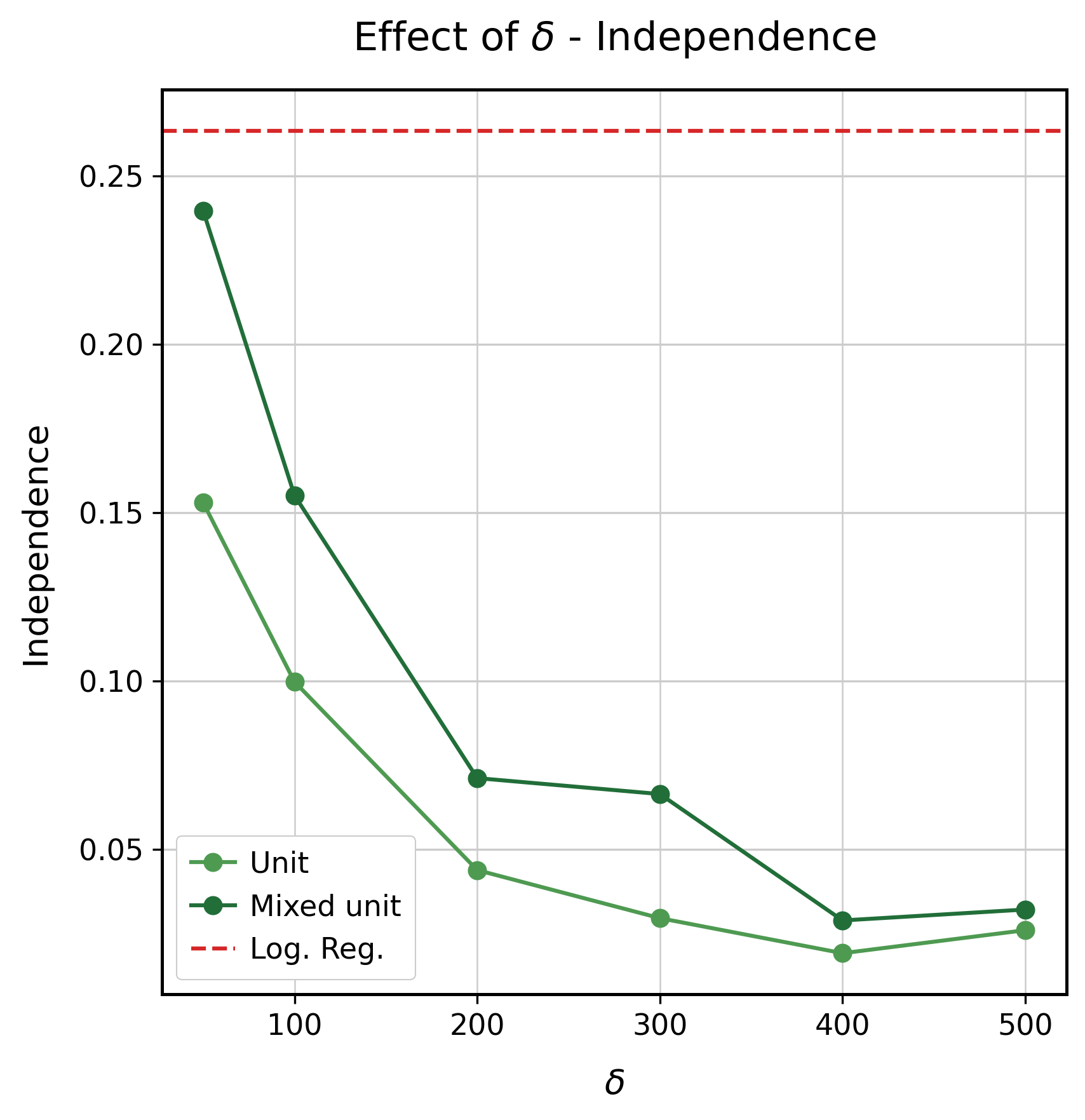}
  \includegraphics[width=.49\textwidth]{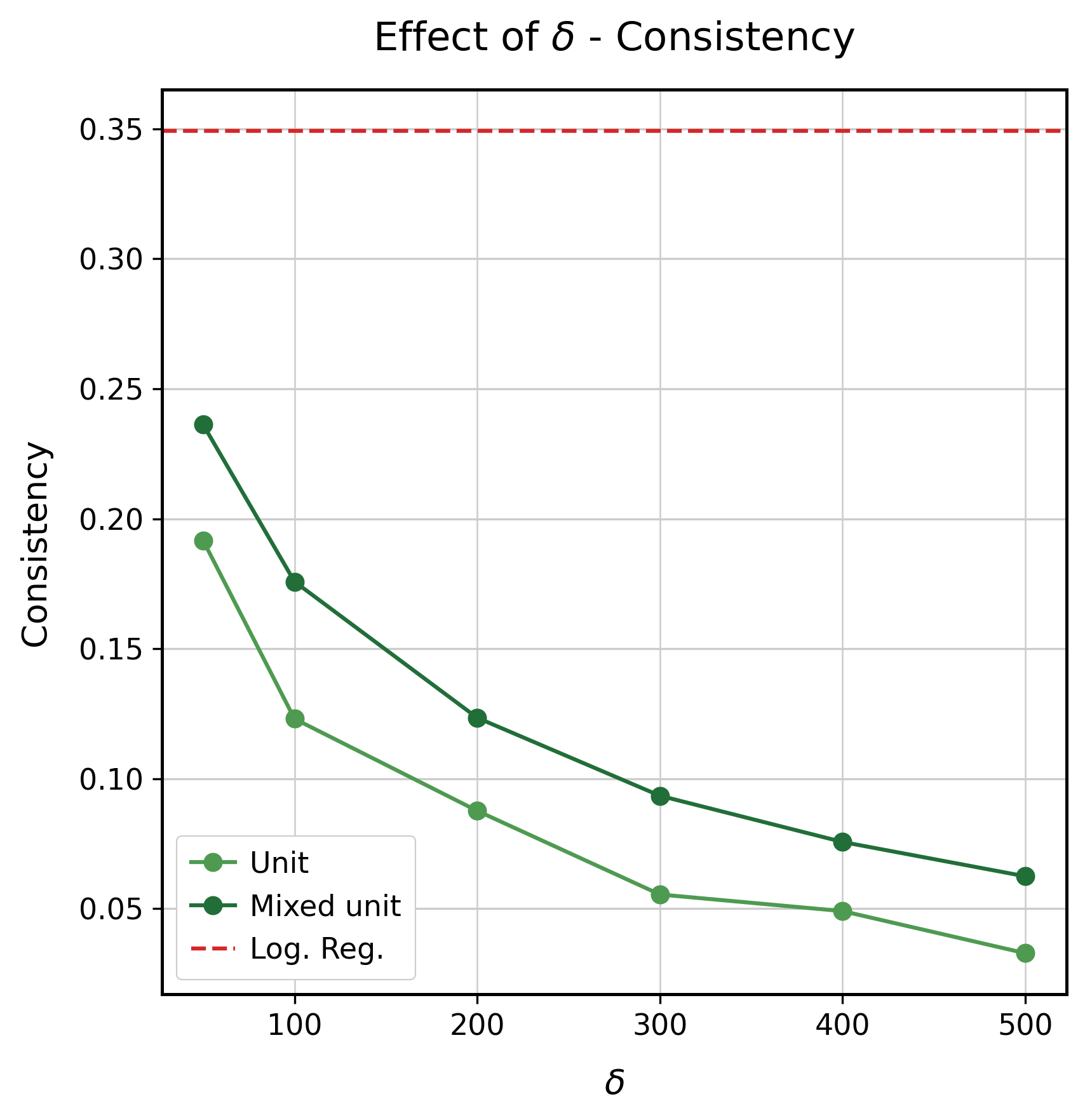}
    \caption{}
  \label{fig:FirstDelta}
\end{figure}

\begin{figure}[!t]
  \centering
  \includegraphics[width=.49\textwidth]{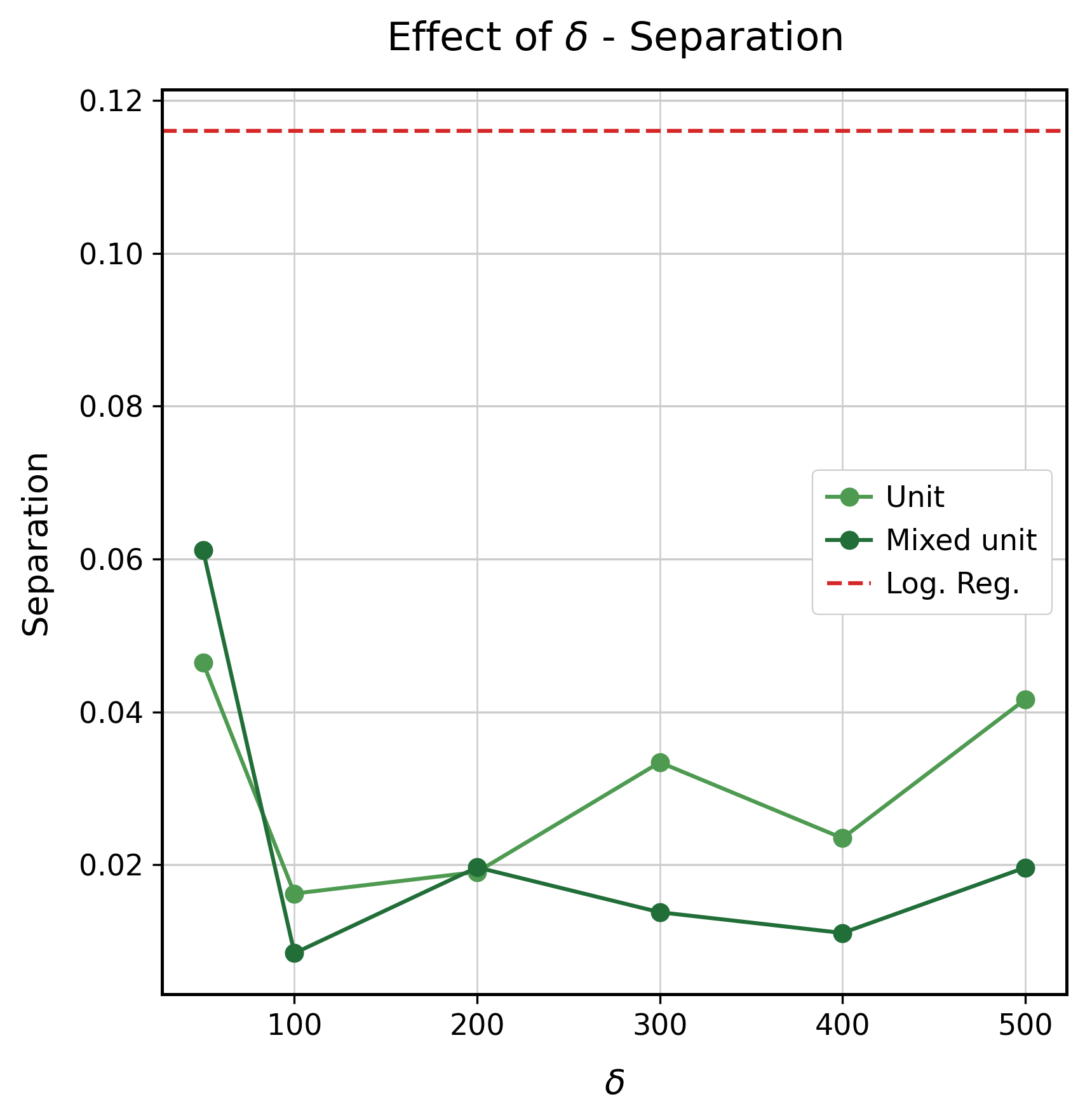}
  \includegraphics[width=.49\textwidth]{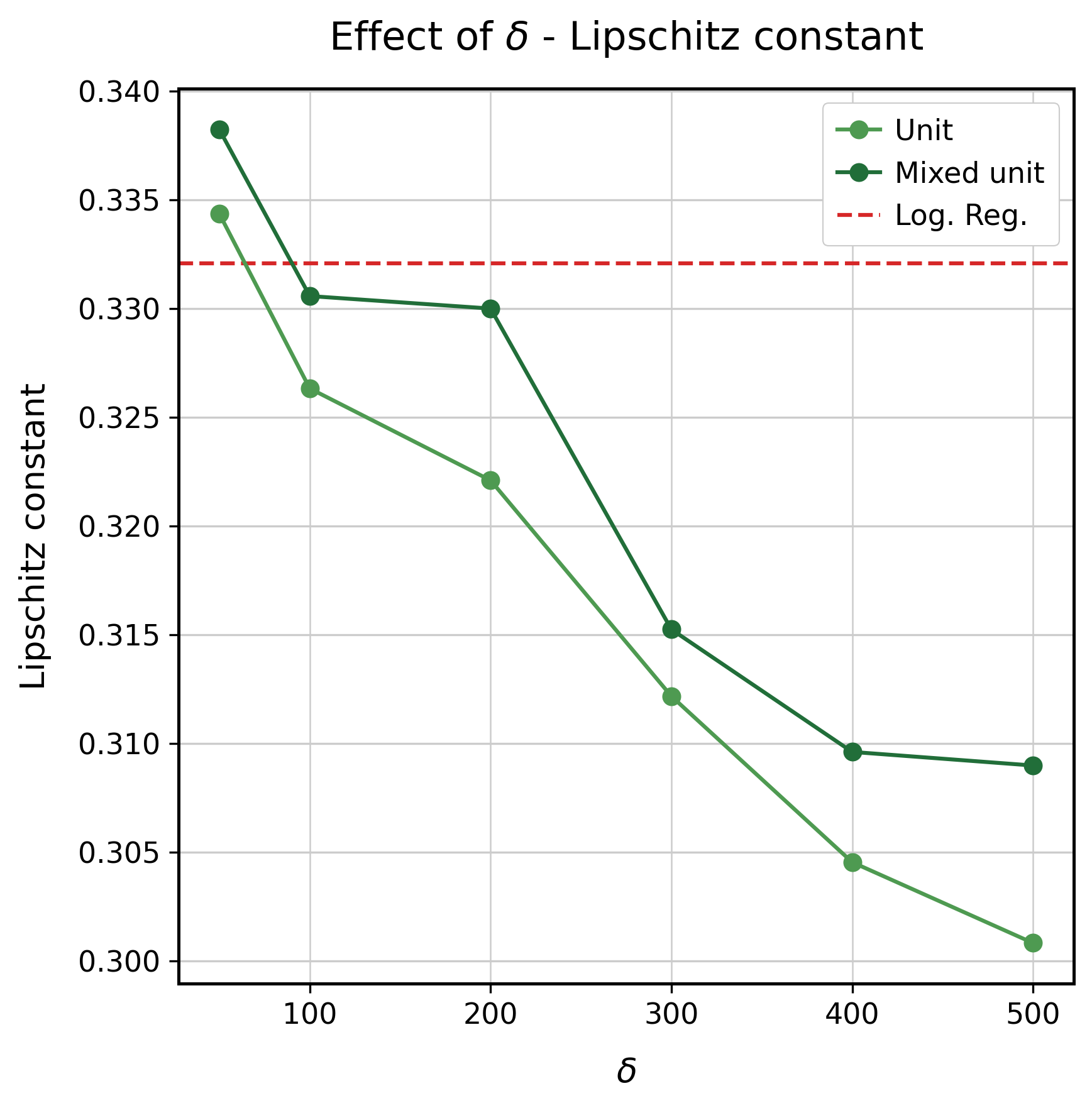}
    \caption{}
\end{figure}

\begin{figure}[!t]
  \centering
  \includegraphics[width=.49\textwidth]{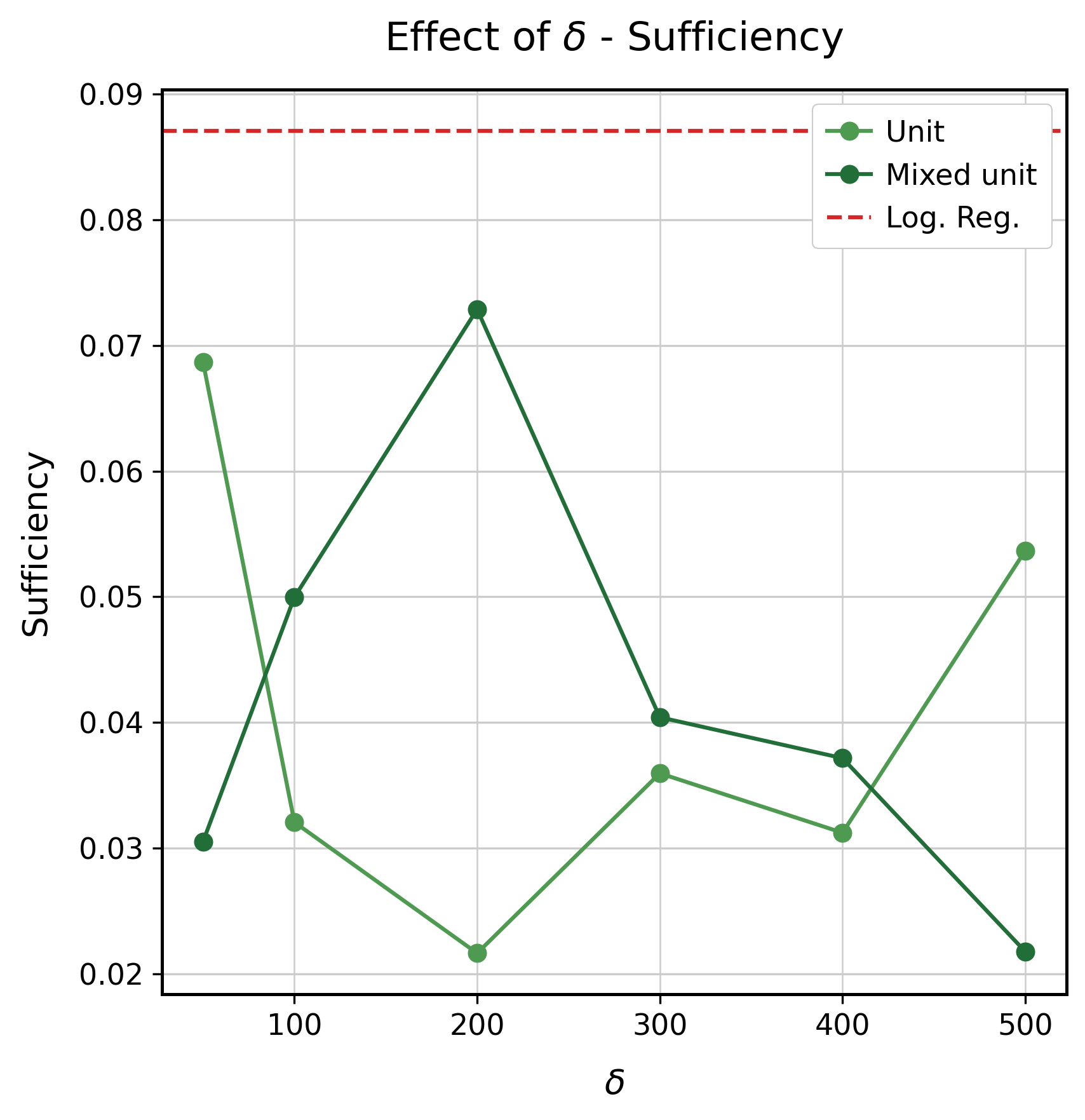}
  \includegraphics[width=.49\textwidth]{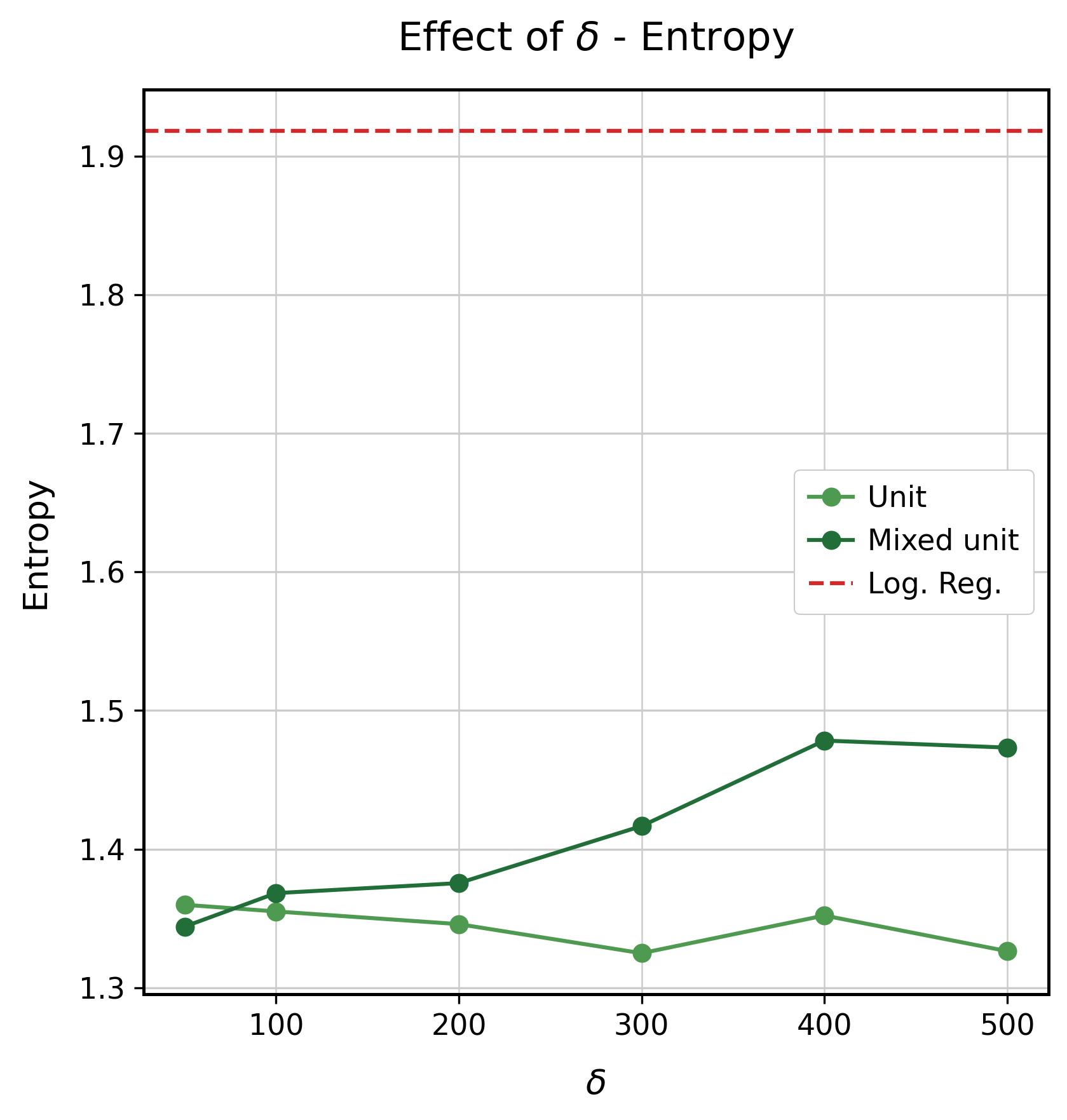}
    \caption{}
\end{figure}

\begin{figure}[!t]
  \centering
  \includegraphics[width=.49\textwidth]{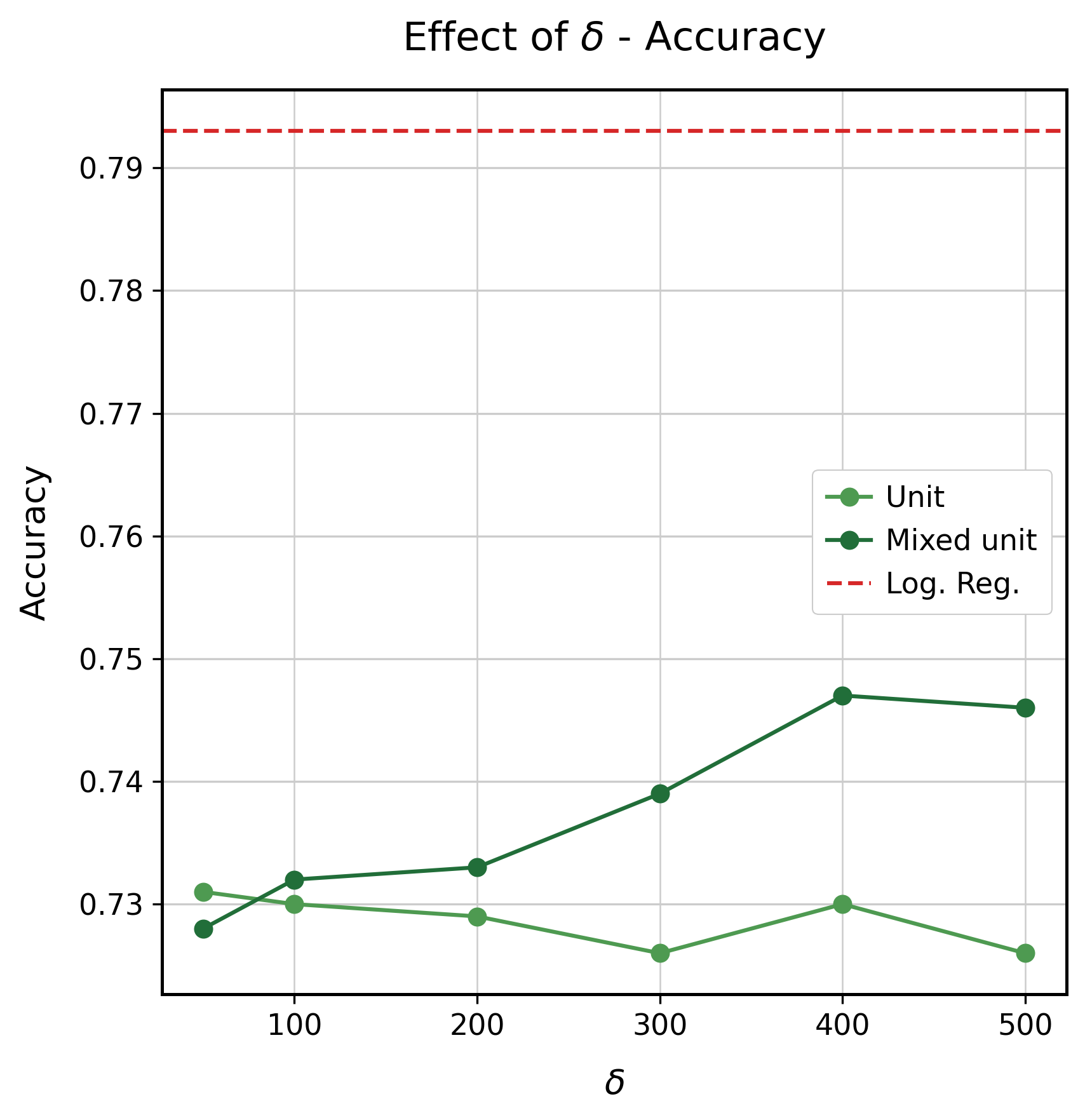}
    \caption{}
  \label{fig:LastDelta}
\end{figure}

Both topologies exhibit similar behavior with respect to the size of the neighbors ($k$ in the case of the kNN configuration and $\delta$ in the case of the unit ball graph), which is proof of the similarity between both topologies. Namely, it seems like accuracy and entropy remain stable while consistency and the lipschitz constant improve as the size of the neighbors increases, which underscores the role of the local topologies in individual fairness. On the other hand, although independence seems to also decrease with the size, the same cannot be said about separation and sufficiency, although this seems to be remedied by using mixed configurations. Importantly, it seems like using a local topology in combination with the global configuration smooths its behavior, improving its stability.

\subsubsection{Effect of convex weights on mixed topologies}
Finally, the last design choice is found in the relative weight given to each graph laplacian when using a mixed configuration. The only two objects of study are therefore the combination of the subset topology with either the kNN or unit ball graphs. Figures \ref{fig:FirstWeight} to \ref{fig:LastWeight} shows the results, which can be understood as a non-linear interpolation between one topology and the other.\\

\begin{figure}[!t]
  \centering
  \includegraphics[width=.49\textwidth]{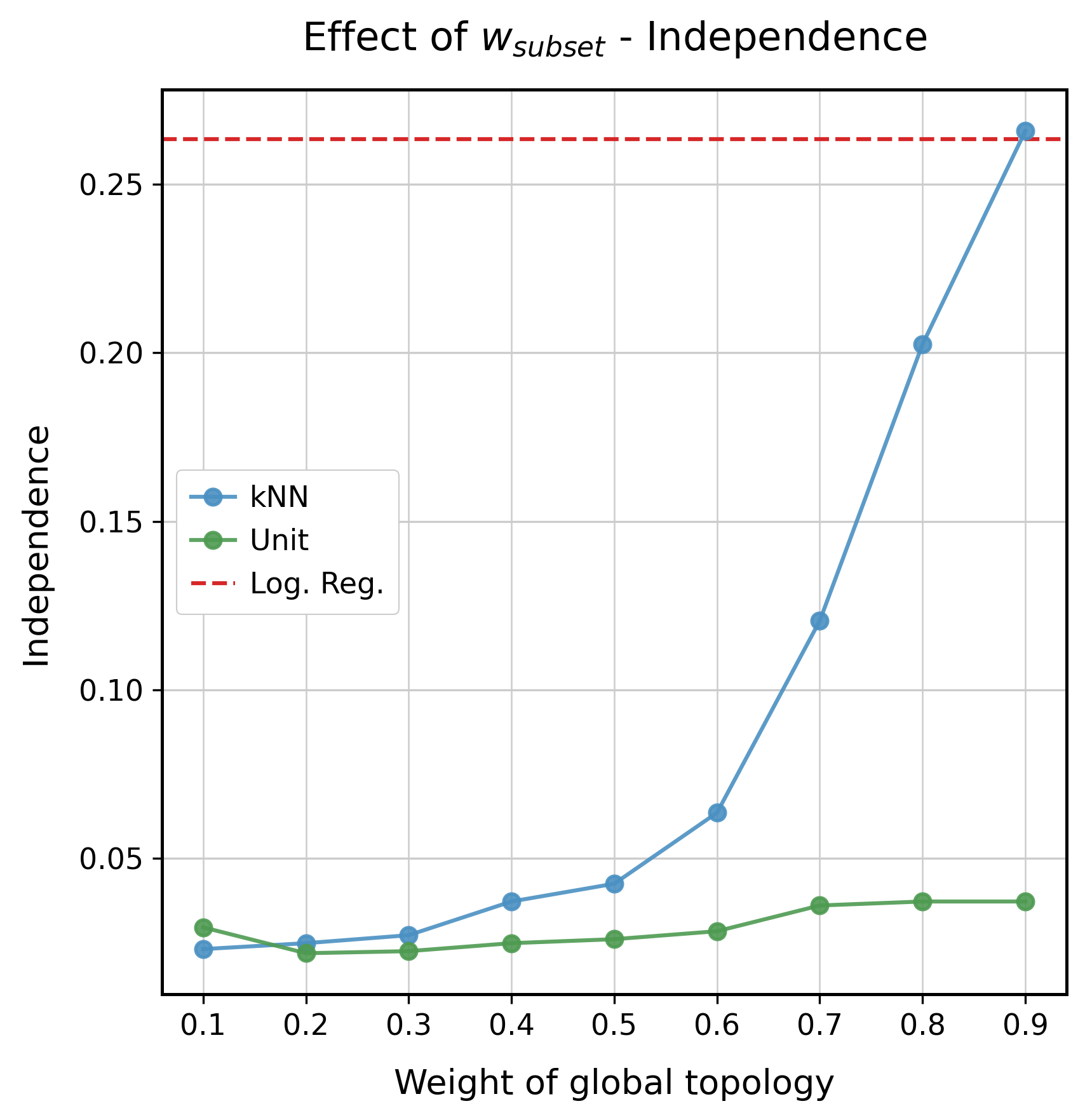}
  \includegraphics[width=.49\textwidth]{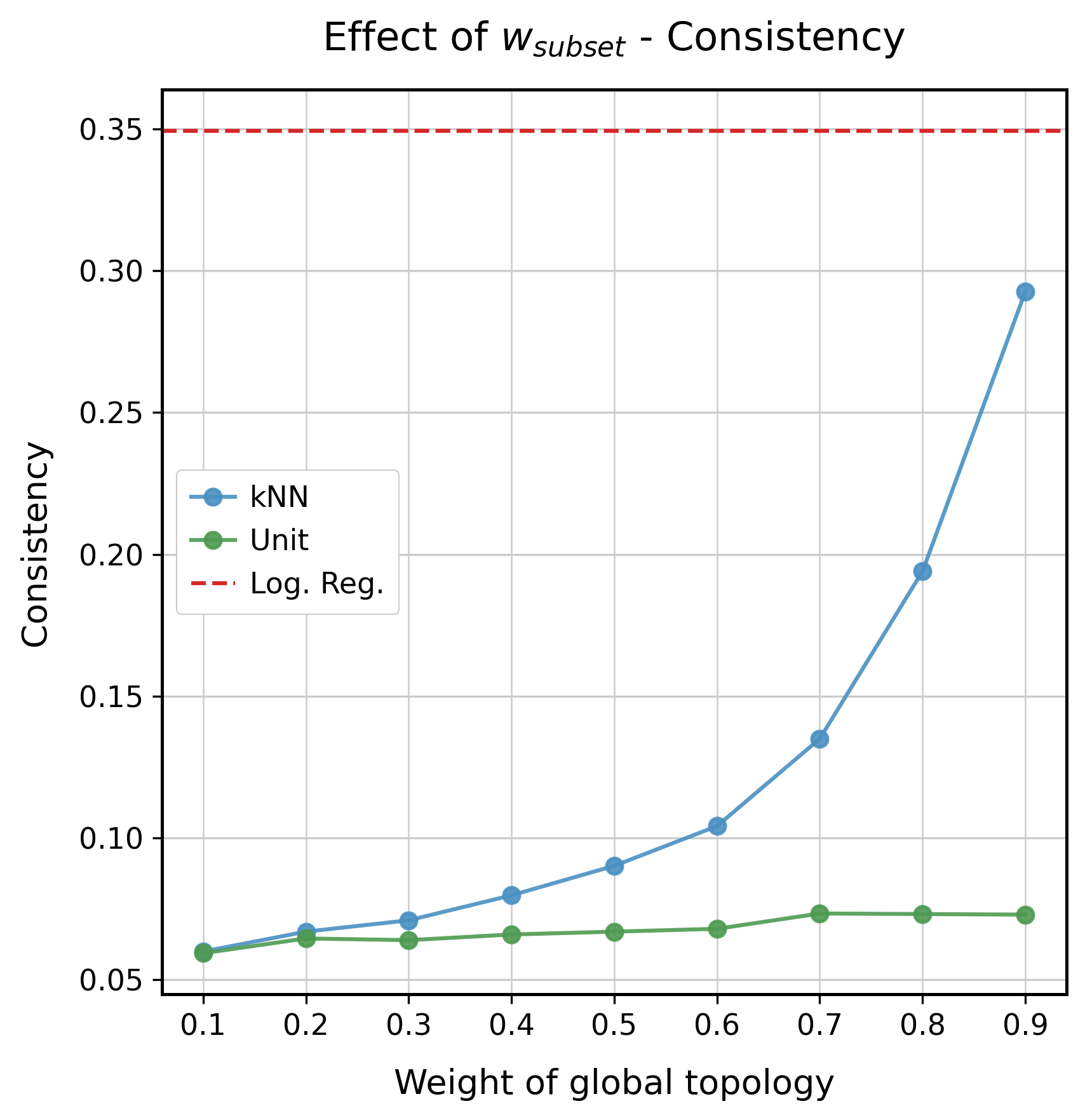}
    \caption{}
  \label{fig:FirstWeight}
\end{figure}

\begin{figure}[!t]
  \centering
  \includegraphics[width=.49\textwidth]{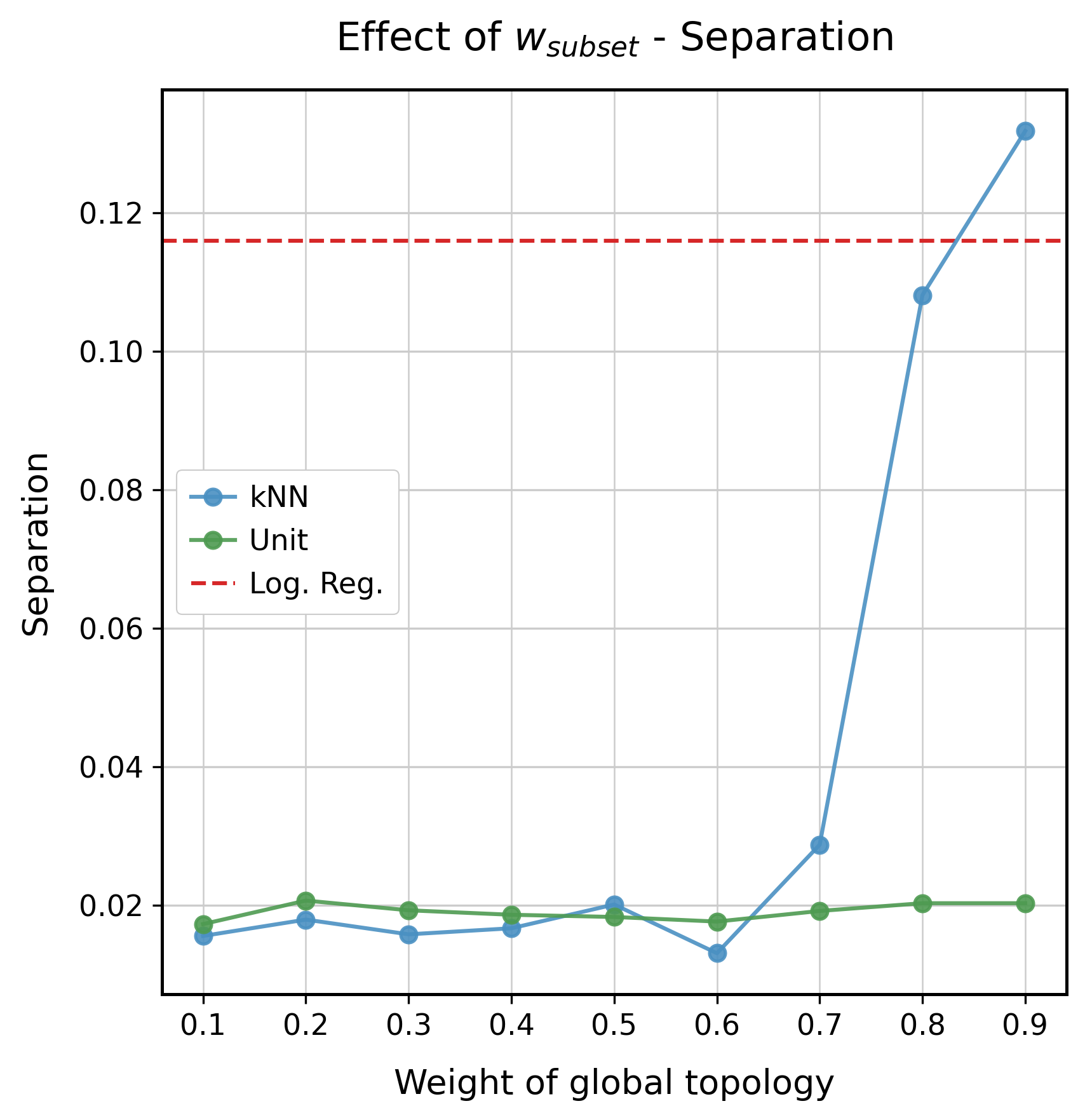}
  \includegraphics[width=.49\textwidth]{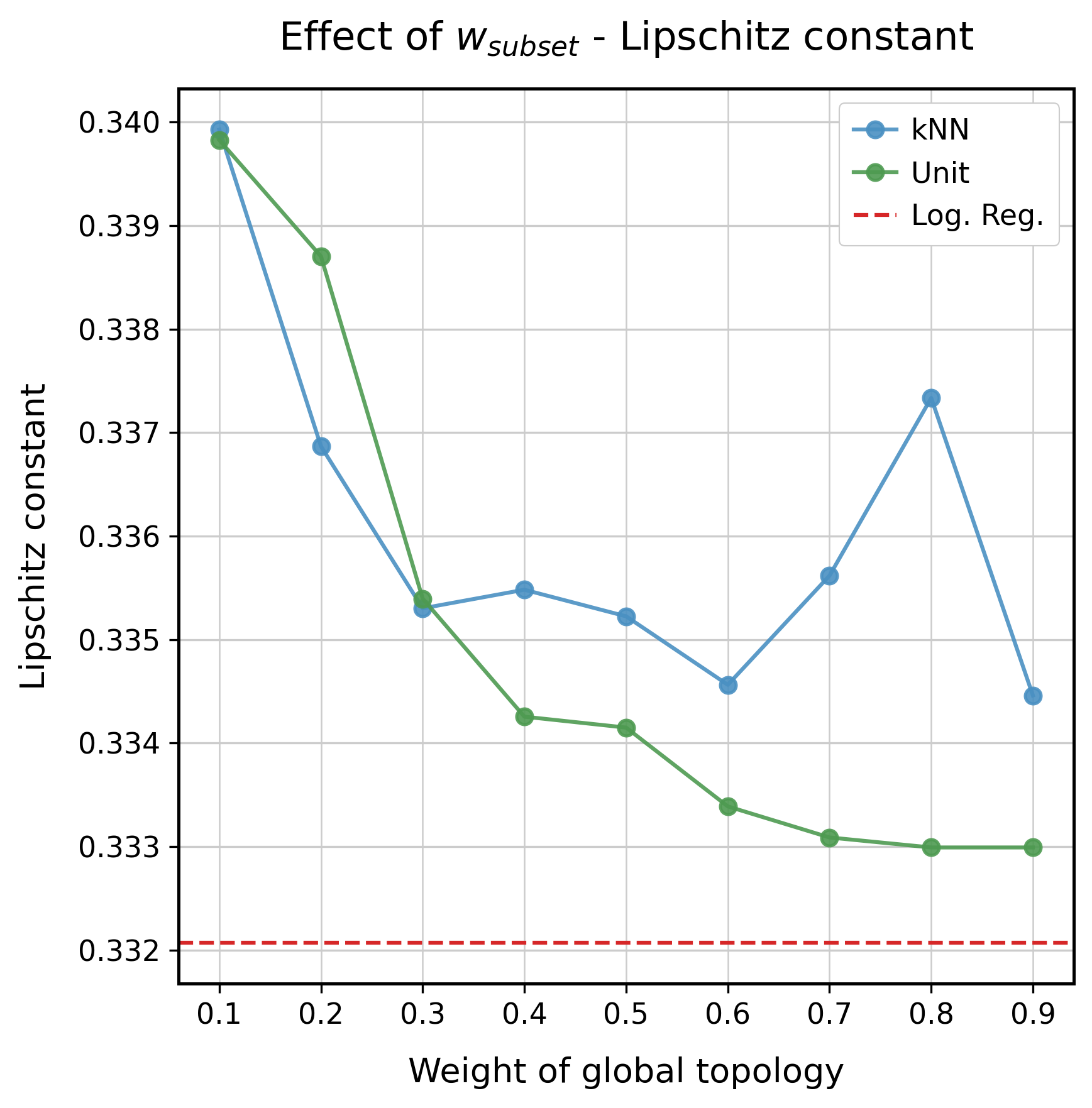}
    \caption{}
\end{figure}

\begin{figure}[!t]
  \centering
  \includegraphics[width=.49\textwidth]{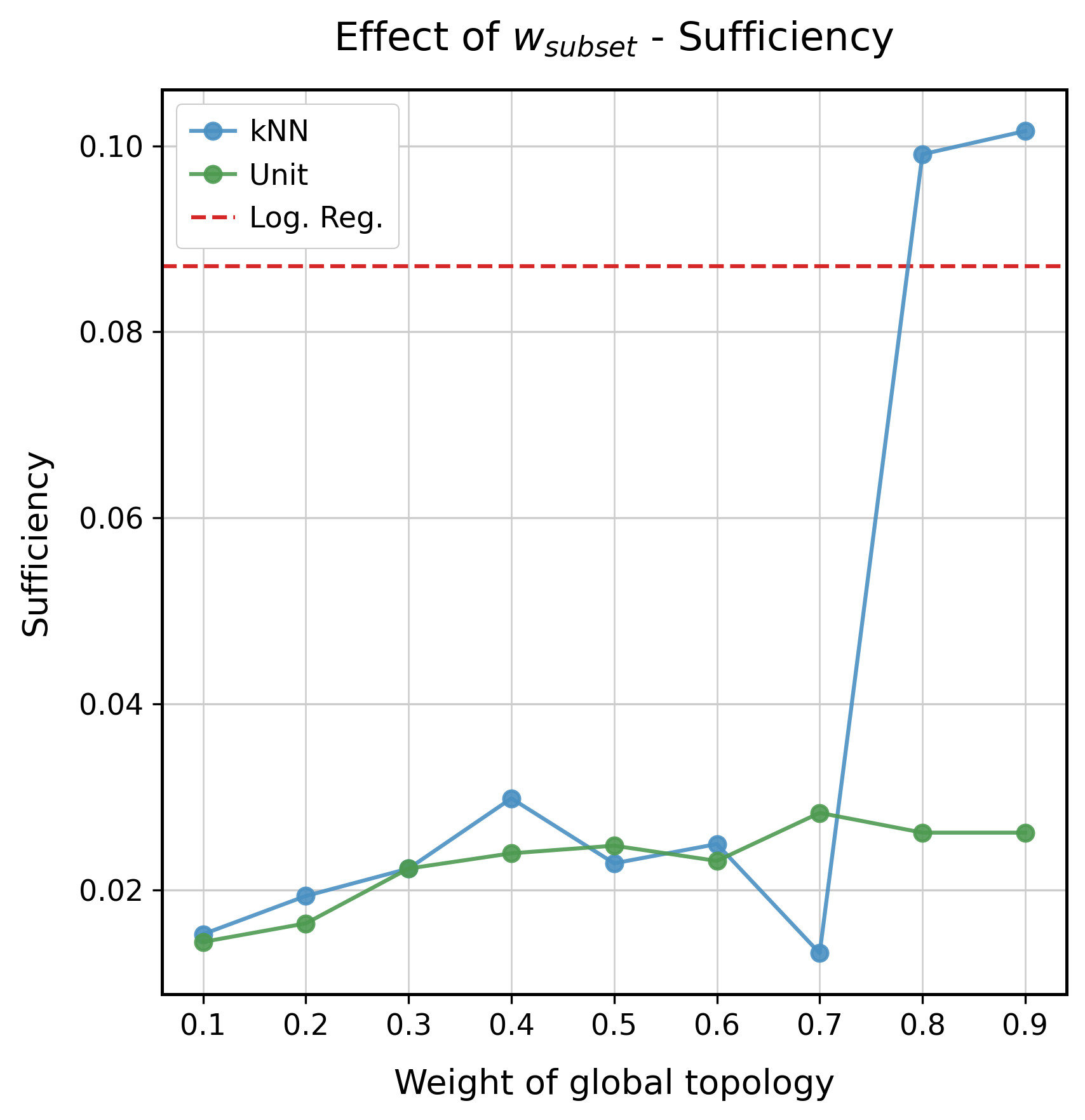}
  \includegraphics[width=.49\textwidth]{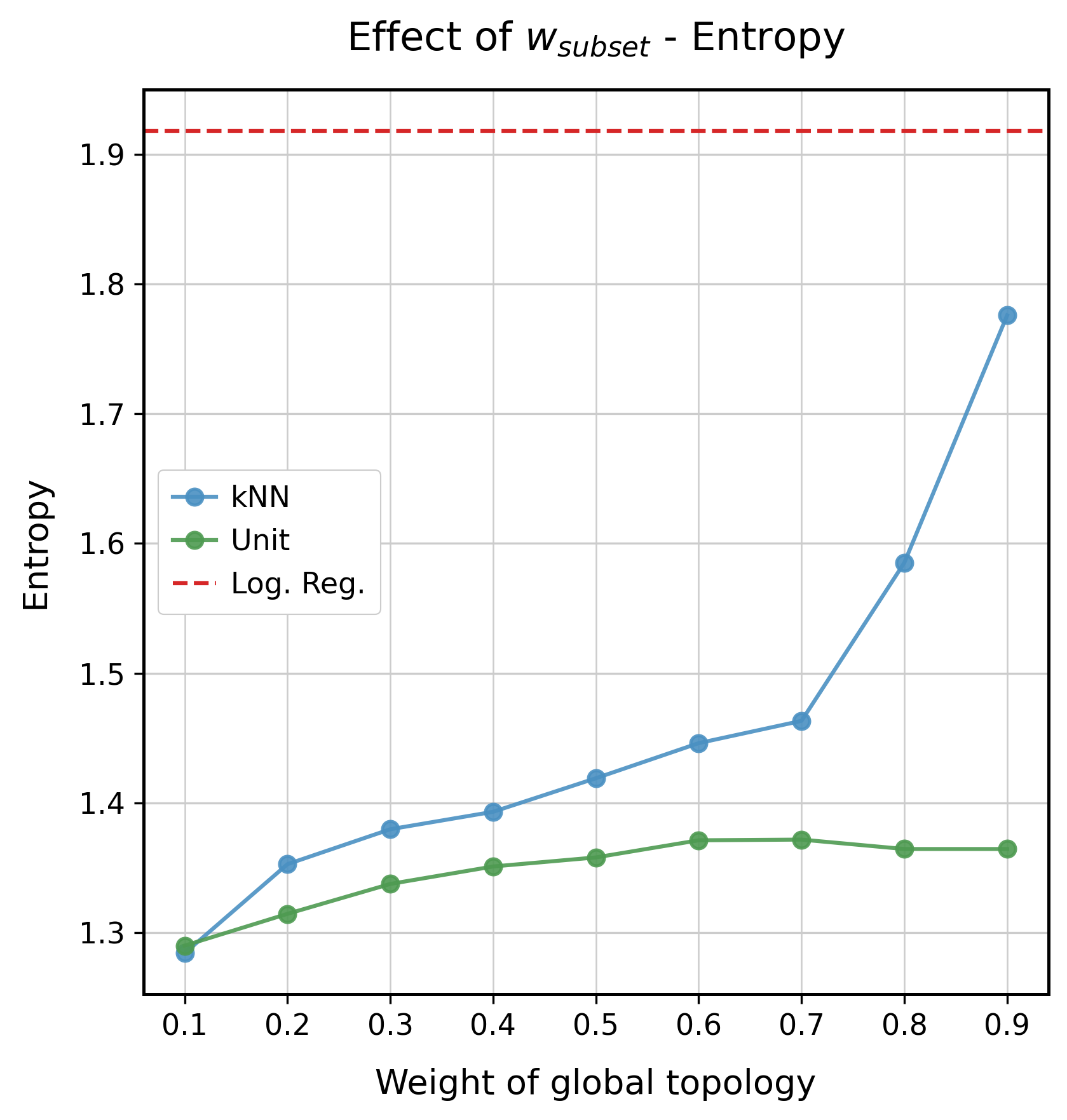}
    \caption{}
\end{figure}

\begin{figure}[!t]
  \centering
  \includegraphics[width=.49\textwidth]{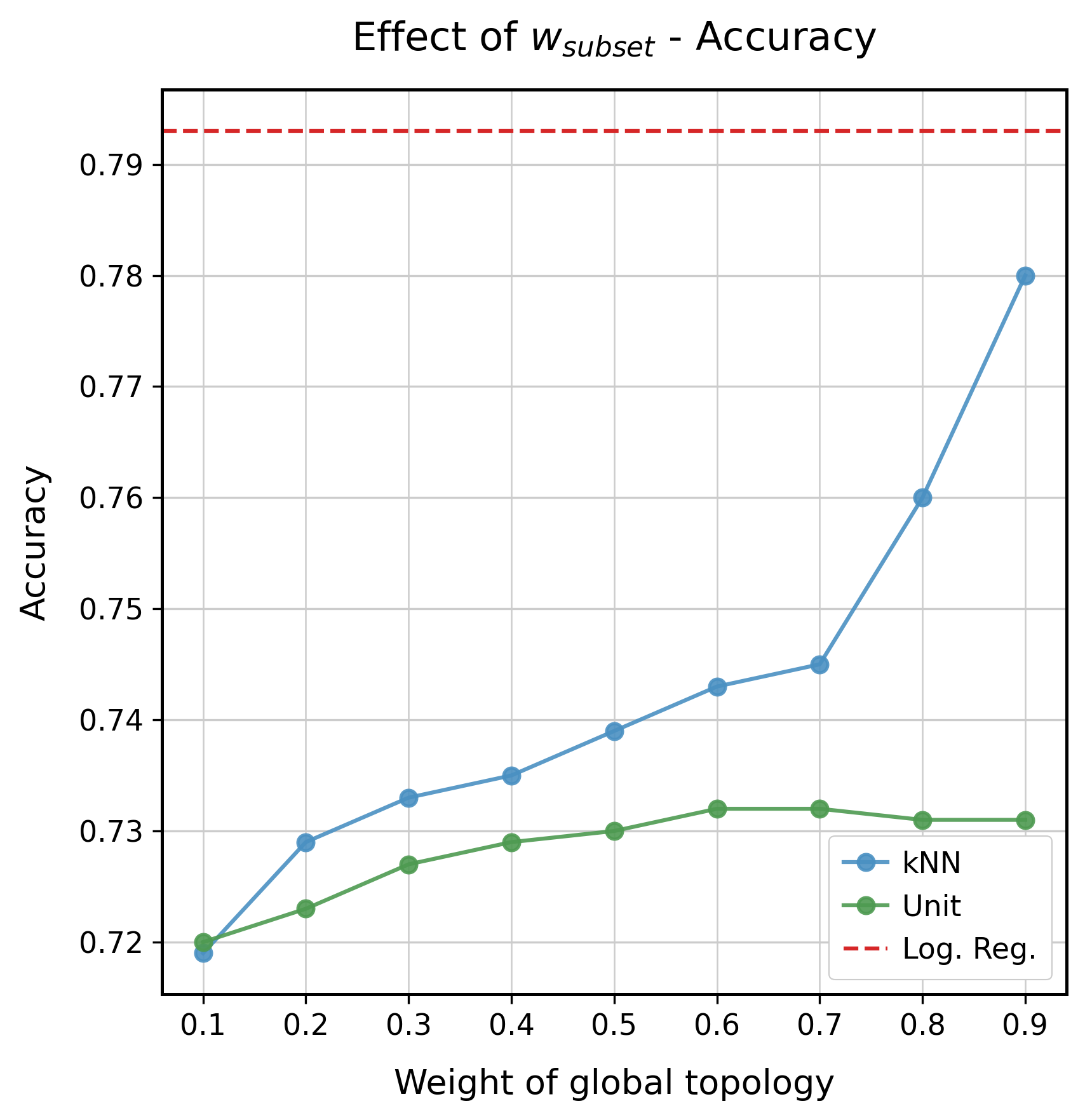}
    \caption{}
  \label{fig:LastWeight}
\end{figure}

The effect of this parameter on most metrics is non-trivial. On the one hand, accuracy, the lipschitz constant, consistency and entropy all seem to improve as the global topology is assigned more weight, although this behavior is more pronounced when using the kNN graph. The fact that a global configuration can improve individual metrics is a consequence of the more overarching regularization it imposes. Independence seems to increase in the mixed kNN configuration as the weight of the global topology gets higher, while it remains constant when using the unit ball topology. Sufficiency and separation, on the other hand, seem to stay stable until the weight reaches a certain threshold around $0.7$ for the kNN topology, although this might be a consequence of the aforementioned numerical instability.\\
All in all, it seems like the best results are obtained when giving equal weight to both topologies, where no topology dominates and results remain stable. Therefore, this will be our choice for the the next Sections. 

\subsection{The cost of fairness}
\label{subapp:cost}

We now delve into the Pareto frontiers hinted at in Section \ref{subsubsec:Cost}. Figures \ref{fig:Pareto2DGerman} to \ref{fig:Pareto2DAdult} show the Pareto fronts for independence and accuracy on the one hand, and for consistency and accuracy on the other hand. Starting with the German dataset, we can compromise accuracy by just $2\%$ while achieving an improvement of nearly $50\%$ in independence or $33\%$ in consistency. The situation in Compas is really different , where a high level of accuracy of more than $0.9$ can be achieved with less than $0.1$ independence or consitency. However, in order to achieve more fairness comparatively bigger sacrifices must be made. In order improve IND by $50\%$ we need to compromise accuracy by nearly $33\%$, while a similar improvement in consistency requires a decrease of $10\%$ in accuracy. Finally, the Adult dataset paints a more optimistic picture, where reducing independence or consistency by $50\%$ requires a decrease of just $2\%$ in accuracy.

\begin{figure}[!ht]
    \centering
    \includegraphics[width=0.49\textwidth]{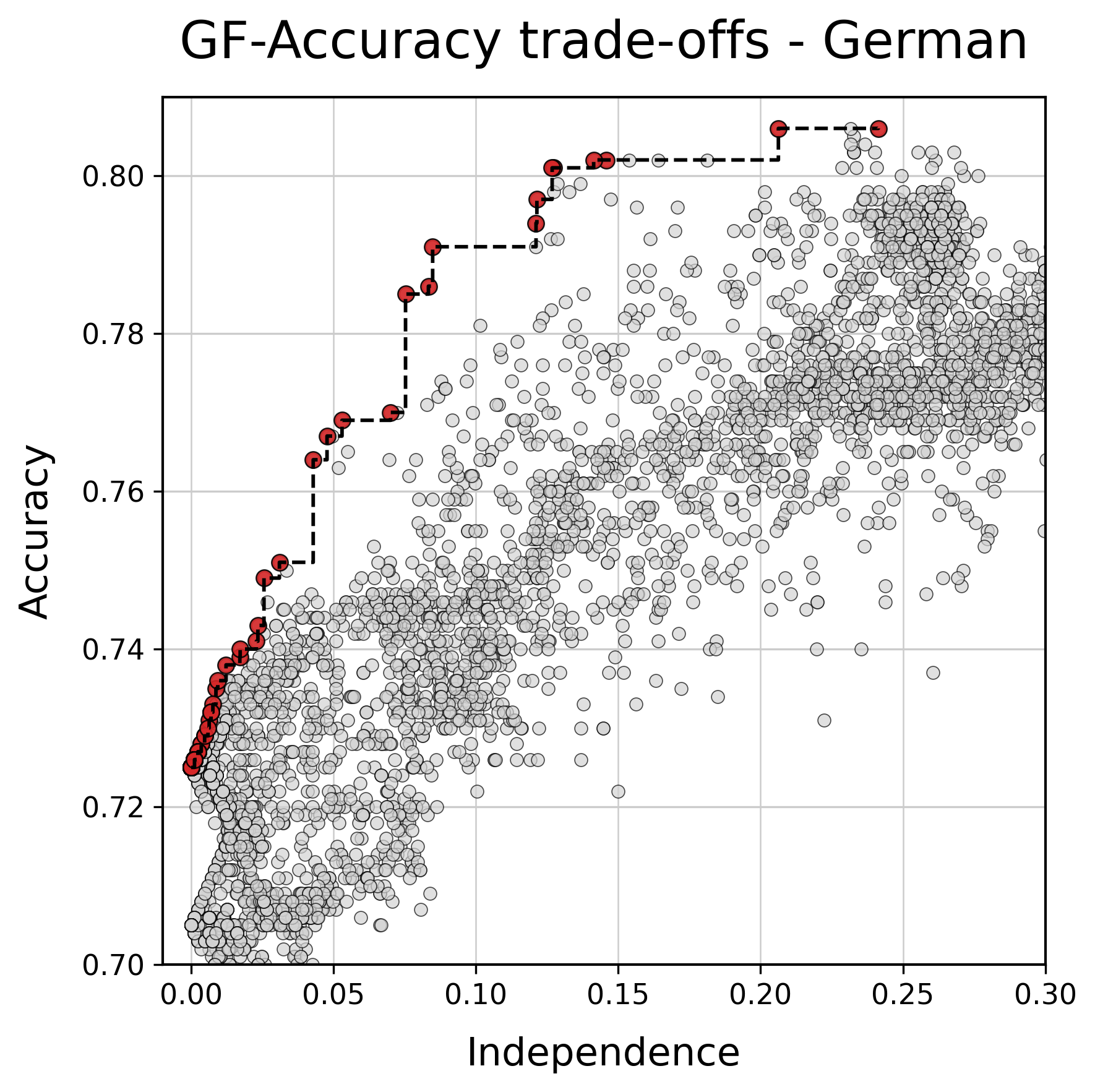}
    \includegraphics[width=0.49\textwidth]{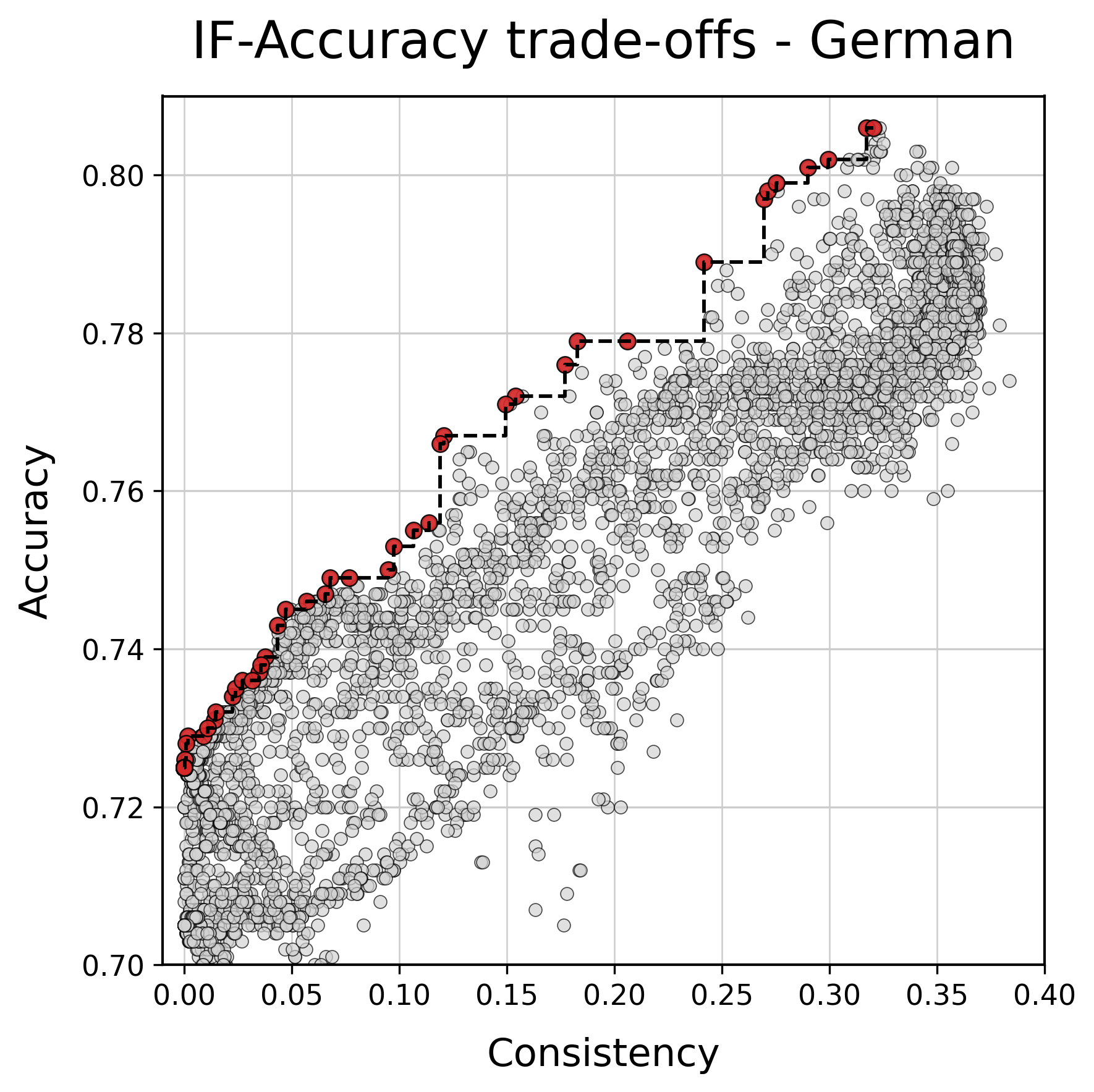}
        \caption{Fairness-accuracy trade-offs for the German dataset.}
    \label{fig:Pareto2DGerman}
\end{figure}

\begin{figure}[!ht]
    \centering
    \includegraphics[width=0.49\textwidth]{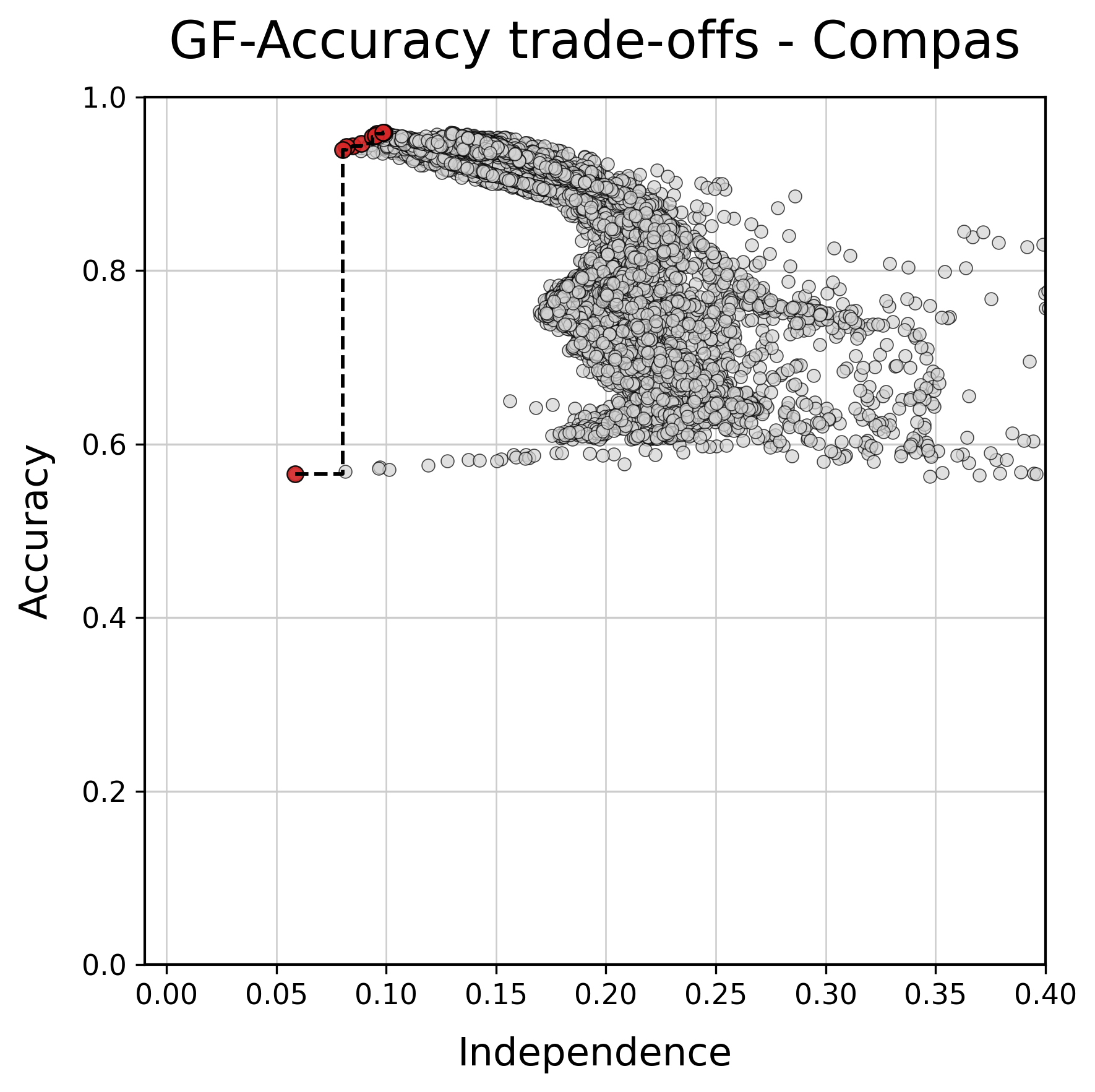}  
    \includegraphics[width=0.49\textwidth]{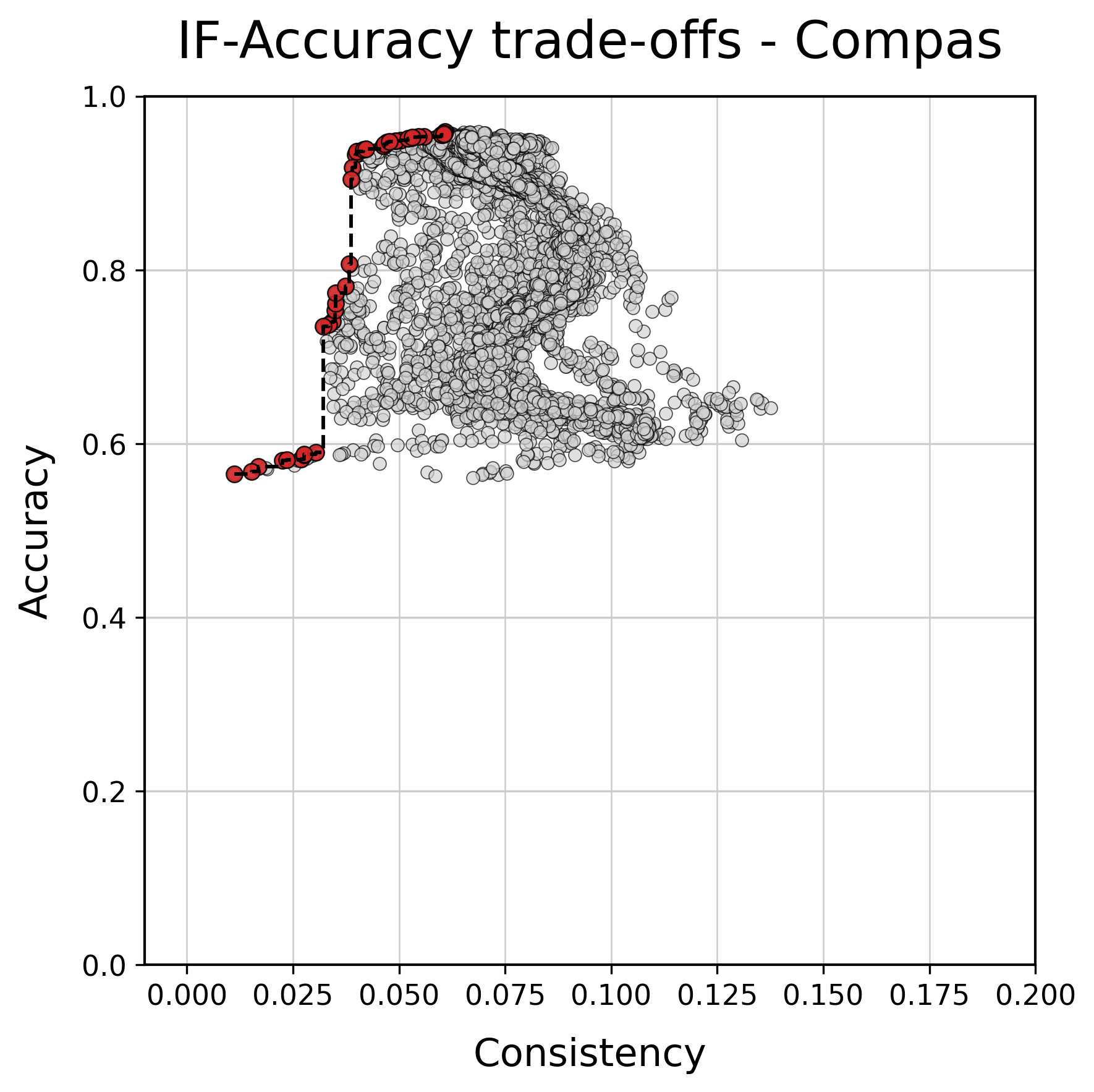}
    \caption{Fairness-accuracy trade-offs for the Compas dataset.}
    \label{fig:Pareto2DCompas}
\end{figure}

\begin{figure}[!ht]
    \centering
    \includegraphics[width=0.49\textwidth]{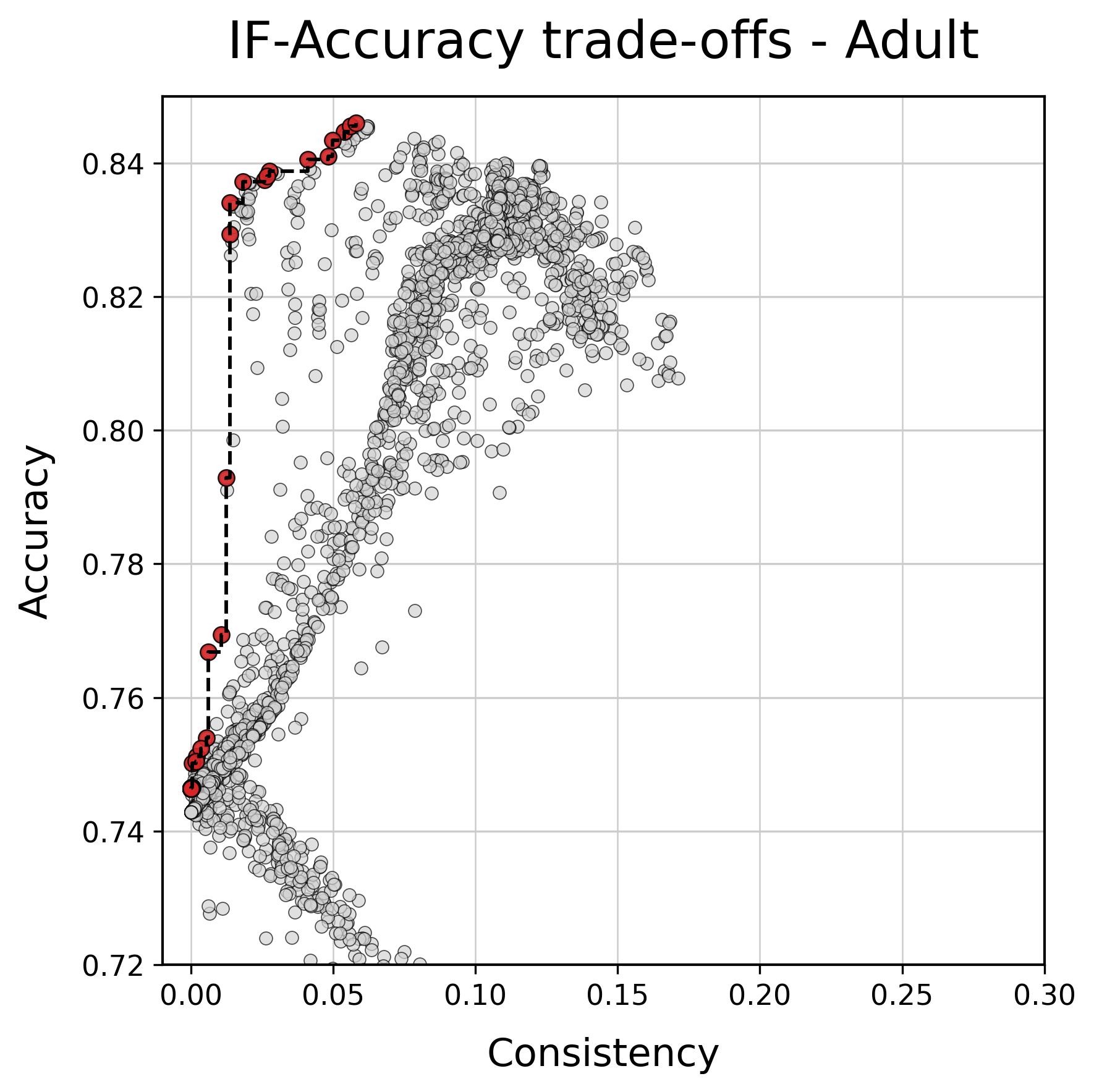}   
    \includegraphics[width=0.49\textwidth]{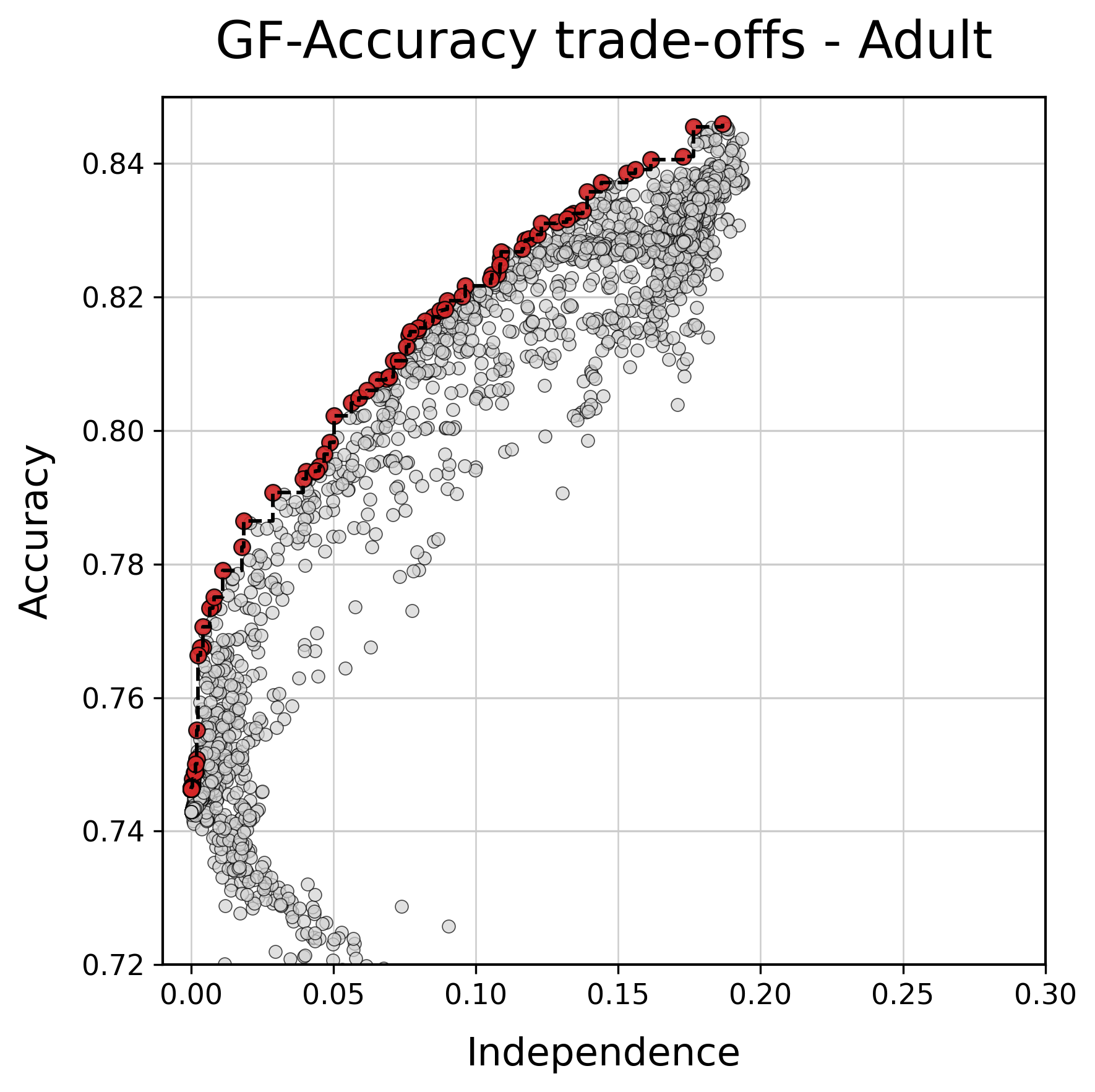}
    \caption{Fairness-accuracy trade-offs for the Adult dataset.}
    \label{fig:Pareto2DAdult}
\end{figure}

Next, the combined perspective of 3D Pareto frontiers is shown in Figure \ref{fig:Pareto3DApp}. Starting with the German dataset, a compromise of just $2\%$ in accuracy leads to a combined improvement of $33\%$ in both independence and consistency. Next, the Compas dataset shows a more grim reality in which small improvements in independence or consistency require step declines in accuracy of nearly $33\%$, although there is a set of solutions at the top of the surface which allow to greatly improve independence (up to $33\%$) while barely changing consistency and accuracy. Finally, the situation in the Adult dataset is very different, with a surface that seems aligned with the consistency axis. This suggests that it is possible to improve consistency while leaving independence and accuracy constant. However, in order to improve independence it is necessary to compromise accuracy while leaving consistency constant. For example, a sacrifice of $4\%$ in accuracy and $1\%$ in consistency allows to improve independence by $33\%$ This analysis sheds light into the exotic situations one might encounter when trying to simultaneously balance performance with IF and GF.

\begin{figure}[!t]
    \centering
    \includegraphics[width=0.49\textwidth]{germanPareto3D.png}
    \includegraphics[width=0.49\textwidth]{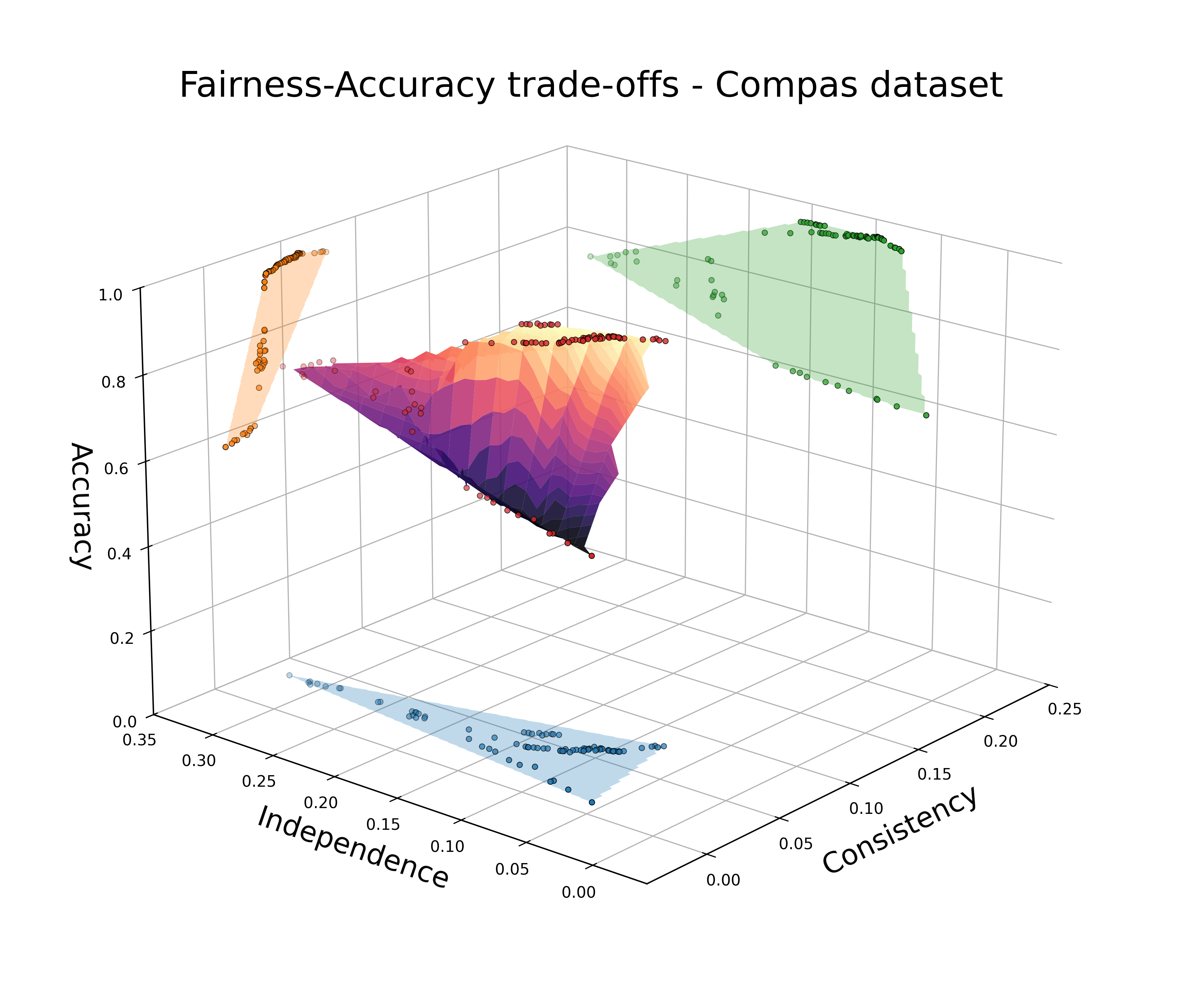}
    \includegraphics[width=0.49\textwidth]{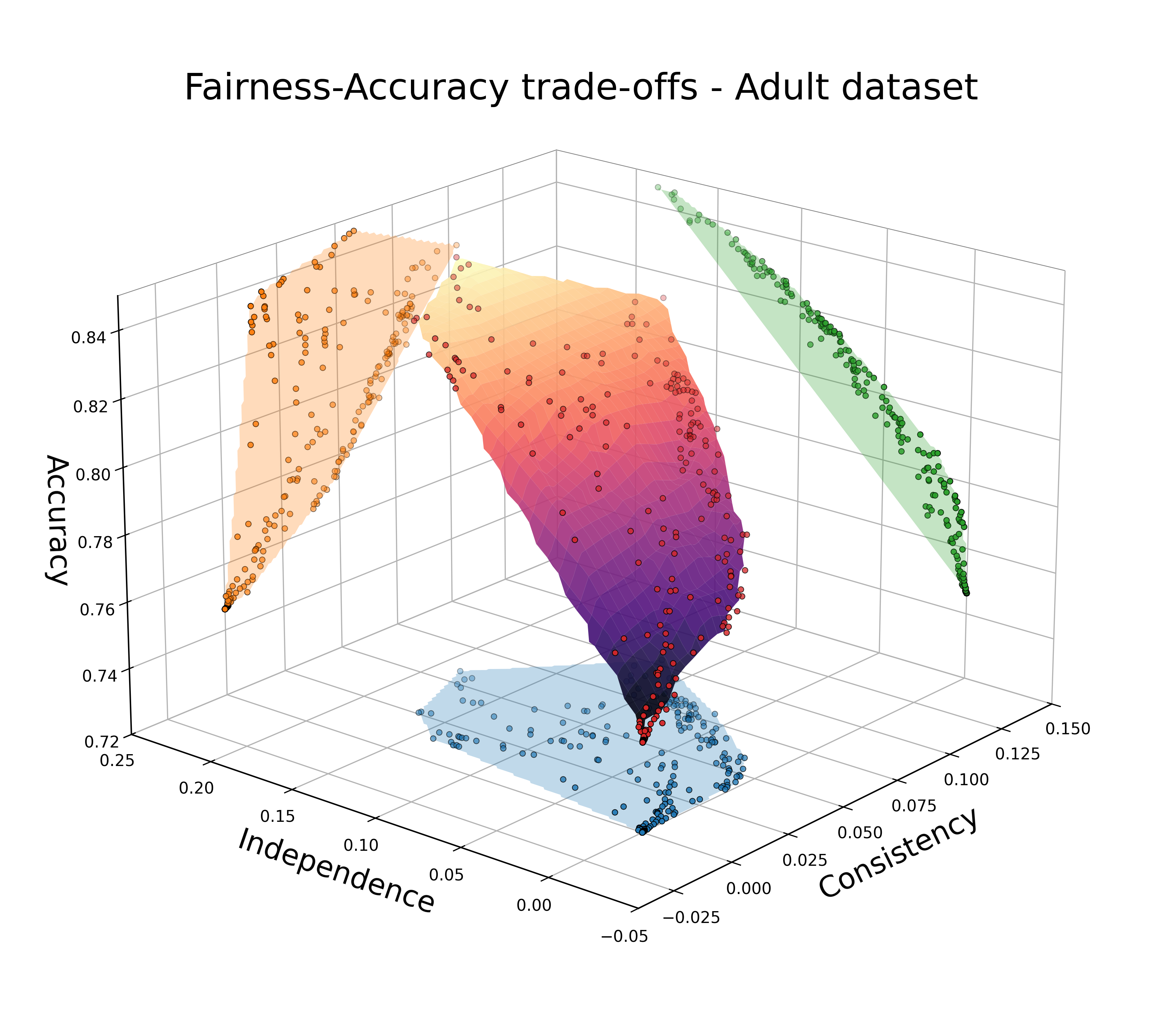}  
    \caption{Fairness-accuracy trade-offs in the Pareto frontier accounting for both independence and consistency at the same time.}
    \label{fig:Pareto3DApp}
\end{figure}

\subsection{Results of the grid search}
\label{subapp:results}

Results for the grid search on all three datasets are reported in Tables \ref{tab:GermanResults} to \ref{tab:AdultResults}. In general, the overarching analysis still holds, with logistic regression providing the most accurate model in general at the cost of group and individual fairness. The proposed topologies compromise performance in order to achieve better fairness metrics, with the kNN topology providing high improvements in said metrics with barely any cost in accuracy.

\begin{table}[htbp]
\centering
\resizebox{\linewidth}{!}{
\begin{tabular}{lccccccc}
\toprule
Model & ACC (\%) & IND (\%) & SUF (\%) & SEP (\%) & CON (\%) & LIP & ENT \\
\midrule
Log. reg. & \textbf{{\footnotesize 79.30}{\scriptsize $\pm$0.68}} & {\footnotesize 26.34}{\scriptsize $\pm$2.91} & {\footnotesize 8.71}{\scriptsize $\pm$6.53} & {\footnotesize 11.60}{\scriptsize $\pm$4.43} & {\footnotesize 34.92}{\scriptsize $\pm$1.64} & {\footnotesize 0.33}{\scriptsize $\pm$0.00} & {\footnotesize 1.92}{\scriptsize $\pm$0.08} \\
\hdashline
kNN & {\footnotesize 77.60}{\scriptsize $\pm$1.32} & \textbf{{\footnotesize 13.13}{\scriptsize $\pm$3.80}} & {\footnotesize 5.31}{\scriptsize $\pm$4.09} & \textbf{{\footnotesize 3.56}{\scriptsize $\pm$2.38}} & \textbf{{\footnotesize 30.44}{\scriptsize $\pm$1.85}} & \textbf{{\footnotesize 0.33}{\scriptsize $\pm$0.00}} & {\footnotesize 1.74}{\scriptsize $\pm$0.14} \\
Mixed kNN & {\footnotesize 77.70}{\scriptsize $\pm$1.03} & {\footnotesize 27.23}{\scriptsize $\pm$4.06} & {\footnotesize 8.04}{\scriptsize $\pm$3.84} & {\footnotesize 11.65}{\scriptsize $\pm$1.44} & {\footnotesize 32.22}{\scriptsize $\pm$2.73} & {\footnotesize 0.33}{\scriptsize $\pm$0.00} & {\footnotesize 1.75}{\scriptsize $\pm$0.11} \\
Unit & {\footnotesize 77.90}{\scriptsize $\pm$0.97} & {\footnotesize 20.57}{\scriptsize $\pm$3.15} & \textbf{{\footnotesize 4.81}{\scriptsize $\pm$3.33}} & {\footnotesize 4.46}{\scriptsize $\pm$3.30} & {\footnotesize 34.14}{\scriptsize $\pm$1.16} & {\footnotesize 0.33}{\scriptsize $\pm$0.00} & {\footnotesize 1.77}{\scriptsize $\pm$0.10} \\
Mixed unit & {\footnotesize 79.20}{\scriptsize $\pm$0.98} & {\footnotesize 32.69}{\scriptsize $\pm$3.95} & {\footnotesize 8.22}{\scriptsize $\pm$5.52} & {\footnotesize 14.14}{\scriptsize $\pm$4.28} & {\footnotesize 34.38}{\scriptsize $\pm$2.99} & {\footnotesize 0.33}{\scriptsize $\pm$0.00} & {\footnotesize 1.91}{\scriptsize $\pm$0.12} \\
Subset & {\footnotesize 76.60}{\scriptsize $\pm$0.86} & {\footnotesize 15.36}{\scriptsize $\pm$2.79} & {\footnotesize 10.52}{\scriptsize $\pm$3.65} & {\footnotesize 8.79}{\scriptsize $\pm$3.74} & {\footnotesize 33.24}{\scriptsize $\pm$7.18} & {\footnotesize 0.33}{\scriptsize $\pm$0.01} & \textbf{{\footnotesize 1.64}{\scriptsize $\pm$0.08}} \\
\bottomrule
\end{tabular}
}
\caption{Results for the german dataset. The best result for each metric is highlighted in black.}
\label{tab:GermanResults}
\end{table}

\begin{table}[htbp]
\centering
\resizebox{\linewidth}{!}{
\begin{tabular}{lccccccc}
\toprule
Model & ACC (\%) & IND (\%) & SUF (\%) & SEP (\%) & CON (\%) & LIP & ENT \\
\midrule
Log. reg. & {\footnotesize 95.54}{\scriptsize $\pm$0.21} & {\footnotesize 9.99}{\scriptsize $\pm$0.67} & {\footnotesize 0.76}{\scriptsize $\pm$0.26} & {\footnotesize 2.54}{\scriptsize $\pm$0.24} & {\footnotesize 6.18}{\scriptsize $\pm$0.15} & {\footnotesize 0.41}{\scriptsize $\pm$0.00} & {\footnotesize 10.74}{\scriptsize $\pm$0.54} \\
\hdashline
kNN & {\footnotesize 95.02}{\scriptsize $\pm$0.34} & {\footnotesize 10.57}{\scriptsize $\pm$0.68} & \textbf{{\footnotesize 0.66}{\scriptsize $\pm$0.61}} & {\footnotesize 1.90}{\scriptsize $\pm$0.58} & \textbf{{\footnotesize 5.99}{\scriptsize $\pm$0.20}} & \textbf{{\footnotesize 0.41}{\scriptsize $\pm$0.00}} & \textbf{{\footnotesize 9.60}{\scriptsize $\pm$0.69}} \\
Mixed kNN & {\footnotesize 95.72}{\scriptsize $\pm$0.23} & {\footnotesize 12.30}{\scriptsize $\pm$0.26} & {\footnotesize 2.58}{\scriptsize $\pm$0.25} & {\footnotesize 1.88}{\scriptsize $\pm$0.33} & {\footnotesize 7.02}{\scriptsize $\pm$0.21} & {\footnotesize 0.44}{\scriptsize $\pm$0.00} & {\footnotesize 11.22}{\scriptsize $\pm$0.62} \\
Unit & \textbf{{\footnotesize 95.79}{\scriptsize $\pm$0.10}} & \textbf{{\footnotesize 9.58}{\scriptsize $\pm$0.38}} & {\footnotesize 0.77}{\scriptsize $\pm$0.24} & {\footnotesize 2.83}{\scriptsize $\pm$0.24} & {\footnotesize 6.20}{\scriptsize $\pm$0.13} & {\footnotesize 0.41}{\scriptsize $\pm$0.00} & {\footnotesize 11.37}{\scriptsize $\pm$0.29} \\
Mixed unit & {\footnotesize 95.24}{\scriptsize $\pm$0.37} & {\footnotesize 10.89}{\scriptsize $\pm$0.80} & {\footnotesize 3.59}{\scriptsize $\pm$1.21} & {\footnotesize 1.84}{\scriptsize $\pm$0.56} & {\footnotesize 7.43}{\scriptsize $\pm$0.50} & {\footnotesize 0.43}{\scriptsize $\pm$0.00} & {\footnotesize 10.06}{\scriptsize $\pm$0.91} \\
Subset & {\footnotesize 95.49}{\scriptsize $\pm$0.47} & {\footnotesize 10.69}{\scriptsize $\pm$0.60} & {\footnotesize 3.76}{\scriptsize $\pm$0.95} & \textbf{{\footnotesize 1.72}{\scriptsize $\pm$0.53}} & {\footnotesize 7.41}{\scriptsize $\pm$0.48} & {\footnotesize 0.44}{\scriptsize $\pm$0.00} & {\footnotesize 10.74}{\scriptsize $\pm$1.39} \\
\bottomrule
\end{tabular}
}
\caption{Results for the Compas dataset. The best result for each metric is highlighted in black.}
\label{tab:CompassResults}
\end{table}

\begin{table}[htbp]
\centering
\resizebox{\linewidth}{!}{
\begin{tabular}{lccccccc}
\toprule
Model & ACC (\%) & IND (\%) & SUF (\%) & SEP (\%) & CON (\%) & LIP & ENT \\
\midrule
Log. reg. & {\footnotesize 80.52}{\scriptsize $\pm$0.46} & {\footnotesize 6.55}{\scriptsize $\pm$1.02} & {\footnotesize 2.24}{\scriptsize $\pm$1.00} & \textbf{{\footnotesize 5.11}{\scriptsize $\pm$2.21}} & {\footnotesize 7.27}{\scriptsize $\pm$0.33} & {\footnotesize 0.49}{\scriptsize $\pm$0.00} & {\footnotesize 2.07}{\scriptsize $\pm$0.06} \\
\hdashline
kNN & {\footnotesize 80.86}{\scriptsize $\pm$0.16} & {\footnotesize 14.05}{\scriptsize $\pm$0.42} & {\footnotesize 3.04}{\scriptsize $\pm$0.39} & {\footnotesize 11.79}{\scriptsize $\pm$0.80} & {\footnotesize 16.77}{\scriptsize $\pm$0.24} & {\footnotesize 0.57}{\scriptsize $\pm$0.00} & {\footnotesize 2.11}{\scriptsize $\pm$0.02} \\
Mixed kNN & \textbf{{\footnotesize 82.13}{\scriptsize $\pm$0.05}} & {\footnotesize 17.20}{\scriptsize $\pm$0.61} & {\footnotesize 3.08}{\scriptsize $\pm$0.22} & {\footnotesize 5.15}{\scriptsize $\pm$1.83} & {\footnotesize 14.86}{\scriptsize $\pm$0.51} & {\footnotesize 0.57}{\scriptsize $\pm$0.00} & {\footnotesize 2.30}{\scriptsize $\pm$0.01} \\
Unit & {\footnotesize 73.78}{\scriptsize $\pm$1.52} & {\footnotesize 1.89}{\scriptsize $\pm$1.96} & {\footnotesize 1.94}{\scriptsize $\pm$2.70} & {\footnotesize 48.63}{\scriptsize $\pm$28.49} & \textbf{{\footnotesize 3.43}{\scriptsize $\pm$5.20}} & \textbf{{\footnotesize 0.19}{\scriptsize $\pm$0.24}} & {\footnotesize 1.41}{\scriptsize $\pm$0.10} \\
Mixed unit & {\footnotesize 72.83}{\scriptsize $\pm$2.28} & {\footnotesize 3.46}{\scriptsize $\pm$4.02} & {\footnotesize 2.25}{\scriptsize $\pm$2.57} & {\footnotesize 36.36}{\scriptsize $\pm$11.80} & {\footnotesize 4.76}{\scriptsize $\pm$7.34} & {\footnotesize 0.20}{\scriptsize $\pm$0.25} & \textbf{{\footnotesize 1.35}{\scriptsize $\pm$0.14}} \\
Subset & {\footnotesize 73.41}{\scriptsize $\pm$1.61} & \textbf{{\footnotesize 1.70}{\scriptsize $\pm$2.01}} & \textbf{{\footnotesize 1.67}{\scriptsize $\pm$2.31}} & {\footnotesize 35.37}{\scriptsize $\pm$34.18} & {\footnotesize 3.54}{\scriptsize $\pm$5.59} & {\footnotesize 0.19}{\scriptsize $\pm$0.24} & {\footnotesize 1.39}{\scriptsize $\pm$0.10} \\
\bottomrule
\end{tabular}
}
\caption{Results for the Adult dataset. The best result for each metric is highlighted in black.}
\label{tab:AdultResults}
\end{table}

\subsection{Understanding fairness}
\label{subapp:understanding}

This Appendix concludes with the analysis of the SHAP variable influence for all datasets in Figures \ref{fig:InfluenceGerman} to \ref{fig:InfluenceAdult}, which show the normalized and absolute aggregated SHAP values for all predictions.\\
The results for the German dataset prove that the graph models use the same variables as logistic regression to make their predictions, although the importance is dampened depending on the underlying topology: this effect is less pronounced on sparse graphs like the unit ball configuration. On the other hand, results for both the Compas and Adult dataset exhibit that graph models can also exacerbate the effect of certain variables to encourage fairness. This is particularly striking in the case of the sensitive variable in the Compas dataset, where the relative importance of race is increased. However, this falls in line with the findings of \cite{ADVERSARIAL}, whose model also increases the influence of the sensitive variable in order to debias. Finally, we find that FSD models tend to exacerbate low-variance variables, which contributes to the reduction in individual bias seen in the main paper.

\begin{figure}[!t]
    \centering
    \includegraphics[width=0.49\textwidth]{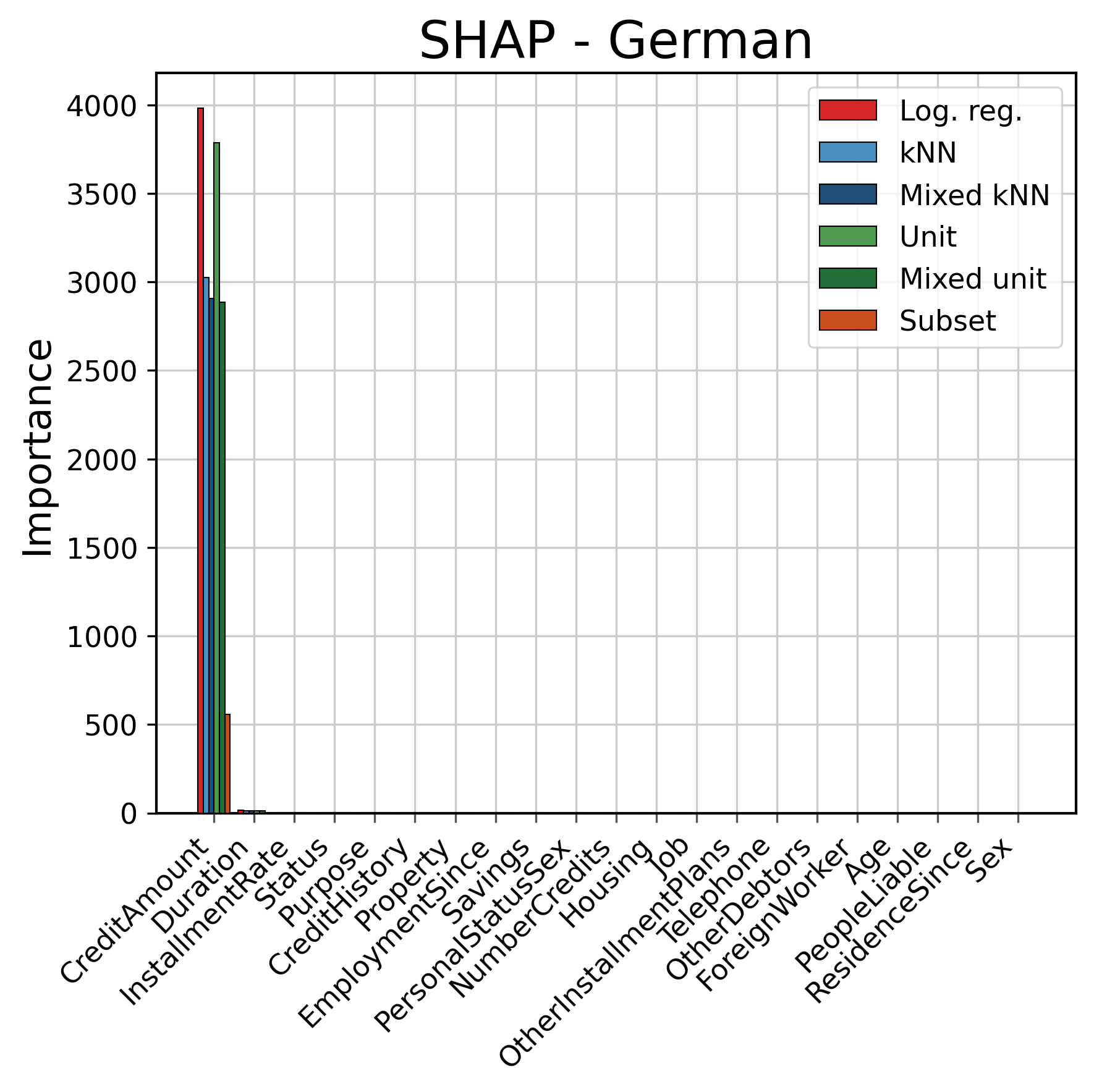}
    \includegraphics[width=0.49\textwidth]{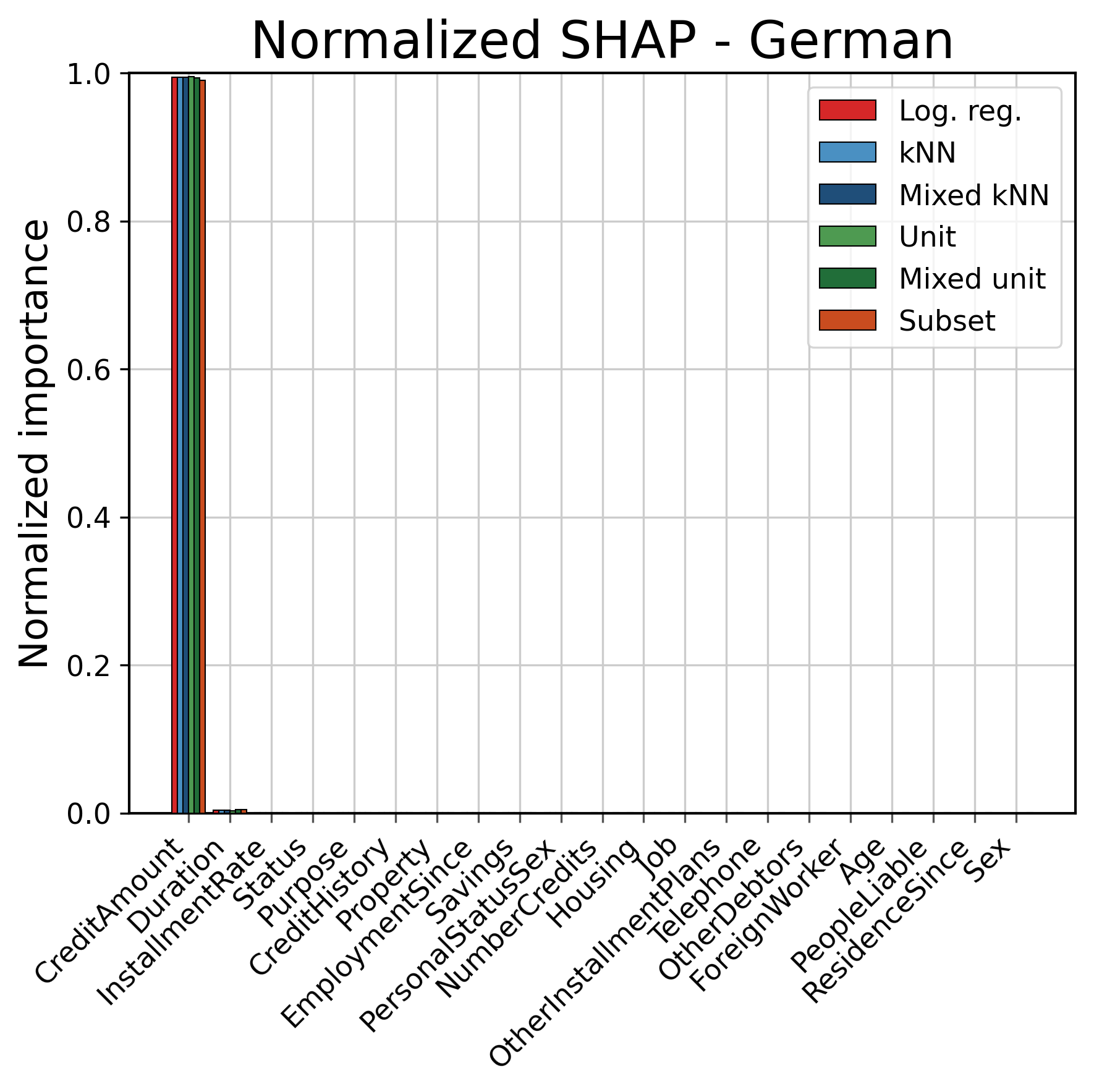} 
        \caption{SHAP variable influence for the German dataset. The left image shows average absolute SHAP importance, while the right image is normalized.}
    \label{fig:InfluenceGerman}
\end{figure}

\begin{figure}[!t]
\centering
    \includegraphics[width=0.49\textwidth]{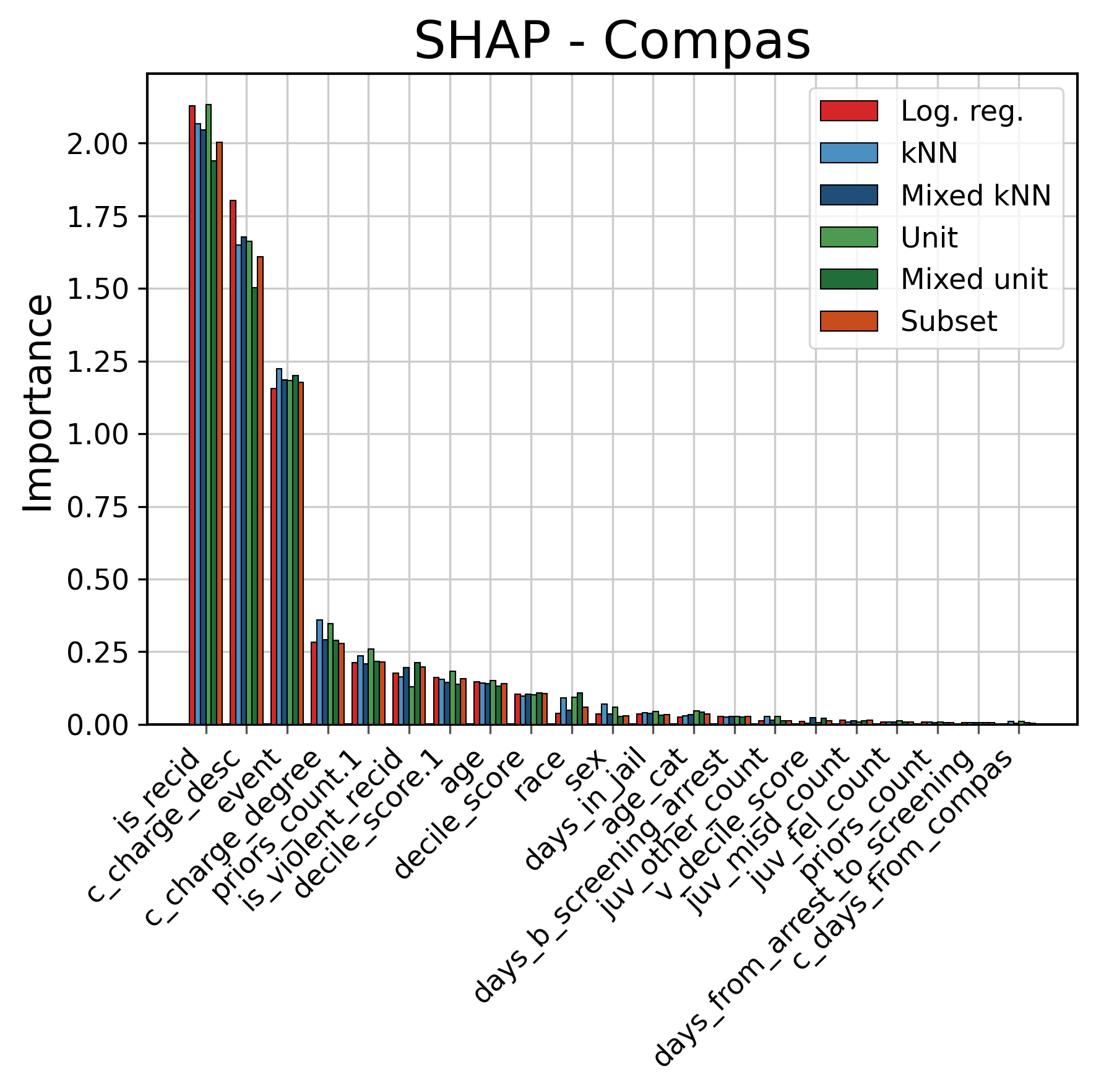}
    \includegraphics[width=0.49\textwidth]{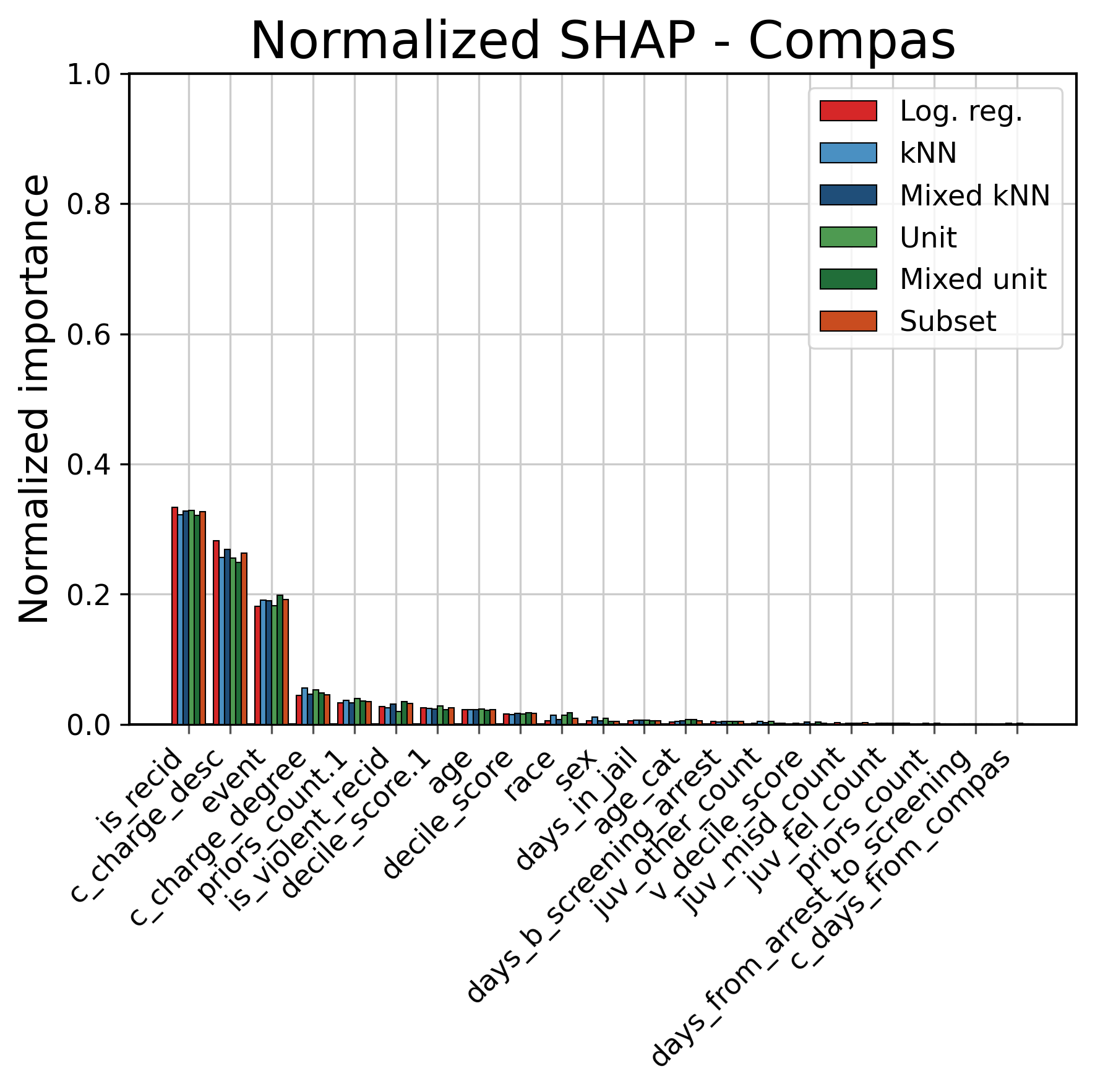} 

        \caption{SHAP variable influence for the Compas dataset. The left image shows average absolute SHAP importance, while the right image is normalized.}
    \label{fig:InfluenceCompas}
\end{figure}

\begin{figure}[!t]
    \centering
    \includegraphics[width=0.49\textwidth]{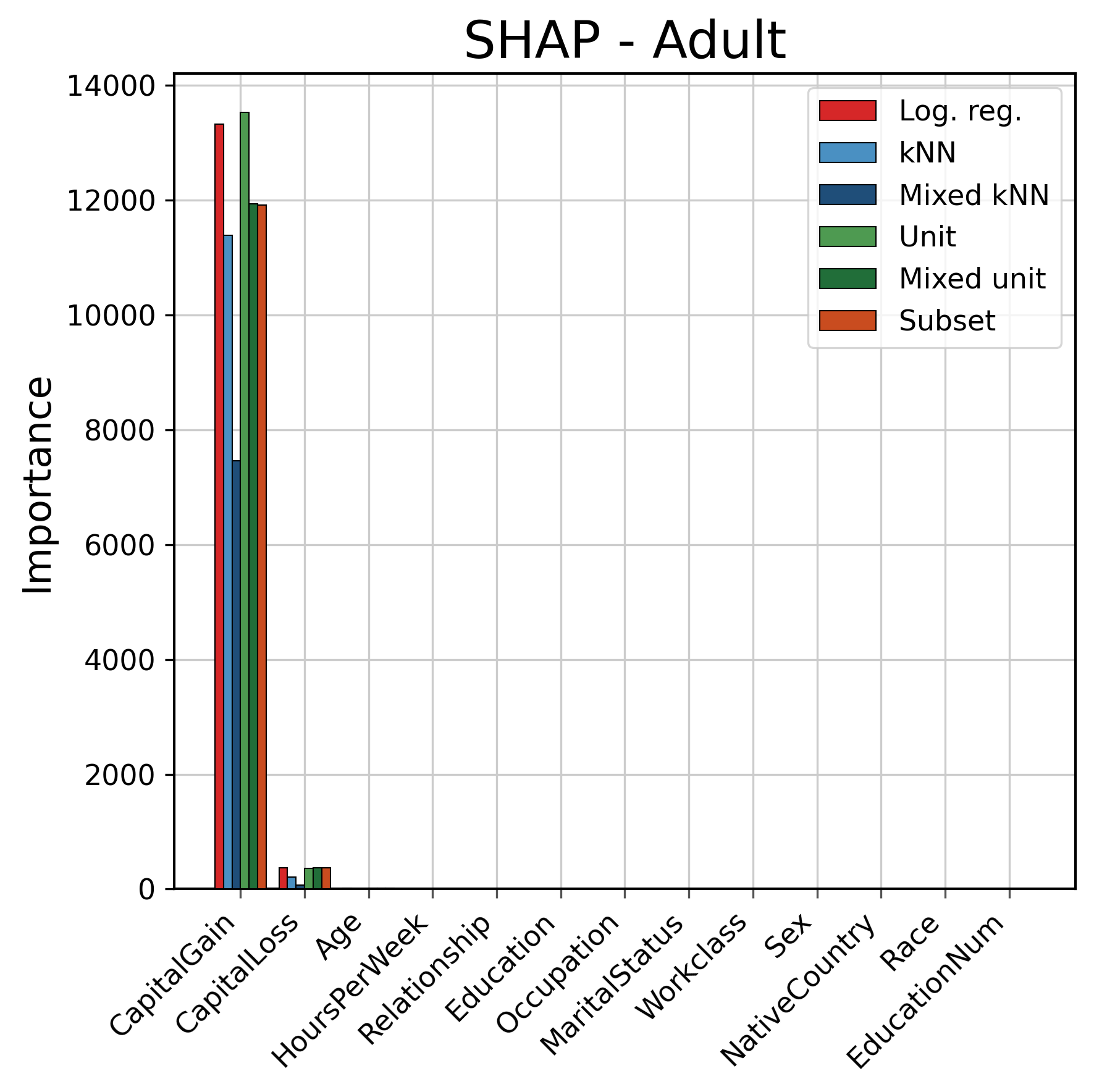}
    \includegraphics[width=0.49\textwidth]{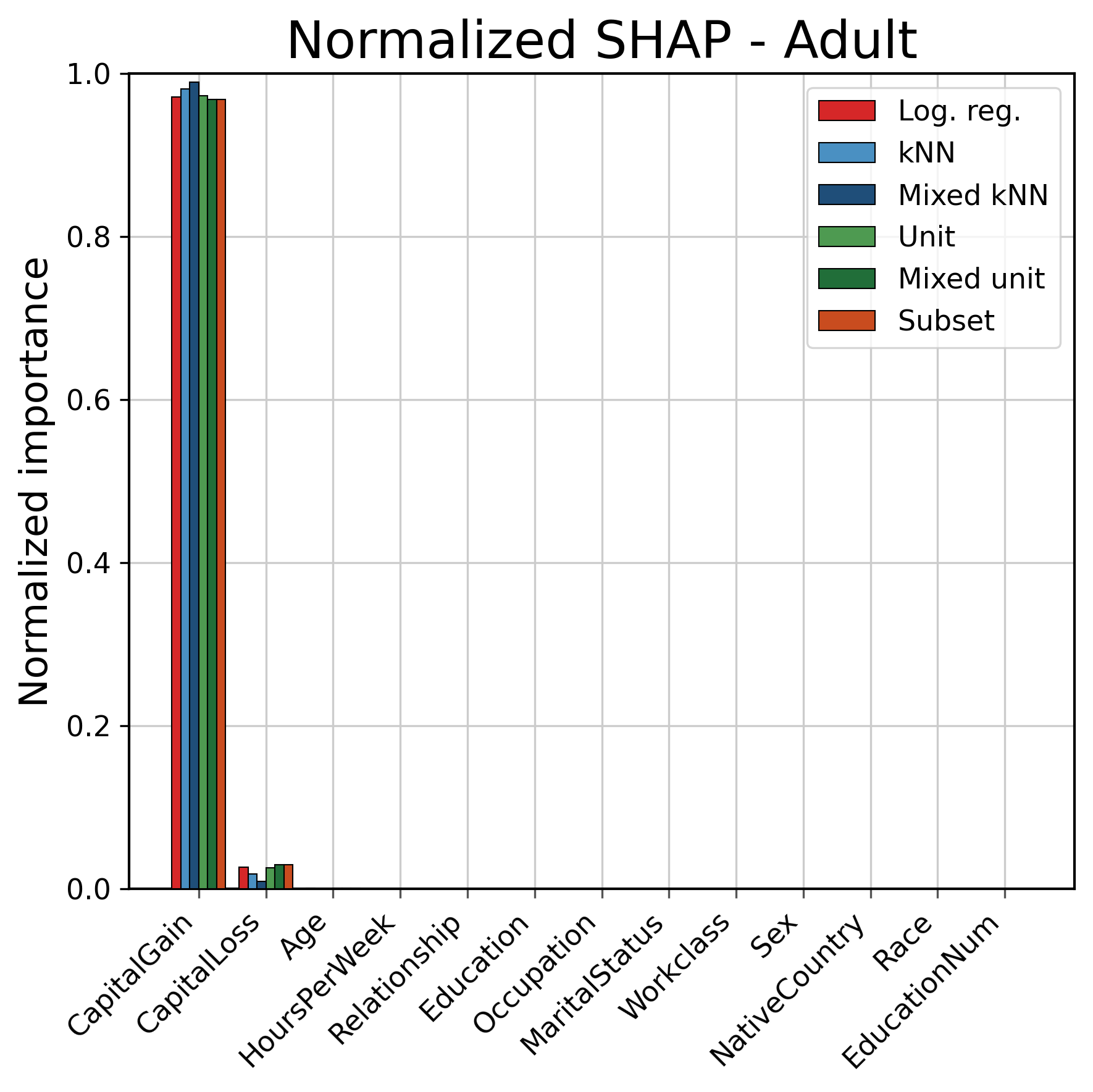}   
    \caption{SHAP variable influence for the Adult dataset. The left image shows average absolute SHAP importance, while the right image is normalized.}
    \label{fig:InfluenceAdult}
\end{figure}

%% else use the following coding to input the bibitems directly in the
%% TeX file.

% \begin{thebibliography}{00}

% %% \bibitem{label}
% %% Text of bibliographic item

% \bibitem{}

% \end{thebibliography}
\end{document}